%% file: main_arxiv.tex
\documentclass{article} % For LaTeX2e
\usepackage{arxiv,times}

\input{math_commands.tex}

\usepackage{hyperref}
\usepackage{url}

\usepackage{amsmath}
\usepackage{amssymb}
\usepackage{mathtools}
\usepackage{amsthm}
\usepackage{algorithm}
\usepackage{algorithmic}
\usepackage{booktabs}
\usepackage{graphicx}
\usepackage{subcaption}
\usepackage{enumerate}
\usepackage{comment}
\usepackage{xcolor}
\usepackage{colortbl}
\usepackage{tablefootnote}
\usepackage{makecell}
\usepackage{pifont}
\usepackage{bbding}
\usepackage{accents}
\usepackage{graphicx}
\usepackage{thmtools}
\usepackage{thm-restate}
\usepackage{wrapfig}
\usepackage[capitalize,noabbrev]{cleveref}
\allowdisplaybreaks
\theoremstyle{plain}
\newtheorem{theorem}{Theorem}[section]

\newtheorem{lemma}[theorem]{Lemma}

\theoremstyle{definition}

\newtheorem{assumption}[theorem]{Assumption}
\theoremstyle{remark}
\newtheorem{remark}[theorem]{Remark}

\crefname{assumption}{Assumption}{Assumptions}

\newcommand{\crefdefpart}[2]{%
  \hyperref[#2]{\namecref{#1}~\labelcref*{#1}(\ref*{#2})}%
}

\newcommand{\refdefpart}[2]{%
  \hyperref[#2]{\labelcref*{#1}(\ref*{#2})}%
}

\newcommand{\factdefpart}[2]{%
  \hyperref[#2]{Fact~(\ref*{#2})}%
}

\newcommand{\halfcheck}{\ding{51}\textsuperscript{\kern-0.55em\ding{55}}}
\newcommand{\ubar}[1]{\underaccent{\bar}{#1}}

\title{Noise-Adaptive Layerwise Learning Rates: Accelerating Geometry-Aware Optimization for Deep Neural Network Training}

\author{Jie Hao, 
~Xiaochuan Gong, ~Jie Xu, ~Zhengdao Wang, ~Mingrui Liu \thanks{Correspondence Author: Mingrui Liu (mingruil@gmu.edu).} \\
George Mason University\\
Fairfax, VA 22030, USA \\
\texttt{\{jhao6, xgong2, jxu13,  zwang52, mingruil\}@gmu.edu} \\
}

\iclrfinalcopy % Uncomment for camera-ready version, but NOT for submission.
\begin{document}

\maketitle

% \input{iclr2026/main/abstract}
% \input{iclr2026/main/main}
% 1.Blockwise learing rate is useful, 2. Develop a generic noise-adaptive algorithm that can adapt to many optimizer, which can get rid of block-wise learnign rate tuning. 3. provide the convergence analysis to demonstrate that...

% with a noise-adaptive layerwise algorithm framed via norm-constrained linear minimization oracles (LMOs).To address these issues, Building on the framework of norm-constrained linear minimization oracles (LMOs), we introduce

%\vspace*{-0.3in}

\begin{abstract}

%Blockwise learning rates and geometry-aware optimizers have been demonstrated to be highly effective in accelerating large-scale neural network training, especially in transformer-based models. Yet they pose a dilemma: blockwise learning rates paired with state-of-the-art optimizers show strong performance, but require extensive per-block hyperparameter tuning. In contrast, geometry-aware methods leverage layerwise structure, but they tend to underperform without careful tuning of architecture-specific hyperparameters. 

% (for example, spectral norm for matrix layers such as query-key, value-output, and MLP blocks in Transformers)
%% However, even within the group of layers defined by the same norm, layers can exhibit vastly different geometries. 

%the loss landscapes can be heterogeneous across layers However, even within a group of layers sharing the same norm,
%

% However, even within a group of layers associated with the same norm,  the loss landscapes can be heterogeneous across layers and evolve dynamically during training.
% For example, recent work shows that sharpness can vary widely across transformer layers and evolve during training, 
Geometry-aware optimization algorithms, such as Muon, have achieved remarkable success in training deep neural networks (DNNs).  These methods leverage the underlying geometry of DNNs by selecting appropriate norms for different layers and updating parameters via norm-constrained linear minimization oracles (LMOs). 
However, even within a group of layers associated with the same norm, the local curvature can be heterogeneous across layers and vary dynamically over the course of training. For example, recent work shows that sharpness varies substantially across transformer layers and throughout training, yet standard geometry-aware optimizers impose fixed learning rates to layers within the same group, which may be inefficient for DNN training.

In this paper, we introduce a \emph{noise-adaptive layerwise learning rate} scheme on top of geometry-aware optimization algorithms and substantially accelerate DNN training compared to methods that use fixed learning rates within each group. Our method estimates gradient variance in the dual norm induced by the chosen LMO \emph{on the fly}, and uses it to assign time-varying noise-adaptive layerwise learning rates within each group.
We provide a theoretical analysis showing that our algorithm achieves a sharp convergence rate. Empirical results on transformer architectures such as LLaMA and GPT demonstrate that our approach achieves faster convergence than state-of-the-art optimizers.

\end{abstract}

\input{main/introduction}

\input{main/related_work}

\input{main/preliminaries}

\input{main/method}

\input{main/analysis}

\input{main/experiments}

\section{Conclusion}

We propose LANTON, a geometry-aware optimizer that incorporates noise-adaptive layer-wise learning-rate scaling on the top of LMO-based updates. By estimating gradient variance in the dual norm space and rescaling learning rate across layers, LANTON accelerates the transformer training hindered by heterogeneous and evolving noise. Theoretically, we obtain a sharp convergence rate of $\tilde{O}(1/\sqrt{T} + \sqrt{\sum_{\ell}\bar{\sigma}_{\ell}}/T^{1/4})$ with improved noise dependence across layers. Empirically, LANTON accelerates pretraining and improves validation metrics on GPT2 and LLaMA under a fixed token budget. One limitation of our work is that the theoretical results may depend on the parameter dimension. Another limitation is that our experiments are conducted on moderately sized models; extending and validating the approach at larger scales is an important direction for future work.
% One limitation of our work is that the experiments are conducted on relatively small-scale models; extending and validating the approach on larger models is left for future work. 
% We plan to further investigate noise-adaptive, layer-wise learning-rate scaling to accelerate the pretraining of larger LLMs.
% This work motivates further exploration of noise-adaptive, layer-wise learning-rate schedules for accelerating pretraining of LLMs.
% \textcolor{red}{Limitation of this work?}

% \section*{Reproducibility Statement}
% We state the formal assumptions and results in the main text (\cref{ass:objective,ass:noise,thm:muon}) and provide complete proofs of \cref{thm:muon} in \cref{app:adaptive_muon,app:thm_proof}. We will release the code with training/evaluation scripts, configurations, seeds, and environment files is included soon. All base models are publicly available: LLaMA  and GPT2-small/medium (used under their official research/community license; license text cited in Appendix \ref{app:license}). Datasets C4, MiniPile, and OpenWebText are accessible on HuggingFace under the licenses stated on their corresponding Hugging Face dataset cards. We include download scripts, preprocessing/splits, and references to their dataset cards and licences (cited in Appendix \ref{app:license}). These materials sufficiently support the reproduction of our results.

\section*{Acknowledgments}
We thank \textbf{Corvex AI Cloud} for providing access to NVIDIA H200 compute resources that enabled the experiments in this work. We are also grateful to Jeff Gahan and Cornell Howard for their generous technical support.

\bibliography{arxiv}
\bibliographystyle{arxiv}

\appendix

\newpage
\input{appendix/appendix}

\end{document}

%% file: math_commands.tex
\usepackage{amsmath,amsfonts,bm}

\def\eqref#1{equation~\ref{#1}}
\def\1{\bm{1}}

\DeclareMathAlphabet{\mathsfit}{\encodingdefault}{\sfdefault}{m}{sl}
\SetMathAlphabet{\mathsfit}{bold}{\encodingdefault}{\sfdefault}{bx}{n}

\def\gD{{\mathcal{D}}}

\def\gF{{\mathcal{F}}}
\def\gG{{\mathcal{G}}}

\def\gK{{\mathcal{K}}}

\def\gS{{\mathcal{S}}}

\newcommand{\E}{\mathbb{E}}

\newcommand{\R}{\mathbb{R}}

\DeclareMathOperator*{\argmin}{arg\,min}

\DeclareMathOperator{\sign}{sign}

%% file: main/introduction.tex
% \vspace*{-0.15in}
\section{Introduction}
% \vspace*{-0.05in}

% Deep neural networks (DNNs) have achieved extraordinary success across many domains. However, as the scale of models and datasets grows larger, the training cost becomes a major bottleneck. To mitigate this, recent work has turned to geometry-aware optimization methods, such as Muon and its relatives, which leverage layer-specific geometry and suitable norm constraints via linear minimization oracles (LMOs) to better align updates with the structure of the parameter space. 

%

% have demonstrated remarkable empirical success by leveraging different geometry for different layers for model updates: the Muon optimizer updates hidden layers by orthogonalized matrix while updating other layers by AdamW~\cite{loshchilov2017decoupled}. Later on,

Optimization algorithms are cornerstones for modern deep learning, enabling the training of increasingly large neural networks, such as LLaMA \citep{touvron2023llama} and GPT \citep{achiam2023gpt} models. While standard optimizers such as SGD~\citep{robbins1951stochastic} and Adam \citep{kingma2014adam} remain widely used, they often overlook the geometry of neural network parameter spaces. 
Recently, geometry-aware optimization algorithms such as Muon \citep{jordan2024muon} have demonstrated remarkable empirical success by performing orthogonalized updates on matrix parameters. Building on this idea, \cite{pethick2025training} developed a framework that selects appropriate norms for different layers and updates parameters via norm-constrained linear minimization oracles (LMOs). 
% In contrast, geometry-aware optimization algorithms have recently emerged as a powerful alternative. For example, the Muon optimizer \citep{jordan2024muon} leverages geometry of different layers and perform different update rules: it performs matrix-orthogonalized updates for hidden layers while applying AdamW \citep{loshchilov2017decoupled} to others, achieving strong empirical performance. Building on this idea, \cite{pethick2025training} developed an algorithmic framework that selects appropriate norms across the entire neural network and update parameters through norm-constrained linear minimization oracles (LMOs).
These methods go beyond standard optimizers by exploiting structural properties (e.g. layer-wise operator norms) of DNNs rather than treating all parameters uniformly, thus leading to improved performance and acceleration for large-scale foundation model pretraining~\citep{liu2025muon}. 

\begin{figure*}[!t]
    \centering
    \includegraphics[width=0.24\linewidth]{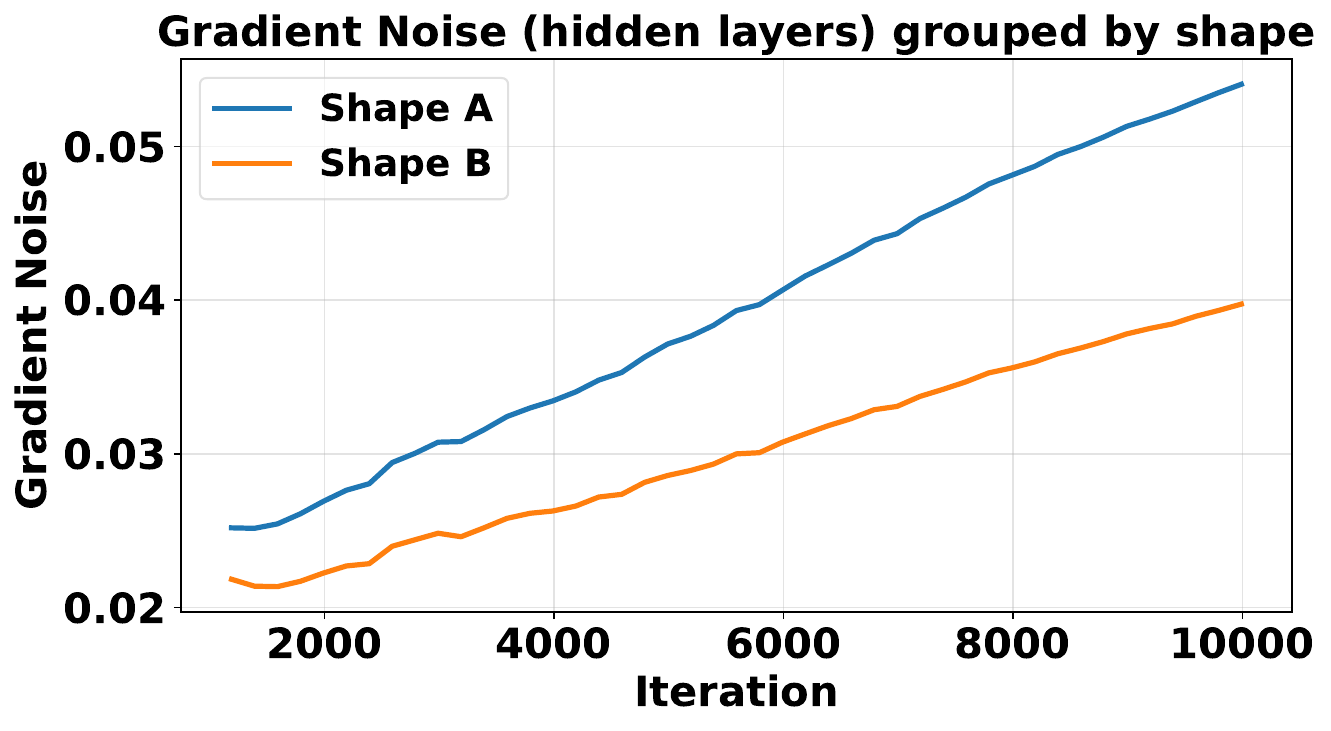}
    \includegraphics[width=0.24\linewidth]{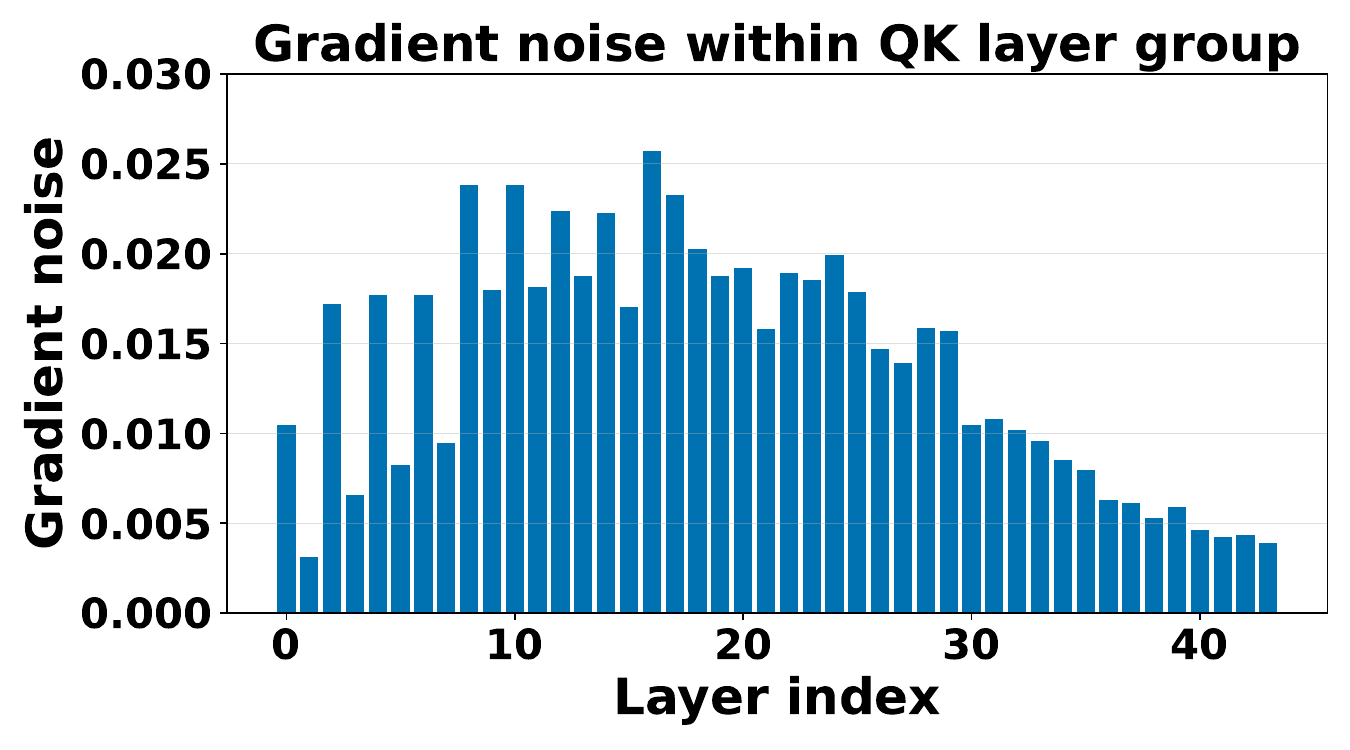}
    \includegraphics[width=0.24\linewidth]{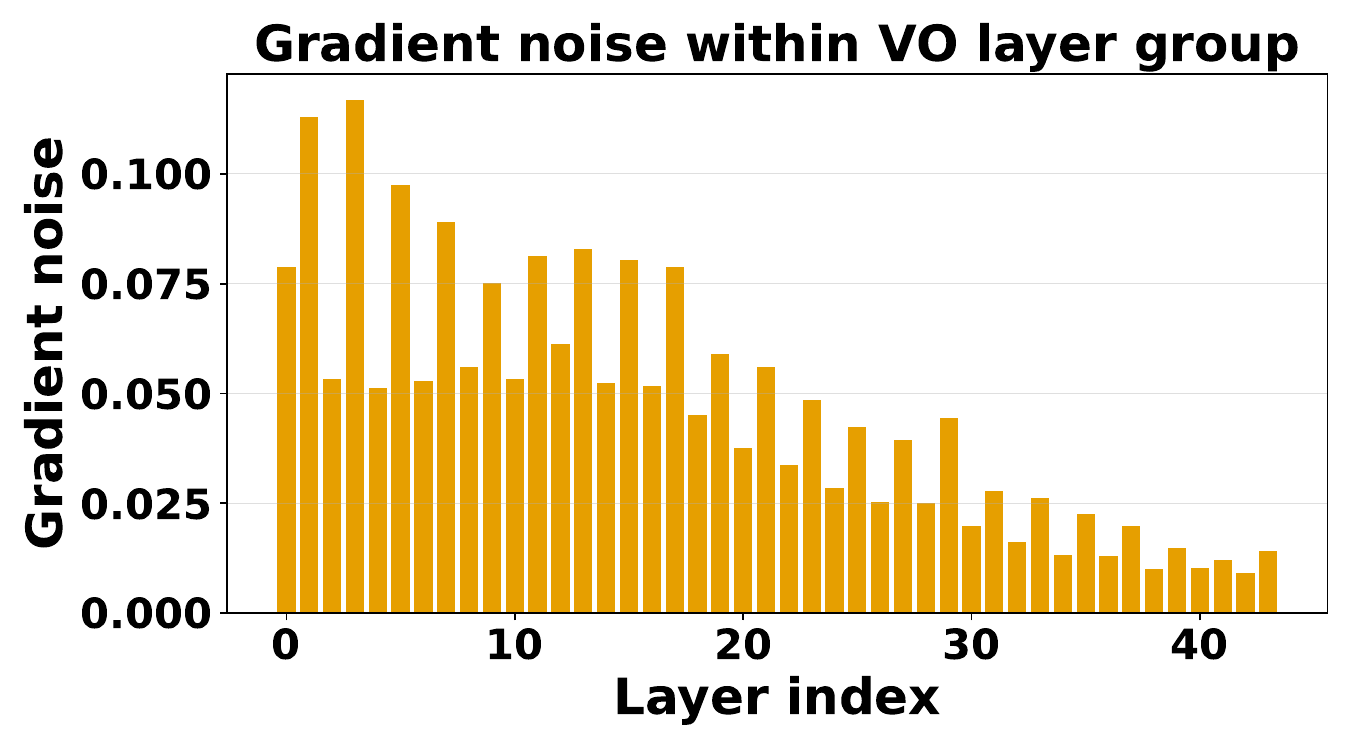}
    \includegraphics[width=0.24\linewidth]{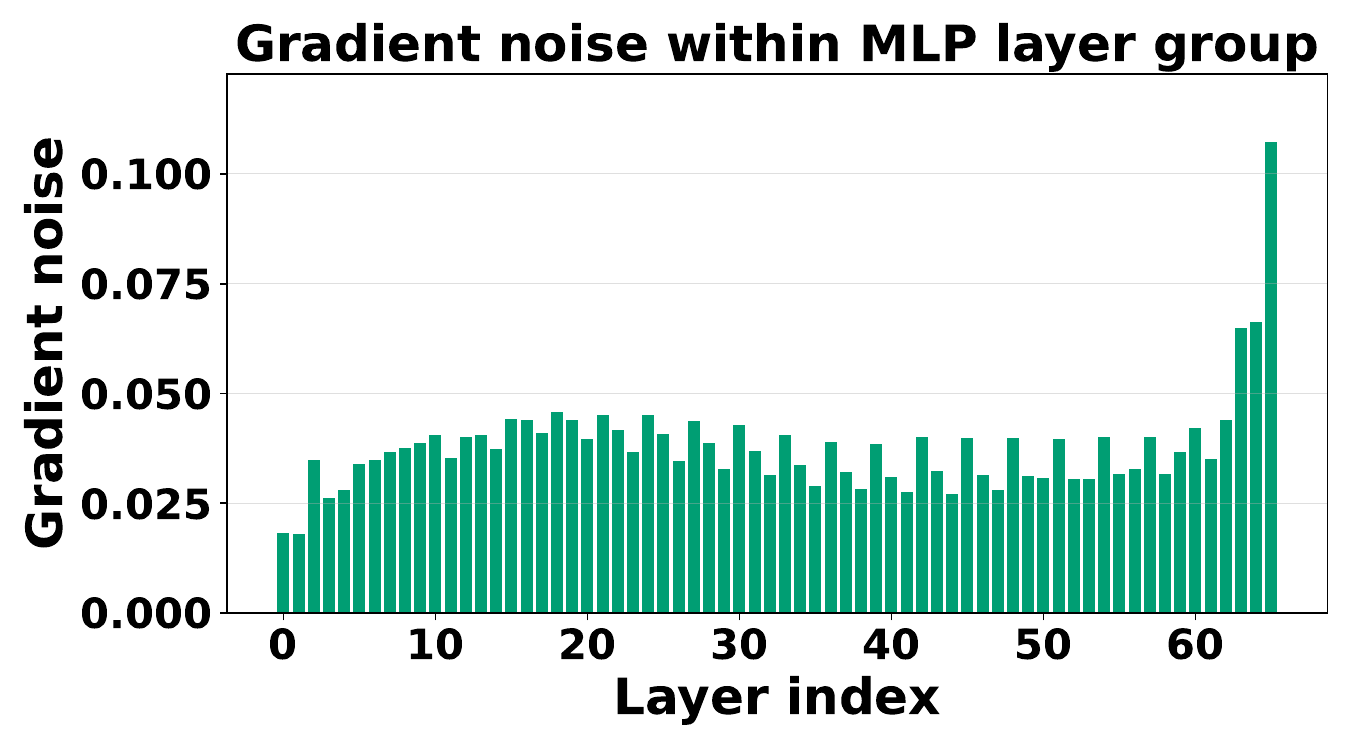}
    % \vspace{-0.1in}
    \caption[]{The stochastic gradient noise is heterogeneous across groups and layers in transformers. The first subfigure shows that average gradient noise in hidden layers varies across parameter groups defined by matrix shape and evolves over training. The last three subfigures illustrate that, within each layer group, the gradient noise varies substantially across layers\footnotemark.}
    % \caption{The stochastic gradient noise is heterogeneous across layer groups and layers in transformers. The stochastic gradient noise is estimated by the nuclear norm (for parameters in Muon optimizer) or 1-norm (\textcolor{red}{$1\to1$ operator norm?}) (for parameters in AdamW optimizer) of the difference between the current step's gradient and the last step's gradient.}
    \label{fig:noise_heterogeneity}
    % \vspace*{-0.1in}
\end{figure*}
% \vspace{-0.1in}
\footnotetext{See \cref{app:noise_verify} for the implementation details.}

Despite their success, most of the existing geometry-aware optimizers simply assign fixed learning rates within groups of layers associated with the same norm choice. However, these algorithms  neglect the heterogeneous and dynamic nature of various layers during the neural network training.   For example, recent studies \citep{wang2025sharpness} have shown that sharpness or local curvature of the objective function can vary substantially across different types of layers (e.g., query-key (QK) layers, value-output (VO) layers, and multilayer perceptron (MLP) in transformers). Moreover, these variations evolve over time, as observed when training with AdamW~\citep{loshchilov2017decoupled}. \citep{riabinin2025gluon} firstly proposed layerwise learning rates for the geometry-aware optimization methods based on smoothness parameters. In contrast, we focus on the heterogeneous noise magnitude of each layer instead of the smoothness parameters. In particular, we have observed similar phenomena in training a LLaMA model with the Muon optimizer\footnote{We follow \url{https://github.com/KellerJordan/modded-nanogpt} to apply Muon optimizer to the transformer hidden layers (including query, key, value, output, MLP layers), and AdamW to the embedding, LM head, normalization layers.}. Figure~\ref{fig:noise_heterogeneity} highlights that the stochastic gradient noise differs substantially across layer groups or layers, and shifts throughout training. Nevertheless, state-of-the-art geometry-aware optimizers such as D-Muon~\citep{liu2025muon} and Scion~\citep{pethick2025training} use the same fixed learning rate for matrices of the same shape, ignoring the fact that gradient noise on layers with the same shape can vary significantly over iterations as shown in Figure~\ref{fig:noise_heterogeneity}. This mismatch suggests that treating such layers uniformly may lead to inefficient training, motivating the need for novel layerwise learning rate schemes.

Layerwise adaptive learning rates~\citep{you2017scaling,you2019large} are widely used in deep learning under standard Euclidean spaces. These optimizers automatically rescale updates according to gradient magnitudes, which reduces manual tuning and often accelerates convergence. However, they disregard the structural geometry of neural networks by treating all parameters as if they belonged to the same category. In reality, neural networks contain diverse parameter groups such as matrices in attention layers, vectors in bias terms, and embedding tables, where different layers in each group exhibit vastly different noise profiles as illustrated in our Figure~\ref{fig:noise_heterogeneity}. The key open question is how to design adaptive learning rates beyond standard Euclidean spaces, enabling geometry-aware optimizers to exploit heterogeneous gradient noise across layers and over the course of training.

%Each group serves a distinct functional role and exhibits different scales and curvature properties in the loss landscape~\citep{wang2025sharpness}. 

%Therefore, it is natural to ask the following question:

%\textcolor{red}{ML: you do not have evidence saying that it leads to slow convergence: cannot make the claim.}
% be inefficient, leading to slower convergence or suboptimal utilization of training dynamics. 
%This mismatch highlights an important gap in the design of geometry-aware optimizers. That naturally raise a question:

%\emph{Can we exploit the stochastic gradient noise from different layers to design a layerwise learning rate scheme on the top of geometry-aware optimizers?}

In this paper, we propose a new geometry-aware optimization algorithm named \emph{Lanton: LAyer-wise Noise-adaptive learning raTe scaling with Operator Norms}. Our algorithm dynamically estimates gradient variance in the dual norm induced by the chosen LMO and uses this estimate to assign layerwise learning rates that adapt over the course of training. Unlike existing approaches, which treat all layers in a group uniformly, our algorithm accounts for the heterogeneity of gradient noise across layers, leading to smaller learning rates for layers with larger gradient noise, thereby enabling finer-grained and more efficient optimization. Importantly, the proposed mechanism is compatible with the geometry-aware optimizers, such as Muon~\citep{jordan2024muon} and D-Muon~\citep{liu2025muon}. Our contribution can be summarized as follows.

\begin{itemize}
 % \vspace*{-0.02in}
    \item We propose a new optimization algorithm named \emph{LANTON: LAyer-wise Noise-adaptive learning raTe scaling with Operator Norms}, which can dynamically capture the gradient noise of each layer and thus accordingly rescale the learning rate of each layer. 
    % \vspace*{-0.02in}
    \item We prove that our method achieves a sharp convergence rate of $\tilde{O}(1/\sqrt{T} + \sqrt{\sum_{\ell}\bar{\sigma}_{\ell}}/T^{1/4})$ for the gradient norm, where $\bar{\sigma}_{\ell}$ denotes an upper bound on the gradient noise of the layer $\ell$. Our bound shows improved noise dependence under the layer-wise noise assumption. By explicitly accounting for the heterogeneous noise levels across layers, our analysis demonstrates the advantage of noise-adaptive layer-wise learning rates.
    %Since different layers exhibit distinct noise magnitudes, we treat them individually rather than uniformly.
        % \vspace*{-0.02in}
    \item Empirically, we evaluate our approach on language model training and image classification, including LLaMA, GPT2 and convolutional neural network, and show that it substantially accelerates training and improves sample efficiency compared to state-of-the-art optimizers.
    % For example, LANTON saves $\sim 1.5 \times$ training samples compared to the state-of-the-art algorithm D-Muon when reaching comparable training or validation loss.
    Our results indicate that dynamically adapting learning rates at the layer level can better capture the evolving optimization landscape, leading to faster convergence and improved training efficiency. Together, these contributions highlight the importance of integrating noise adaptivity into geometry-aware optimization and open new directions for scalable and effective training of deep neural networks.
\end{itemize}

% Yet even within layer groups sharing a norm, there remains substantial heterogeneity: different layers can exhibit very different curvature, or gradient noise, and these properties can evolve during training. We have observed this phenomenon during the training of transformers, such as LLaMA, by muon optimizer\footnote{We follows \url{https://github.com/KellerJordan/modded-nanogpt} to apply muon optimizer to the transformer hidden layers (including query, key, value, output, mlp layers), and AdamW to the embdedding, LM head, normalization layers.}. Figure \ref{fig:noise_heterogeneity} shows the heterogeneity of gradient noise across different layer groups. 

%% file: main/related_work.tex
% \vspace{-0.1in}

\section{Related Work}

A long line of work has studied optimization for deep learning. The most classical method is SGD \citep{robbins1951stochastic}. Early advances focused on adaptive learning rates, including Adagrad \citep{duchi2011adaptive}, RMSProp \citep{tieleman2012lecture}, Adadelta \citep{zeiler2012adadelta}, and the widely used Adam \citep{kingma2014adam}. Later developments improved Adam in various ways: AdamW \citep{loshchilov2017decoupled} introduced decoupled weight decay and has become the default choice for deep learning; several variants incorporate variance reduction, such as AdEMAMix \citep{pagliardini2024ademamix} and MARS-AdamW \citep{yuan2024mars}; others target memory efficiency, including Adafactor \citep{shazeer2018adafactor}, Lion \citep{chen2023symbolic}, MeZO \citep{malladi2023fine}, GaLore \citep{zhao2024galore}, Adam-mini \citep{zhang2024adam}, and Signum \citep{zhao2024deconstructing}.

Another line of work approximates or leverages second-order information. K-FAC \citep{martens2015optimizing} and Shampoo \citep{gupta2018shampoo} are classical examples. The substantial compute and memory overheads of second-order optimizers have motivated distributed implementations of Shampoo \citep{anil2020scalable,shi2023distributed}. More recently, lightweight preconditioned optimizers such as Sophia \citep{liu2023sophia} and SOAP \citep{vyas2024soap} have been proposed, achieving substantial speedups over AdamW in large-scale language model pretraining.

A third research direction focuses on layer-wise or block-wise learning rates to accelerate training. LARS \citep{you2017scaling} and LAMB \citep{you2019large} are widely used for large-batch training, while more recent approaches extend AdamW with blockwise learning rates \citep{wang2025sharpness}. 

Several parameter-free or schedule-free optimizers aim to reduce the burden of hyperparameter tuning, including Dog \citep{ivgi2023dog}, Prodigy \citep{mishchenko2023prodigy}, and Schedule-Free AdamW \citep{defazio2024road}.

Most recently, the theory of modular duality in optimization and the perspective of steepest descent under different operator norms \citep{bernstein2024modular, bernstein2024old, large2024scalable} have inspired the design of matrix-based and geometry-aware optimizers, including Muon \citep{jordan2024muon} and Scion \citep{pethick2025training}, as well as variance-reduced variants \citep{liu2025mars,qian2025muon} and distributed implementations such as D-Muon \citep{liu2025muon}, Dion \citep{ahn2025dion}, and MuonBP \citep{khaled2025muonbp}, which further improve training efficiency and stability at scale.

% \textcolor{red}{variance reduction MARS-M \citep{liu2025mars} and Muon-MVR \citep{qian2025muon}, weijie su, benamin grimmar}

%% file: main/preliminaries.tex
% \vspace{-0.1in}
\section{Preliminaries}
% \vspace{-0.05in}
% In this work, we consider the following stochastic optimization problem:
% \begin{align}
%     \min_{X}f(X):=\E_{\xi\in\gD}[F(X;\xi)],
% \end{align}
% where $\xi$ is a randomly selected data sample or random noise following an underlying distribution $\gD$, and $X\in\R^{m\times n}$ is the model parameter. Throughout, we assume that the objective function is bounded from below, i.e., $f^*\coloneqq \inf_{X}f(X) > -\infty$.
In this work, we consider the stochastic optimization problem $\min_{X}f(X):=\E_{\xi\in\gD}[F(X;\xi)]$, where $\xi$ is random noise sampled from an unknown distribution $\gD$, and $X\in\gS$ is the model parameter, where $X=[X_1,\dots,X_p]$, $X_i\in\gS_i:=\R^{m_i\times n_i}$, and $\gS:= \prod_{i=1}^p\gS_i$ (Cartesian products). Similarly, write the gradient as $\nabla f(X) = [\nabla_1 f(X), \dots, \nabla_{p} f(X)] \in \mathcal{S}$, and the stochastic gradient as $\nabla F(X;\xi) = [\nabla_1 F(X;\xi), \dots, \nabla_{p} F(X;\xi)] \in \mathcal{S}$ (here we adopt the notation and setup from \citep{riabinin2025gluon}. We assume that the objective is bounded from below, i.e., $f^*\coloneqq \inf_{X}f(X) > -\infty$.

\textbf{Notations.}
Let $\|\cdot\|$ denote an arbitrary (not necessarily Euclidean) vector/matrix norm with associated dual norm $\|\cdot\|_{*}$, and let $\|\cdot\|_{\text{nuc}}$ denote the nuclear norm. We use $\langle \cdot,\cdot\rangle$ for the trace inner product, defined as $\langle A,B\rangle = \mathrm{tr}(A^{\top}B)$ for $A,B \in \mathbb{R}^{m \times n}$. For two positive functions $f$ and $g$, we write $f \lesssim g$ (resp. $f \gtrsim g$) if there exists $c > 0$ such that $f(x) \leq c g(x)$ (resp. $f(x) \geq c g(x)$) for all $x$. We use standard big-O notation, with $\tilde{O}$ and $\tilde{\Omega}$ used to hide polylogarithmic factors, respectively.
%% The linear minimization oracle (LMO) is a fundamental concept in convex optimization \citep{frank1956algorithm}, particularly in the context of algorithms like the Frank-Wolfe algorithm (also known as the conditional gradient method \citep{jaggi2013revisiting}). Given a convex feasible set $\gK$ and a direction vector/matrix $u$, the LMO returns an extreme point of $\gK$ that minimizes the linear function $\langle u,x \rangle$ over $\gK$. Mathematically, this can be expressed as:
% % $\mathrm{LMO}(u) = \argmin_{x\in\gK} \langle u,x \rangle$.
% \begin{align*}
%     \mathrm{LMO}(u) = \argmin_{x\in\gK} \langle u,x \rangle.
% \end{align*}

\textbf{Linear Minimization Oracle (LMO).} The LMO is a fundamental concept in convex optimization \citep{frank1956algorithm}, particularly in the context of algorithms like the Frank-Wolfe algorithm (also known as the conditional gradient method \citep{jaggi2013revisiting}). Given a convex feasible set $\gK$ and a direction vector/matrix $u$, the LMO returns an extreme point of $\gK$ that minimizes the linear function $\langle u,x \rangle$ over $\gK$. Mathematically, this can be expressed as: $\mathrm{LMO}(u) = \argmin_{x\in\gK} \langle u,x \rangle$.

Throughout this paper, we focus on the special case where $\gK:=\{x\mid\|x\|\leq 1\}$ for some chosen (not necessarily Euclidean) norm $\|\cdot\|$~\citep{pethick2025training}, unless specified otherwise.

\textbf{Operator Norm and RMS Norm.}
% We introduce the definitions of the (matrix) operator norm and the (vector) RMS norm. 
Given a matrix $A\in\R^{m\times n}$ and two normed vector spaces $(\R^{n}, \|\cdot\|_{a})$ and $(\R^{m}, \|\cdot\|_{b})$, the ``$a$ to $b$" induced operator norm is defined as
%\begin{align*}
    $\|A\|_{a\to b} := \max_{x\in\R^{n}, x\neq0}\frac{\|Ax\|_{b}}{\|x\|_{a}}
    = \sup_{\|x\|_{a}=1}\|Ax\|_{b}$.
%\end{align*}
Given a vector $x\in\R^{d}$, the RMS norm is defined as $\|x\|_{\text{RMS}} := \frac{1}{\sqrt{d}}\|x\|_{2}$.
% \begin{align*}
%     \|x\|_{\text{RMS}} := \frac{1}{\sqrt{d}}\|x\|_{2}.
% \end{align*}

%% file: main/method.tex
% \vspace*{-0.05in}

\section{Our Method}
% \vspace*{-0.05in}
\begin{algorithm}[t]
    \caption{LANTON: LAyer-wise Noise-adaptive raTe scaling with Operator Norms} \label{alg:muon}
    \begin{algorithmic}[1]
        \STATE \textbf{Input:} $X_1, \alpha, \beta_1, \beta_2, \gamma, \eta, B_0 = \nabla F(X_1; \xi_1), H_0^{\ell}=0$ 
        \FOR{$t=1$ to $T$}
        % \FOR{$t=0,\dots,T$}
            % \FOR{each layer group $i$ in $\mathcal{I}$}
        \FOR{each layer $\ell$}  
            % \hfill \COMMENT{bar}
            \STATE $G_t^{\ell} = \nabla_{\ell} F(X_t;\xi_t)$, $\tilde{G}_t^{\ell} = \nabla_{\ell} F(X_t;\tilde{\xi}_t)$ 
            \hfill ($\tilde{G}_t^{\ell}$ is used only in Option II)
            \STATE $B_t^{\ell} = \beta_1 B_{t-1}^{\ell} + (1-\beta_1) G_t^{\ell}$
            \STATE $O_t^{\ell} = \mathrm{LMO}(B_t^{\ell})$
            \hfill (choose norm based on $\ell$'s group $\gG_{\ell}$, \cref{tbl:lmo} line 5)
            % \STATE $H_t^{\ell} = \beta_2 H_{t-1}^{\ell} + (1-\beta_2)\|G_t^{\ell}-G_{t-1}^{\ell}\|_{*}^2$ (Option I, works in practice)
            % \STATE $H_t^{\ell} = \beta_2 H_{t-1}^{\ell} + (1-\beta_2)\|G_t^{\ell}-\tilde{G}_t^{\ell}\|_{*}^2$ (Option II, works in theory)
            \STATE $H_t^{\ell} = \beta_2 H_{t-1}^{\ell} + (1-\beta_2)\cdot
            \begin{cases}
                \|G_t^{\ell}-G_{t-1}^{\ell}\|_{*}^2 & \text{Option I (practical)} \\
                \|G_t^{\ell}-\tilde{G}_t^{\ell}\|_{*}^2 & \text{Option II (theoretical)}
            \end{cases}$ \hfill (\cref{tbl:lmo} line 4)\footnotemark
            % \STATE $\alpha_t^{\ell} = \alpha/\sqrt{\alpha^2+H_t^{\ell}}$, $\alpha_t^{m} = \max_{\ell\in\gG_{\ell}}\alpha_t^{\ell}$
            % \textcolor{red}{add a footnote and say section 6.4 is a reasonable approximation, cite scion}
            \STATE $\alpha_t^{\ell} = \alpha/\sqrt{\alpha^2+H_t^{\ell}}$, $\alpha_t^{m} = \max_{j\in\gG_{\ell}}\alpha_t^{j}$
            \hfill ($\max$ is over $\ell$'s group $\gG_{\ell}$, \cref{tbl:lmo} line 1)
            \STATE $\eta_t^{\ell} = \eta_t\sqrt{\alpha_t^{\ell}/\alpha_t^{m}}$ 
            \hfill ($\eta_t\in[\eta_{\min}, \eta_{\max}]$ follows a cosine decay schedule)
            \STATE $X_{t+1}^{\ell} = X_t^{\ell} + \eta_t^{\ell}O_t^{\ell}$
            % \STATE $X_{t+1}^{\ell} = (1-\eta_t\gamma)X_t^{\ell} - \eta_t^{\ell}O_t^{\ell}$
            % \hfill (weight decay)
            % \STATE $X_{t+1}^{\ell} = (1-\eta_t\gamma)X_t^{\ell} - \eta_t^{\ell}(0.2\cdot O_t^{\ell}\cdot \sqrt{\max(d_{\text{in}}^{\ell}, d_{\text{out}}^{\ell})})$
            % \hfill (weight decay)
        \ENDFOR
        \ENDFOR
        % \ENDWHILE
    \end{algorithmic}
\end{algorithm}
\footnotetext{We use randomized SVD to efficiently approximate the calculation of nuclear norm in practice. See \cref{sec:running_time} for details.}

%%%%%%%%%%%%%%%%%%%%%%%%%%%%%%%%%%%%%%%%%%%%%%%%%
\begin{table*}[t]
% \vspace{-0.15in}
\centering
\caption{The choice of LMO can be different between layers. Denote $G=U\Sigma V^\top\in \R^{d_\mathrm{out} \times d_\mathrm{in}}$ as a matrix and $g\in \mathbb{R}^d$ as a vector. 
% We use $W=U\Sigma V^\top\in \R^{d_\mathrm{out} \times d_\mathrm{in}}$ to denote a matrix and $w\in \mathbb{R}^d$ to denote a vector. 
% Write the SVD as $W = U\Sigma V^\top$.
% For simplicity we overload notation and write the SVD as $W = U\Sigma V^\top$.
}
% \vspace{0.02in}
\label{tbl:lmo}
% \bgroup
\def\arraystretch{1.0}
\resizebox{\textwidth}{!}{
\begin{tabular}{|c|c|c|c|c|c|c|}
\hline
\multicolumn{1}{|c|}{\textbf{Parameter Group}} & \multicolumn{1}{c|}{Hidden layers (query, key, value, output, mlp)} & \multicolumn{1}{c|}{Embedding, LM head layers}  & \multicolumn{1}{c|}{RMS norm} \\
\hline
\textbf{Size} & $\text{Matrix}\in \R^{d_\mathrm{out} \times d_\mathrm{in}}$ & $\text{Matrix}\in \R^{d_\mathrm{out} \times d_\mathrm{in}}$ & $\text{Vector}\in \R^{d}$ \\
\hline
\textbf{Norm $\|\cdot\|$} & $ \text{RMS} \rightarrow \text{RMS}$ & $1 \rightarrow \infty$  & $\text{RMS}$ \\
\hline
\textbf{Dual Norm $\|\cdot\|_{*}$} & $ \sqrt{d_\mathrm{out}/ d_\mathrm{in}}\|\cdot\|_{\text{nuc}} $ & $\|\cdot\|_{1\rightarrow1}$  & $\sqrt{d}\|\cdot\|_2$ \\
\hline
\textbf{LMO} 
  & $-\sqrt{d_\mathrm{out}/ d_\mathrm{in}} UV^\top$
  & $-\frac{1}{d_\mathrm{in}} \sign(G)$
  & $-\sqrt{d}\frac{g}{\|g\|_{2}}$ \\
\hline
\textbf{LMO Implementation} 
  & Newton-Schulz
  & Signum
  & RMS Normalization
\\
\hline
\end{tabular}
}
% \vspace{-0.15in}
% \egroup
\end{table*}

%%%%%%%%%%%%%%%%%%%%%%%%%%%%%%%%%%%%%%%%%%%%%%%%%

\textbf{Algorithmic Framework.}
Our proposed algorithmic framework (\cref{alg:muon}) consists of three main stages at each iteration. First (lines 4-6), we compute the stochastic gradient $G_t^{\ell}$ for each layer, accumulate its momentum $B_t^{\ell}$, and then obtain the direction $O_t^{\ell} = \text{LMO}(B_t^{\ell})$ by invoking a LMO, where the choice of norm depends on the structural group of layer $\ell$ (embedding/LM head layers, hidden layers, or non-matrix layers; see \cref{tbl:lmo}). Note that line 4-6 is the same as the work of Scion \citep{pethick2025training} and Gluon \citep{riabinin2025gluon}. Second (lines 7-9), the key novelty of our framework is to incorporate noise-adaptive layer-wise learning rate scaling. We maintain a momentum buffer $H_t^{\ell}$ to track the moving average of the estimated noise level for each layer. This buffer can be updated in two ways: a practical option (using $G_t^{\ell}$ and $G_{t-1}^{\ell}$ and avoiding extra computation)
% and keeping the batch size unchanged) 
and a theoretical option (using two independent stochastic gradients $G_t^{\ell}$ and $\tilde{G}_t^{\ell}$ at each step). Based on $H_t^{\ell}$, the layer-wise scaling $\alpha_t^{\ell}$ is computed, and the effective learning rate is adjusted proportionally through the ratio $\sqrt{\alpha_t^{\ell}/\alpha_t^{m}}$, ensuring that layers with larger noise magnitudes employ smaller learning rates. Finally (lines 10-11), we update the model parameters with the scaled stepsize and the direction given by LMO. 

%similar to Muon \citep{jordan2024muon}, Scion \citep{pethick2025training}, D-Muon \citep{liu2025muon}, and Gluon \citep{riabinin2025gluon}.
% \textcolor{red}{ML: group is not specified in the algorithm, the definition of $\alpha_t^m$ is ambiguous without defining the group.}

\textbf{Choice of Norm Constraint and LMO Implementation.}
To determine appropriate norm constraints for different types of parameters in deep neural networks, we adopt the operator norm perspective recently advanced in \citep{large2024scalable, bernstein2024modular, pethick2025training}. As summarized in \cref{tbl:lmo}, parameters naturally fall into three groups: (i) hidden layers (e.g., query, key, value, output, and MLP weights), which are represented as matrices and we use the RMS $\to$ RMS operator norm with dual nuclear norm (scaled by $\sqrt{d_{\mathrm{out}}/d_{\mathrm{in}}}$); (ii) weight-sharing layers such as embedding and LM head matrices, where the $\ell_{1}\to \ell_{\infty}$ operator norm is used with dual $\ell_{1}\to \ell_{1}$ norm; and (iii) non-matrix parameters like RMS normalization vectors, where the RMS norm with dual $\ell_2$ norm (scaled by $\sqrt{d_{\text{model}}}$) is adopted. These dual norms are critical in line 7 of \cref{alg:muon} for estimating the layer-wise gradient noise magnitude. Based on the chosen norms, the corresponding LMOs in line 6 of \cref{alg:muon} also differ across parameter types: for hidden layers, the LMO corresponds to a scaled $UV^{\top}$ computed efficiently via Newton-Schulz iterations; for embedding and LM head layers, the LMO reduces to a scaled element-wise sign operator; and for RMS normalization vectors, the LMO is implemented by RMS normalization. This unified design of norm constraints, dual norms, and LMOs with their implementations ensures both theoretical consistency with our algorithmic framework and practical efficiency in large-scale deep learning.

\textbf{Noise-Adaptive Layer-wise Learning Rates.}
To capture the heterogeneous noise levels across different layers, we introduce noise-adaptive layer-wise learning rates, which dynamically scale the stepsize of each layer according to its estimated stochastic gradient variance. Specifically, we maintain a variance tracker $H_{t}^{\ell} = \beta_2 H_{t-1}^{\ell} + (1-\beta_2) \|G_{t}^{\ell}-\tilde{G}_t^{\ell}\|_{*}^2$ (line 7), where $\beta_2\in(0,1)$ serves as a momentum-like parameter that smooths the estimate, akin to second-moment accumulation in adaptive optimizers. The resulting adaptive scaling factor $\alpha_t^{\ell} = \alpha/\sqrt{\alpha^2 + H_{t}^{\ell}}$ (line 8) ensures that layers subject to higher noise levels (large $H_t^\ell$) receive proportionally smaller effective learning rates, consistent with classical stochastic optimization theory. We implement this by reweighting the base learning rate with the ratio $\alpha_t^{\ell}/\alpha_t^{m}$ (where $\alpha_t^{m}=\max_{\ell\in\gG_{\ell}}\alpha_t^{\ell}$), thereby aligning the updates across layers under a unified theoretical principle. While our theoretical framework (see \cref{sec:analysis}) assumes two independent gradient estimates $G_{t}^{\ell}$ and $\tilde{G}_{t}^{\ell}$, in practice we approximate $\tilde{G}_t^\ell$ by the previous step gradient $G_{t-1}^{\ell}$. This avoids doubling the batch size and keeps the total number of sampled data consistent with standard baselines, thus ensuring fair comparisons in empirical evaluation.

\textbf{Comparison with Other Optimizers.} Compared to Muon \citep{jordan2024muon}, Scion \citep{pethick2025training}, Gloun \citep{riabinin2025gluon}, and D-Muon \citep{liu2025muon}, our method introduces noise-adaptive layer-wise learning rates by estimating gradient variance in the dual norm induced by the chosen LMO. Unlike Muon and D-Muon, which use AdamW for embedding and LM head layers, we adopt a geometry-aware framework (similar to Scion) and update these weight-sharing layers with Signum (see \cref{tbl:lmo}).
The main distinction between our work and \cite{riabinin2025gluon} is that our paper studies noise-adaptive layerwise learning rates motivated by \cref{fig:noise_heterogeneity}, whereas \cite{riabinin2025gluon} considers layerwise learning rates arising from varying smoothness parameters. This conceptual difference leads to quite different proof techniques (see \cref{sec:proof_outline}).
% Our main technical contributions, Lemma 5.4 and Lemma 5.5, explicitly characterize stochastic gradient noise magnitudes and are essential for establishing Theorem 5.3. These results do not appear in \citep{riabinin2025gluon}.

Optimizers such as LARS \citep{you2017scaling} and LAMB \citep{you2019large} also use layer-wise rescaling to stabilize large-batch training. However, these methods treat all layers uniformly. In contrast, our algorithm is geometry-aware, selecting norms tailored to hidden, embedding, and normalization layers, and updating them through LMOs with noise-adaptive scaling.

Finally, although \cref{alg:muon} resembles \cite{gong2025adaptive} in estimating noise magnitude, there are key differences. Our method is LMO-based and works under arbitrary norms, while \cite{gong2025adaptive} is restricted to the Euclidean space. Our noise adaptivity refers to per-layer scaling based on estimated variance, whereas theirs targets convergence without prior noise knowledge. Moreover, our moving-average variance estimator $H_t^{\ell}$ remains $O(1)$ with high probability, in contrast to their cumulative estimator $\sum_{k=1}^{t}\|G_k-\tilde{G}_k\|^2$ which grows as $O(t)$.

% Nevertheless, state-of-the-art geometry-aware optimizers such as D-Muon~\citep{liu2025muon} and Scion~\citep{pethick2025training} use the same fixed learning rate for matrices of the same shape, ignoring the fact that gradient noise on layers with the same shape can vary significantly over iterations as shown in Figure~\ref{fig:noise_heterogeneity}. This mismatch suggests that treating such layers uniformly may lead to inefficient training, motivating the need for novel layerwise learning rate schemes.

%% file: main/analysis.tex
% \vspace*{-0.05in}
\section{Analysis}
\label{sec:analysis}
% \vspace*{-0.05in}

In this section, we provide theoretical convergence guarantees for \cref{alg:muon}.
% Let $\|\cdot\|_{(\ell)}$ denote the chosen norm of layer $\ell$ and $\|\cdot\|_{(\ell)*}$ its associated dual norm, and $p$ be the number of layers. 
Let $\|\cdot\|_{(\ell)}$ denote the chosen norm of layer $\ell$ with dual norm $\|\cdot\|_{(\ell)*}$, and let $p$ be the number of layers.
We begin by presenting the assumption of layer-wise $L$-smoothness. Importantly, we do not assume that either the primal norm $\|\cdot\|_{(\ell)}$ or the dual norm $\|\cdot\|_{(\ell)*}$ is Euclidean. 
% This ensures that our results apply to the layer-wise geometries relevant to neural network training. 
A similar layer-wise smoothness assumption is also imposed in \cite{riabinin2025gluon} to capture the geometry of neural networks.

\begin{assumption} \label{ass:objective}
The objective $f$ is layer-wise $L$-smooth with constants $L:= (L_1,\dots,L_{p})\in \R_{+}^{p}$, i.e., for all $\ell=1,\dots,p$, $X=[X_1,\dots,X_{p}]$, and $Y=[Y_1,\dots,Y_p]$, $\|\nabla_{\ell} f(X)-\nabla_{\ell} f(Y)\|_{(\ell)*}\leq L_{\ell}\|X_{\ell}-Y_{\ell}\|_{(\ell)}$.
\end{assumption}

Our second assumption states that the stochastic gradient oracle is unbiased and the layer-wise gradient noise is almost surely bounded both above and below in the dual space.

\begin{assumption} \label{ass:noise}
(i) The stochastic gradient oracle is unbiased, i.e., $\E[\nabla F(X,\xi) \mid X] = \nabla f(X)$. (ii) It holds with probability one for all $\ell$ that $\ubar{\sigma}_{\ell} \leq \|\nabla_{\ell} F(X,\xi)-\nabla_{\ell} f(X)\|_{(\ell)*}\leq \bar{\sigma}_{\ell}$ with $\ubar{\sigma}_{\ell}\geq 0$.
\end{assumption}

Compared to the standard bounded variance assumption (used for expectation-based analysis) or the almost surely bounded-noise assumption (used for high-probability analysis) in stochastic optimization, \cref{ass:noise} additionally requires that the stochastic gradient noise is almost surely lower bounded. A similar assumption is also made in \citep{gong2025adaptive}. Specifically, the empirical noise lower bound is $\ubar{\sigma}_{\ell}=0.01$, as shown in \cref{fig:noise_heterogeneity}.
In the noisy setting, we assume $0<\ubar{\sigma}_{\ell}\leq \bar{\sigma}_{\ell}$, while in the noiseless setting we have $\bar{\sigma}_{\ell} = \ubar{\sigma}_{\ell} = 0$. Note that in practice, we are always in the noisy setting where $0<\ubar{\sigma}_{\ell}\leq \bar{\sigma}_{\ell}$, as illustrated in Figure~\ref{fig:noise_heterogeneity}.
% In the noiseless setting, $\bar{\sigma}_{\ell} = \ubar{\sigma}_{\ell} = 0$. 
From a technical perspective, this assumption is crucial for establishing a tight lower bound on $\alpha_t^{\ell}/\alpha_t^{m}$. 
% we would like to clarify that in the noisy setting we assume $0<\ubar{\sigma}_{\ell}\leq \bar{\sigma}_{\ell}$, whereas in the noiseless setting we set $\ubar{\sigma}_{\ell}= \bar{\sigma}_{\ell}= 0$.
% See \cref{lem:lr_range} for further proof details.
% , ensuring that $\alpha_t^{\ell}/\alpha_t^{m} = O(1)$. 
For further proof details, see \cref{lem:lr_range}. 
% why O(1)? does it depend on variance?

% Compared to standard bounded variance (for expectation analysis) or almost surely bounded gradient noise (for high probability analysis) assumptions in stochastic optimization, \cref{ass:noise} additionally requires that the stochastic gradient noise is almost surely lower bounded. This assumption has been empirically verified in the training process of deep AUC maximization on a 2-layer Transformer XX. In the noiseless case, we have $\bar{\sigma}_x=\ubar{\sigma}_x=0$. Technically, it is crucial for us to establish tight lower bound for $\alpha_t^{\ell}/\alpha_t^{m}$ such that $\alpha_t^{\ell}/\alpha_t^{m}=O(1)$. Please refer to XX for more details.

% \subsection{Main Result}

We now present our main result. Here $C_1, C_2$ (with $C_2 \geq 1$) are the universal constants defined in \cref{lem:norm}, which may depend on the dimension of the model parameters. Depending on the choice of norm constraint, one may select different $C_1, C_2$ to obtain tighter dimension-dependent bounds, rather than applying a uniform choice. A detailed discussion is provided in \cref{rem:C}.

\begin{restatable}{theorem}{thm} \label{thm:muon}
Suppose \cref{ass:objective,ass:noise} hold. Let $\Delta_1 = f(X_1)-f^*$. Set $\beta_1=1-\min\left(\frac{\sqrt{\Delta_1\sum_{\ell}L_{\ell}}}{\sum_{\ell}\bar{\sigma}_{\ell}\sqrt{T}}, 1\right)$, $1 - \min_{\ell}\frac{\ubar{\sigma}_{\ell}^4}{32(2C_2\bar{\sigma}_{\ell}^2-\ubar{\sigma}_{\ell}^2)^2\log(4T/\delta)}\leq \beta_2< 1$, $\eta_{\max}=\sqrt{\frac{\Delta_1\alpha}{\sum_{\ell}L_{\ell}T}}$, and $\eta_{\min}=\eta_{\max}/\kappa_{\eta}$ with $1\leq\kappa_{\eta}\leq O(1)$. With probability at least $1-\delta$, we have
\begin{small}
\begin{align*}
    &\frac{1}{T}\sum_{t=1}^{T}\sum_{\ell=1}^{p}\|\nabla_{\ell} f(X_t)\|_{(\ell)*}
    \lesssim \frac{\sqrt{C_2}(\sum_{\ell}\bar{\sigma}_{\ell})^2}{\sqrt{\Delta_1\sum_{\ell}L_{\ell}T}} \\
    &+ \frac{C_2^{3/2}}{C_1}\sqrt{\log\frac{T}{\delta}}\left(\frac{\sqrt{\Delta_1\sum_{\ell}L_{\ell}}}{\sqrt{T}} + \frac{\sqrt{\sum_{\ell}\bar{\sigma}_{\ell}}(\Delta_1\sum_{\ell}L_{\ell})^{1/4}}{T^{1/4}}\right).
\end{align*}
\end{small}%
\end{restatable}

% \vspace{-0.05in}

\cref{thm:muon} shows that \cref{alg:muon} achieves a convergence rate of $\tilde{O}(1/\sqrt{T} + \sqrt{\sum_{\ell}\bar{\sigma}_{\ell}}/T^{1/4})$. 
Our bound highlights the advantage of adopting a layer-wise noise assumption. It achieves improved noise dependence compared to the $O(1/T^{3/4}+\sum_{\ell}\bar{\sigma}_{\max}/T^{1/4})$\footnote{This rate is obtained by replacing the global variance in \citep{pethick2025training} with the layer-wise variance.} bound established in \citep[Theorem 5.7]{pethick2025training}, where $\bar{\sigma}_{\max}$ is the uniform noise bound assumed in prior work \citep{pethick2025training}. This improvement arises from recognizing that different layers exhibit distinct noise levels during training, and thus should not be treated uniformly. Empirically, we observe noise heterogeneity across layer groups (see \cref{fig:noise_heterogeneity,tbl:noise_stat}). Moreover, we compute that $\sqrt{\sum_{\ell}\bar{\sigma}_{\ell}} = 3.654$, which is significantly smaller than $\sum_{\ell}\bar{\sigma}_{\max} = 18.018$ in the LLaMA-1.1B pretraining on C4 dataset \citep{dodge2021documenting}, thereby validating our theoretical gain in both analysis and experiments.

% such as QK, VO, and MLP
% \citep{pethick2025training}: $O(1/T^{3/4}+\sum_{\ell}\bar{\sigma}_{\max}/T^{1/4})$;
% \citep{gong2025adaptive}: $\tilde{O}(1/\sqrt{T} + \sqrt{\sum_{\ell}\bar{\sigma}_{\max}}/T^{1/4})$;
% $\sum_{\ell}\bar{\sigma}_{\ell} = 13.354$

% significantly

% For $p>1$, our bound exhibits improved noise dependence since XXX. This theoretical gain comes from accounting for layer-wise noise during training: since different layers exhibit distinct noise magnitudes (see \cref{fig:noise_heterogeneity}), they should be treated individually rather than uniformly, as prescribed in \cref{ass:noise}.

% For $p=1$, our rate recovers the complexity for finding a stationary point of $L$-smooth functions established by \citep{gong2025adaptive} for (adaptive) normalized SGD with momentum. For $p>1$, our bound exhibits improved noise dependence since XXX. This theoretical gain comes from accounting for layer-wise noise during training: since different layers exhibit distinct noise magnitudes (see \cref{fig:noise_heterogeneity}), they should be treated individually rather than uniformly, as prescribed in \cref{ass:noise}.

\subsection{Proof Outline} \label{sec:proof_outline}
Here we give an outline of the proof of \cref{thm:muon}, containing the main components of our analysis; see \cref{app:adaptive_muon,app:thm_proof} for full details. The proof sketch below is based on the setting of \cref{thm:muon}. To start, we introduce a few key definitions (with the convention $0/0\coloneqq1$):
% We first introduce the definition of $\kappa_{\sigma}$ and $t_0$, which will be frequently used throughout the analysis. Specifically, we define (with the convention $0/0\coloneqq1$)
\begin{equation} \label{eq:t0_main}
\begin{aligned}
    &\kappa_{\sigma}^{\ell} = 
    \begin{cases}
        \bar{\sigma}_{\ell} / \ubar{\sigma}_{\ell} & \ubar{\sigma}_{\ell}>0 \\
        1 & \bar{\sigma}_{\ell}=0
    \end{cases},
    \quad
    \kappa_{\sigma} = \max_{\ell}\kappa_{\sigma}^{\ell}, \\
    % \quad
    &\bar{\sigma}_{\max} = \max_{\ell}\bar{\sigma}_{\ell}, 
    \quad\text{and}\quad
    t_0 = \frac{\log 2}{\log(1/\beta_2)}.
\end{aligned}
\end{equation}

% $
% \kappa_{\sigma}^{\ell} = 
%     \begin{cases}
%         \bar{\sigma}_{\ell} / \ubar{\sigma}_{\ell} & \ubar{\sigma}_{\ell}>0 \\
%         1 & \bar{\sigma}_{\ell}=0
%     \end{cases},
%     \quad
%     \kappa_{\sigma} = \max_{\ell}\kappa_{\sigma}^{\ell},
%     \quad
%     \bar{\sigma}_{\max} = \max_{\ell}\bar{\sigma}_{\ell},$
%     and
% $
%     t_0 = \frac{\log 2}{\log(1/\beta_2)}.
% $

% The proof sketch below is based on the setting of \cref{thm:muon}, unless stated otherwise.
% \begin{lemma}[Equivalence of norms] \label{lem:norm_main}
% For any two matrix norms $\|\cdot\|_{a}$ and $\|\cdot\|_{b}$, there exists $0< C_1\leq C_2$ such that $C_1\|A\|_{a}\leq \|A\|_{b}\leq C_2\|A\|_{a}$ for all matrices $A\in\R^{m\times n}$.
% \end{lemma}
The following lemma provides high-probability two-sided bounds for the variance tracker $H_t^{\ell}$,
% $\sum_{k=1}^{t}\beta_2^{t-k}(1-\beta_2)\|G_k^{\ell}-\tilde{G}_k^{\ell}\|_{(\ell)*}^2$, 
which in turn allow us to derive tight upper and lower bounds for $\alpha_t^{\ell}$ (numerator of the noise ratio term). 
% The analysis uses Azuma-Hoeffding inequality (see \cref{lem:azuma}). 
The key to the analysis is an application of the Azuma-Hoeffding inequality (see \cref{lem:azuma}).

\begin{restatable}{lemma}{noise} \label{lem:noise}
With probability at least $1-\delta$, for all $\ell$ and $t_0\leq t\leq T$, 
$ 
\frac{\ubar{\sigma}_{\ell}^2(1-\beta_2^t)}{C_2}
    % \leq \sum_{k=1}^{t}\beta_2^{t-k}(1-\beta_2)\|G_k^{\ell}-\tilde{G}_k^{\ell}\|_{(\ell)*}^2
    \leq H_t^{\ell}
    \leq 4\bar{\sigma}_{\ell}^2(1-\beta_2^t).
$
% $\frac{\ubar{\sigma}_{\ell}^2(1-\beta_2^t)}{C_2}\leq H_t^{\ell}\leq 4\bar{\sigma}_{\ell}^2(1-\beta_2^t)$.
% \begin{align*}
%     \frac{\ubar{\sigma}_{\ell}^2(1-\beta_2^t)}{C_2}
%     % \leq \sum_{k=1}^{t}\beta_2^{t-k}(1-\beta_2)\|G_k^{\ell}-\tilde{G}_k^{\ell}\|_{(\ell)*}^2
%     \leq H_t^{\ell}
%     \leq 4\bar{\sigma}_{\ell}^2(1-\beta_2^t).
% \end{align*}
\end{restatable}

With \cref{lem:noise}, we can effectively lower bound the noise ratio term $\alpha_t^{\ell}/\alpha_t^{m}$, which is used to assign layerwise learning rates in line 9 of \cref{alg:muon}, with high probability. Our next lemma shows that $\alpha_t^{\ell}/\alpha_t^{m}$ is both upper and lower bounded throughout training under our assumptions. Consequently, the learning rate $\eta_t^{\ell}$ is bounded on both sides with high probability.

\begin{restatable}{lemma}{lrrange} \label{lem:lr_range}
With probability at least $1-\delta$, for all $\ell$ and $t\leq T$,
% \begin{align}
%     &\frac{\alpha}{\sqrt{\alpha^2+4\bar{\sigma}^2(1-\beta_2^t)}} 
%     \leq \alpha_t^{\ell} 
%     \leq \mathbb{I}(t<t_0) + \frac{\alpha}{\sqrt{\alpha^2+\ubar{\sigma}^2(1-\beta_2^t)/C_2}}\mathbb{I}(t\geq t_0), \label{eq:alpha_tl} \\
%     &\min\left\{\frac{\alpha}{\sqrt{\alpha^2+4\bar{\sigma}^2}}, \frac{\ubar{\sigma}}{2\sqrt{C_2}\bar{\sigma}}\right\}
%     \eqqcolon \alpha_{r}
%     \leq \frac{\alpha_t^{\ell}}{\alpha_t^{m}}
%     \leq 1, \label{eq:alpha_ratio}
% \end{align}
\begin{align}
    % &\frac{\alpha}{\sqrt{\alpha^2+4\bar{\sigma}_{\ell}^2(1-\beta_2^t)}} 
    % \leq \alpha_t^{\ell} 
    % \leq \mathbb{I}(t<t_0) + \frac{\alpha}{\sqrt{\alpha^2+\ubar{\sigma}_{\ell}^2(1-\beta_2^t)/C_2}}\mathbb{I}(t\geq t_0), \label{eq:alpha_tl} \\
    % \textstyle
    \min\left\{\frac{\alpha}{\sqrt{\alpha^2+4\bar{\sigma}_{\max}^2}}, \frac{1}{2\sqrt{C_2}\kappa_{\sigma}}\right\}
    \eqqcolon \alpha_{r}
    \leq \frac{\alpha_t^{\ell}}{\alpha_t^{m}}
    \leq 1, \label{eq:alpha_ratio}
\end{align}
and therefore, with probability at least $1-\delta$, we have $\sqrt{\alpha_{r}}\eta_{\min}\leq \eta_t^{\ell}\leq \eta_{\max}$ for all $\ell$ and $t\leq T$.
\end{restatable}

% \begin{lemma} \label{lem:lr_range}
% With probability at least $1-\delta$, 
% $\min\{\frac{\alpha}{\sqrt{\alpha^2+4\bar{\sigma}_{\max}^2}}, \frac{1}{2\sqrt{C_2}\kappa_{\sigma}}\}
%     \eqqcolon \alpha_{r}
%     \leq \alpha_t^{\ell}/\alpha_t^{m}
%     \leq 1$
% for all $\ell$ and $t\leq T$,
% and therefore, we have $\alpha_{r}\eta_{\min}\leq \eta_t^{\ell}\leq \eta_{\max}$.
% \end{lemma}

We now provide a high-level proof sketch of our main result. See \cref{app:thm_proof} for full proof details.

\begin{proof}[\textbf{Proof sketch of \cref{thm:muon}}]
% The proof follows a similar procedure as \citep[Theorem 1]{cutkosky2020momentum}.
The main novelty in the proof is to leverage the magnitude of $H_t^\ell$ (\cref{lem:noise}) as a surrogate for the true stochastic gradient variance, ensuring that the noise-adaptive layerwise learning rate $\alpha_t^\ell$ has roughly the same magnitude as if the stochastic gradient noise were known (\Cref{lem:lr_range}).
The rest of the proof proceeds similarly to that of \citep[Theorem 1]{cutkosky2020momentum} and \citep{li2025note, shen2025convergence, riabinin2025gluon}. Define $\hat{\epsilon}_t^{\ell} = B_t^{\ell}-\nabla_{\ell} f(X_t)$ and $\epsilon_t^{\ell} = G_t^{\ell}-\nabla_{\ell} f(X_t)$. We begin by applying \cref{lem:lr_range} to the descent lemma (see \cref{lem:descent}), rearranging to obtain:

% \begin{small}
\begin{align*}
    &\sum_{t=1}^{T}\sum_{\ell=1}^{p}\eta_t^{\ell}\|\nabla_{\ell} f(X_t)\|_{(\ell)*}
    \leq \frac{\Delta_1}{\sqrt{\alpha_r}\eta_{\min}} \\
    &\quad+ \sum_{\ell=1}^{p}\left(\frac{2\eta_{\max}}{\sqrt{\alpha_r}\eta_{\min}}\sum_{t=1}^{T}\|\hat{\epsilon}_t^{\ell}\| + \frac{\eta_{\max}^2}{2\sqrt{\alpha_r}\eta_{\min}}L_{\ell}T\right).
    % \Delta_1 + \sum_{t=1}^{T}\sum_{\ell=1}^{p}\left(2\eta_t^{\ell}\|\hat{\epsilon}_t^{\ell}\|_{(\ell)*} + \frac{L_{\ell}}{2}(\eta_t^{\ell})^2\right).
\end{align*}
% \end{small}%
Using $L$-smoothness (\cref{ass:objective}) and standard calculations, we have
% \begin{small}
\begin{align} \notag
    \|\hat{\epsilon}_{t+1}^{\ell}\|_{(\ell)*}
    &\leq \beta_1^{t}\|\hat{\epsilon}_1^{\ell}\|_{(\ell)*} + (1-\beta_1)\left\|\sum_{\tau=0}^{t-1}\beta_1^{\tau}\epsilon_{t-\tau}
    ^{\ell}\right\|_{(\ell)*} \\ \label{eq:hat_eps}
    &\quad+ \eta_{\max}L_{\ell}\sum_{\tau=0}^{t-1}\beta_1^{\tau}.
\end{align}
% \end{small}%
% The remaining challenge is to give high probability bound for $\|\sum_{\tau=0}^{t-1}\beta_1^{\tau}\epsilon_{t-\tau}^{\ell}\|_{(\ell)*}$. 
% equivalence of norms
Next, we apply the concentration inequality introduced in \citep[Lemma 2.4]{liu2023near} to bound $\|\sum_{\tau=0}^{t-1}\beta_1^{\tau}\epsilon_{t-\tau}^{\ell}\|_{F}$, and then use the equivalence of norms (see \cref{lem:norm}) to derive that, with probability at least $1-\delta$,
\begin{align} \notag
    \left\|\sum_{\tau=0}^{t-1}\beta_1^{\tau}\epsilon_{t-\tau}
    ^{\ell}\right\|_{(\ell)*}
    &\leq \frac{1}{C_1}\left\|\sum_{\tau=0}^{t-1}\beta_1^{\tau}\epsilon_{t-\tau}
    ^{\ell}\right\|_{F} \\ \label{eq:sum_eps}
    % \leq \frac{4}{C_1}\sqrt{\log\frac{2T}{\delta}\sum_{\tau=0}^{t-1}(C_2\beta_1^{\tau}\bar{\sigma})^2}
    &\leq \frac{4C_2\bar{\sigma}}{C_1}\sqrt{\frac{\log(2T/\delta)}{1-\beta_1}}.
\end{align}
Substituting \cref{eq:sum_eps} back into \cref{eq:hat_eps} gives the bound for $\|\hat{\epsilon}_{t}^{\ell}\|_{(\ell)*}$. With suitable parameter choices as specified in \cref{thm:muon}, this concludes the proof.
\end{proof}

%% file: main/experiments.tex
% \vspace*{-0.15in}
\section{Experiments}
% \vspace*{-0.05in}

In this section, we present the empirical results in comparison with the state-of-the-art optimizers by pretraining two mainstream transformer architectures GPT \citep{radford2019language} and LLaMA \citep{touvron2023llama} series. The experiment of image classification is deferred to \cref{app:image_class}. We include the ablation studies about learning rate choice and batch size in \cref{app:robustness}, the estimation method of gradient noise in \cref{app:ablation_option}. All experiments were run on $4\times$ NVIDIA H200 graphic cards. 

\subsection{Experimental Settings}
% \textcolor{red}{mention Newton-Schulz iteration somewhere, since the alg only contains LMO}
\paragraph{Baselines} We compare our LANTON with AdamW \citep{loshchilov2017decoupled}, Muon \citep{jordan2024muon}, MARS (short for MARS-AdamW) \citep{yuan2024mars}, SCION \citep{pethick2025training}, D-Muon \citep{liu2025muon}, the layer-wise learning rate algorithm LAMB \citep{you2019large}, and block-wise learning rate algorithm BW-AdamW \citep{wang2025sharpness}. SCION and D-Muon apply the Muon optimizer to matrix parameters in hidden layers (e.g., query, key, value, mlp), and all algorithms use Newton-Schulz iteration \citep{bernstein2024old} to approximately orthogonalize the update matrix, i.e., $UV^{\top}$ in Table \ref{tbl:lmo}.
% Some LMO-based algorithms (including SCION, and D-Muon) apply the Muon optimizer to matrix parameters in hidden layers (e.g., query, key, value, mlp), and all these algorithms use Newton-Schulz iteration \citep{bernstein2024old} to approximately orthogonalize the update matrix, i.e., $UV^{\top}$ in Table \ref{tbl:lmo}.

% \textcolor{red}{ML: MUON is not LMO based algorithm? how can you claim MARS-AdamW is the strongest variant?}

% \vspace*{-0.10in}

\paragraph{Models}
We evaluate on both GPT and LLaMA-style decoders. For GPT we use the HuggingFace GPT2 family: GPT2-small (124M parameters) and GPT2-medium (355M parameters). For LLaMA we configure two sizes: LLaMA-0.5B, LLaMA-1.1B and LLaMA-2B. Unless noted, all models are decoder-only with rotary positional embeddings and RMSNorm/LayerNorm per architecture defaults. Refer to Table \ref{tbl:model_config} for detailed model configuration.

% \vspace*{-0.10in}

% \paragraph{Datasets} We pretrain GPT2 and LLaMA models on three datasets. OpenWebText-100k is used for GPT-small/medium models, and it is a subset of Openwebtext dataset \citep{Gokaslan2019OpenWeb}. As there is no validation set in OpenWebText-100k,  we split $90\%/10\%$ into training/validation set and train models with teacher forcing. MiniPile \citep{kaddour2023minipile} is used for LLaMA-0.5B, where minipile is a subset of the deduplicated Pile corpus \citep{gao2020pile}. C4 (Colossal Clean Crawled Corpus) \citep{dodge2021documenting} is a large-scale English text corpus constructed by aggressively cleaning Common Crawl webpages, and we use it to pretrain LLaMA-1.1B following the standard text-to-token pipeline. All corpora are tokenized with the model’s native tokenizer. 

\paragraph{Datasets} We pretrain GPT-2 and LLaMA models on three datasets. For GPT-small and GPT-medium, we use OpenWebText-100k, a subset of the OpenWebText corpus \citep{Gokaslan2019OpenWeb}. Since OpenWebText-100k does not provide a validation split, we partition the data into $90\%/10\%$ training and validation sets and train the models using teacher forcing. For LLaMA-0.5B, we adopt MiniPile \citep{kaddour2023minipile}, a curated subset of the deduplicated Pile corpus \citep{gao2020pile}. We pretrain LLaMA-1.1B on the C4 (Colossal Clean Crawled Corpus) dataset \citep{dodge2021documenting}, following the standard text-to-token preprocessing pipeline. All datasets are tokenized using the native tokenizer of each model.

% \begin{table}[t]\label{tbl:model_config}
% \centering
% \caption{Model Configurations ($d_{\text{model}}$ denotes the hidden dimension, $d_{\text{FF}}$ denotes the feed-forward dimension, and $n_{\text{head}}$ denotes the number of attention head in transformer).}
% \label{tab:model-sizes}
% \scalebox{0.8}{
% \setlength{\tabcolsep}{15pt}
% \renewcommand{\arraystretch}{1.15}
% \begin{tabular}{l|cccccc}
% \toprule
% \textbf{Model} & \textbf{Size} & $d_{\text{model}}$ & $d_{\text{FF}}$ & $n_{\text{head}}$ & \textbf{depth} \\
% \hline
% GPT-2 (small)     & 124M  & 768  & 3072 & 12 & 12 \\
% GPT-2 (medium)     & 355M  & 1024  & 4096 & 16 & 24 \\
% LLaMA (0.5B)      & 522M  & 1280 & 5120 & 20 & 15 \\
% LLaMA (1.1B)      & 1175M & 2048 & 5632 & 32 & 22 \\
% \bottomrule
% \end{tabular}
% }% end scalebox
% \end{table}

\begin{figure}[t]
  \centering

  % ---------- (a) left group: two plots ----------
  \begin{subfigure}{0.48\linewidth}
    \centering
    \includegraphics[width=0.48\linewidth]{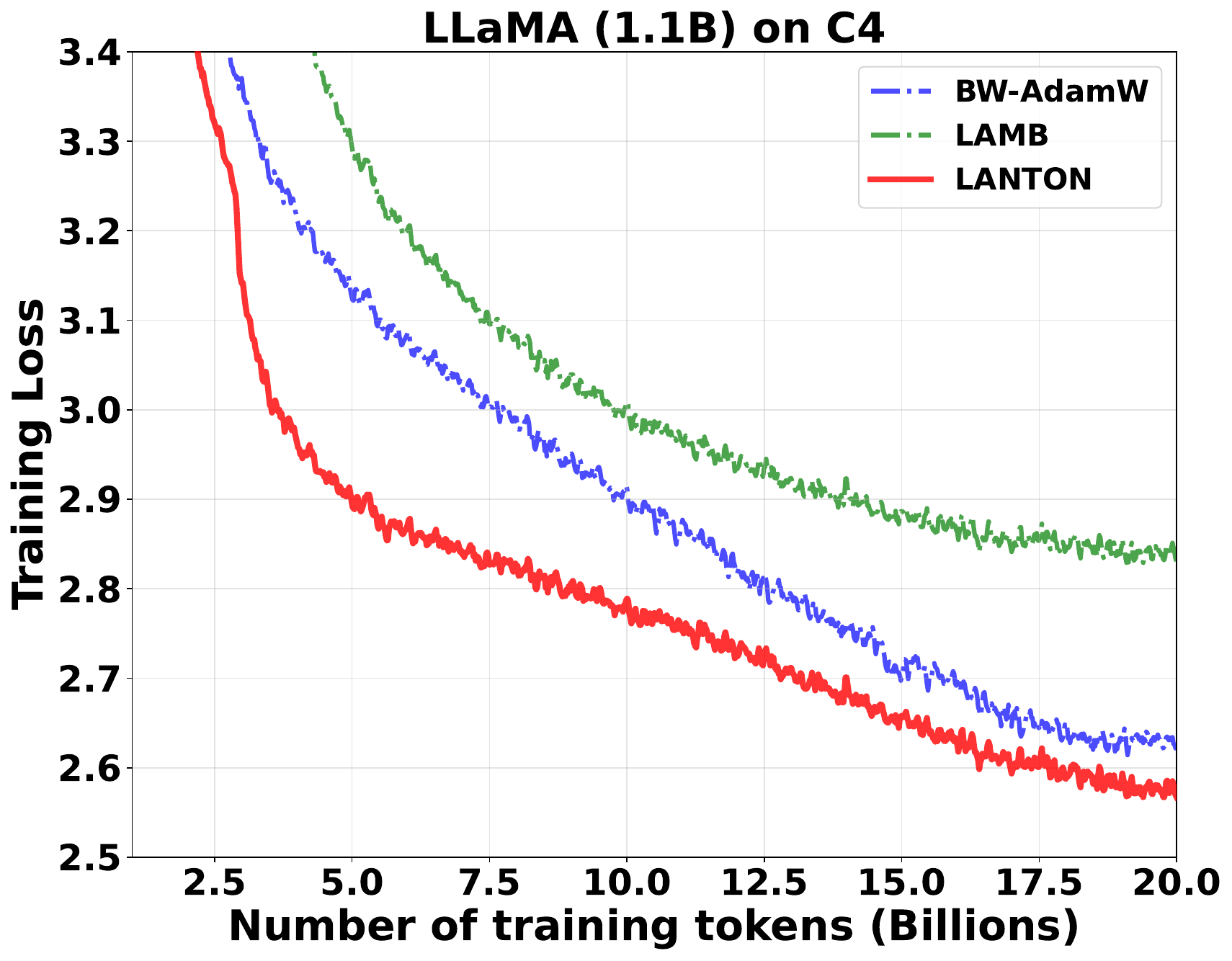}\hfill
    \includegraphics[width=0.48\linewidth]{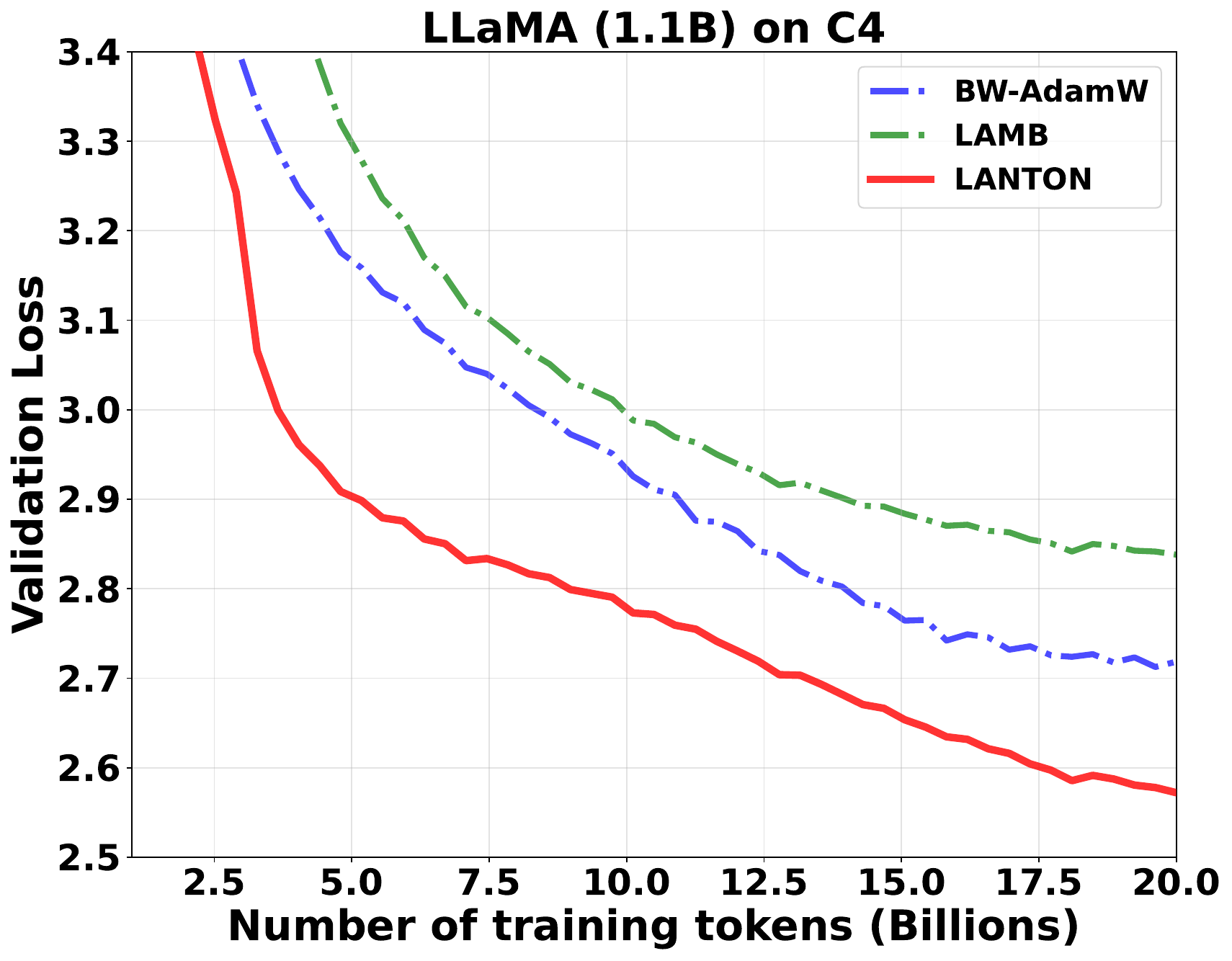}
    \caption{Comparison with layer-/block-wise methods.}
  \end{subfigure}
  \hfill
  % ---------- (b) right group: two plots ----------
  \begin{subfigure}{0.48\linewidth}
    \centering
    \includegraphics[width=0.48\linewidth]{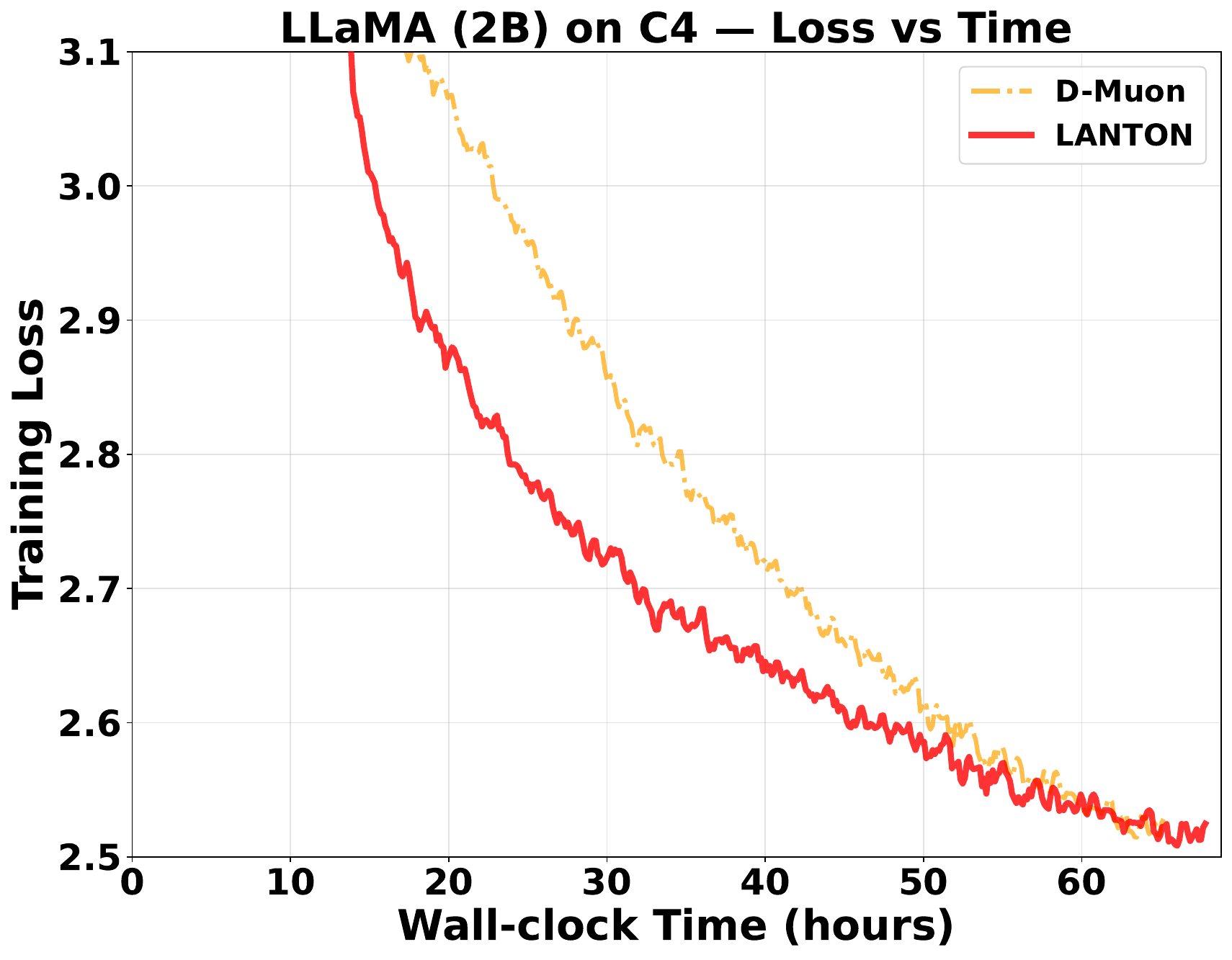}\hfill
    \includegraphics[width=0.48\linewidth]{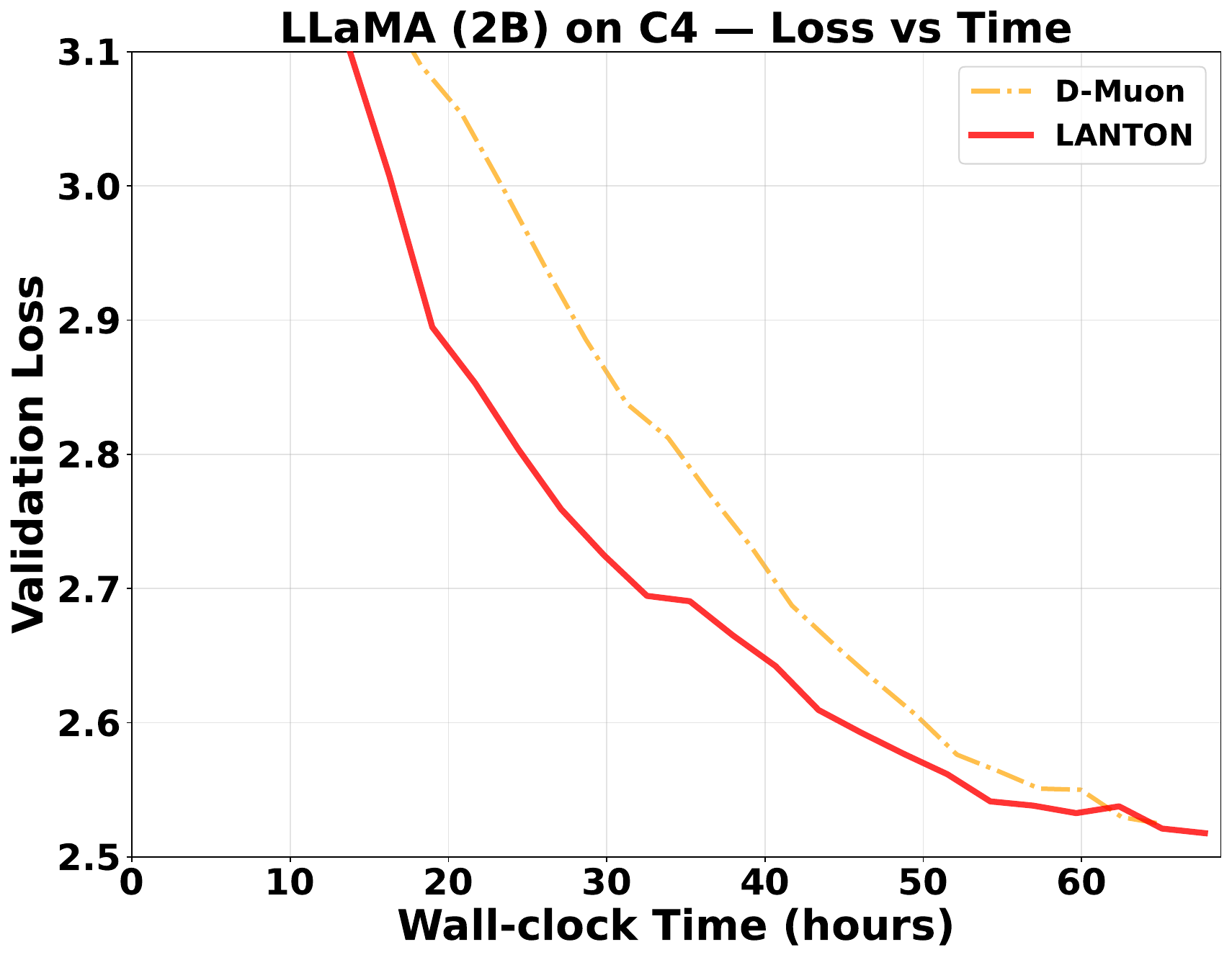}
    \caption{Comparison of running time.}
    \label{fig:running_time_2b}
  \end{subfigure}

  \caption{Training/validation loss on C4 datasets.
  (a) Comparison with algorithms using layer-wise/block-wise learning rates.
  (b) \textsc{LANTON} shows superior runtime performance compared to D-Muon.}
  \label{fig:results_speedup}
      % \vspace{-0.1in}
\end{figure}

\begin{figure}
    \centering
    \includegraphics[width=0.24\linewidth]{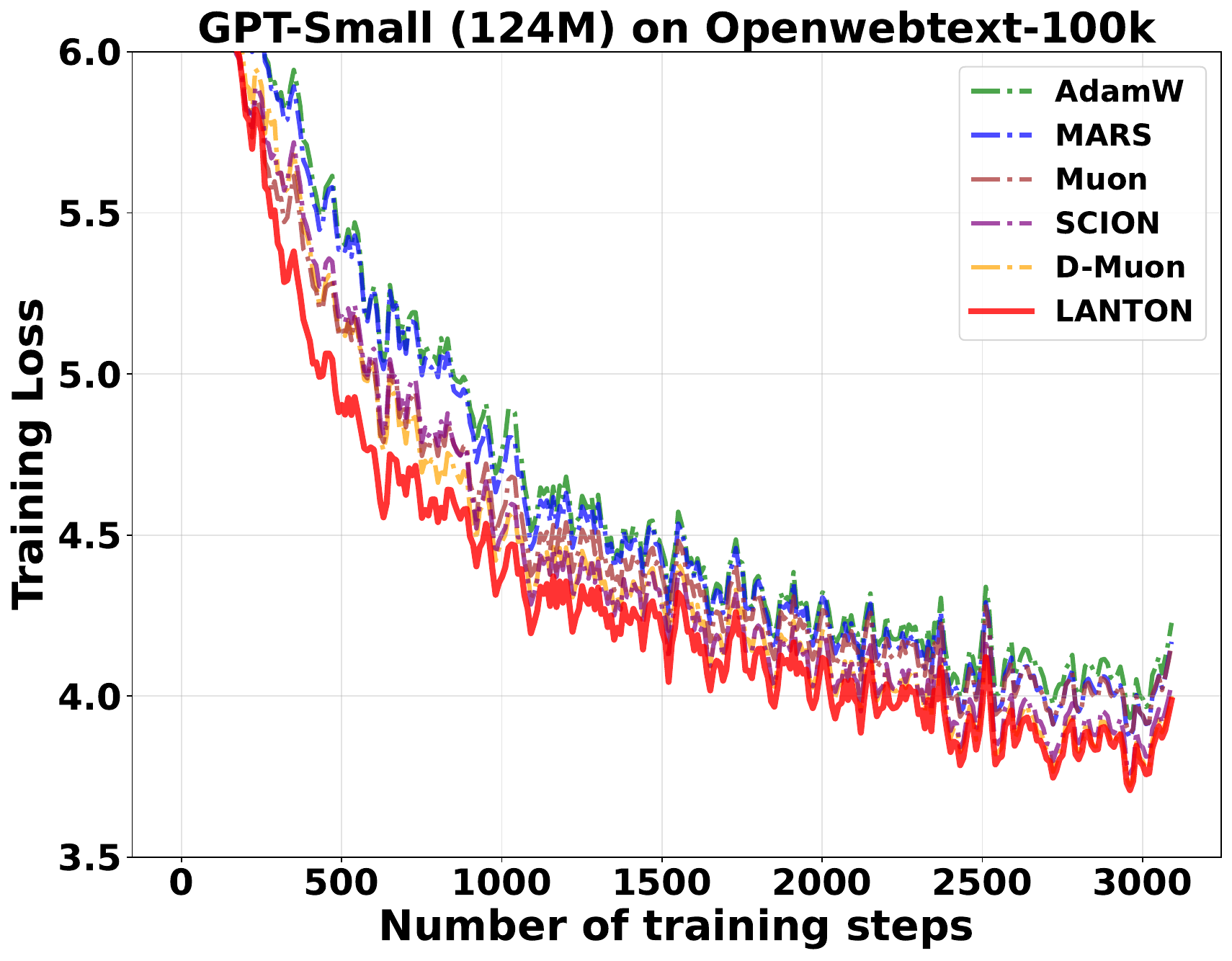} \
    \includegraphics[width=0.24\linewidth]{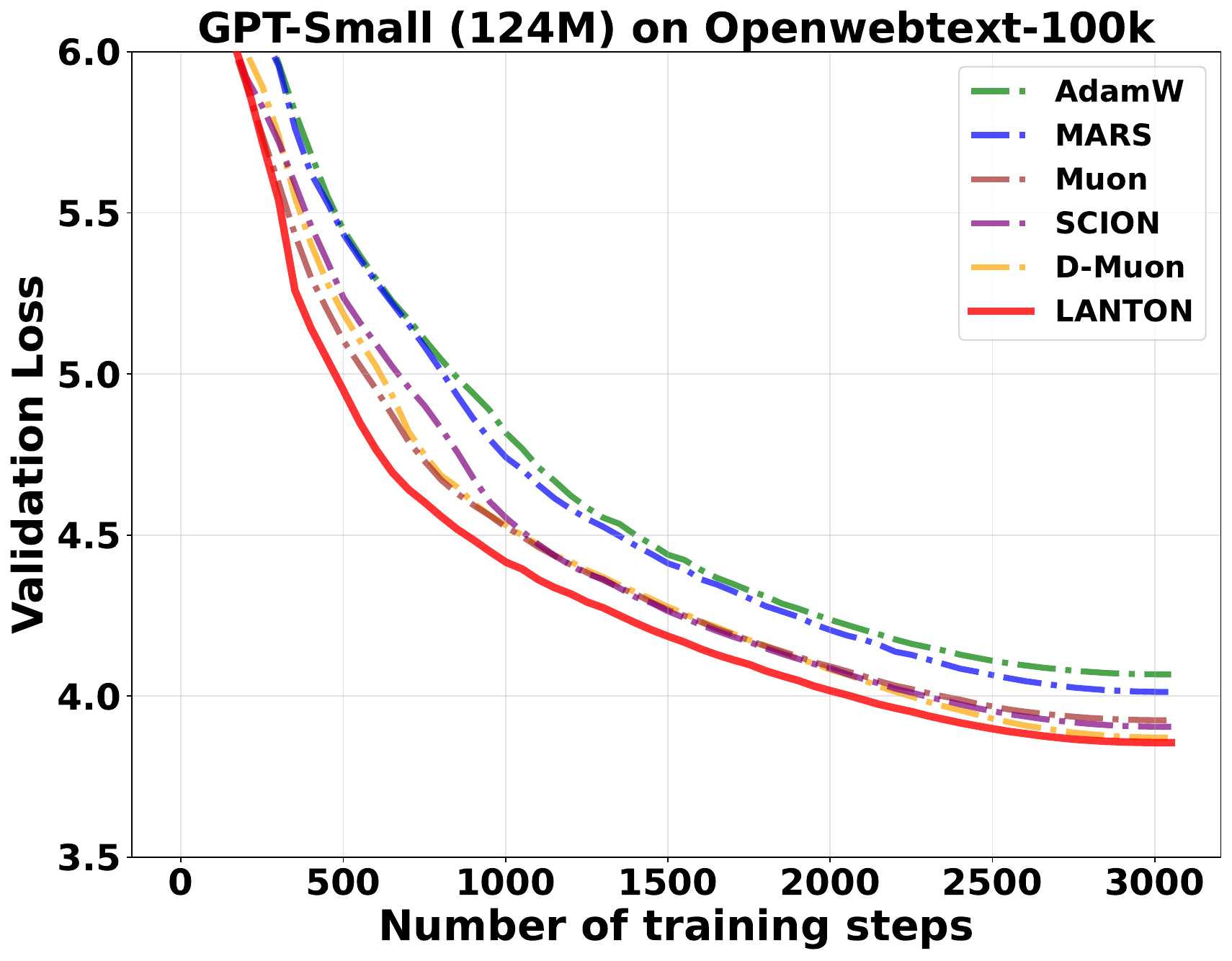} 
    \includegraphics[width=0.24\linewidth]{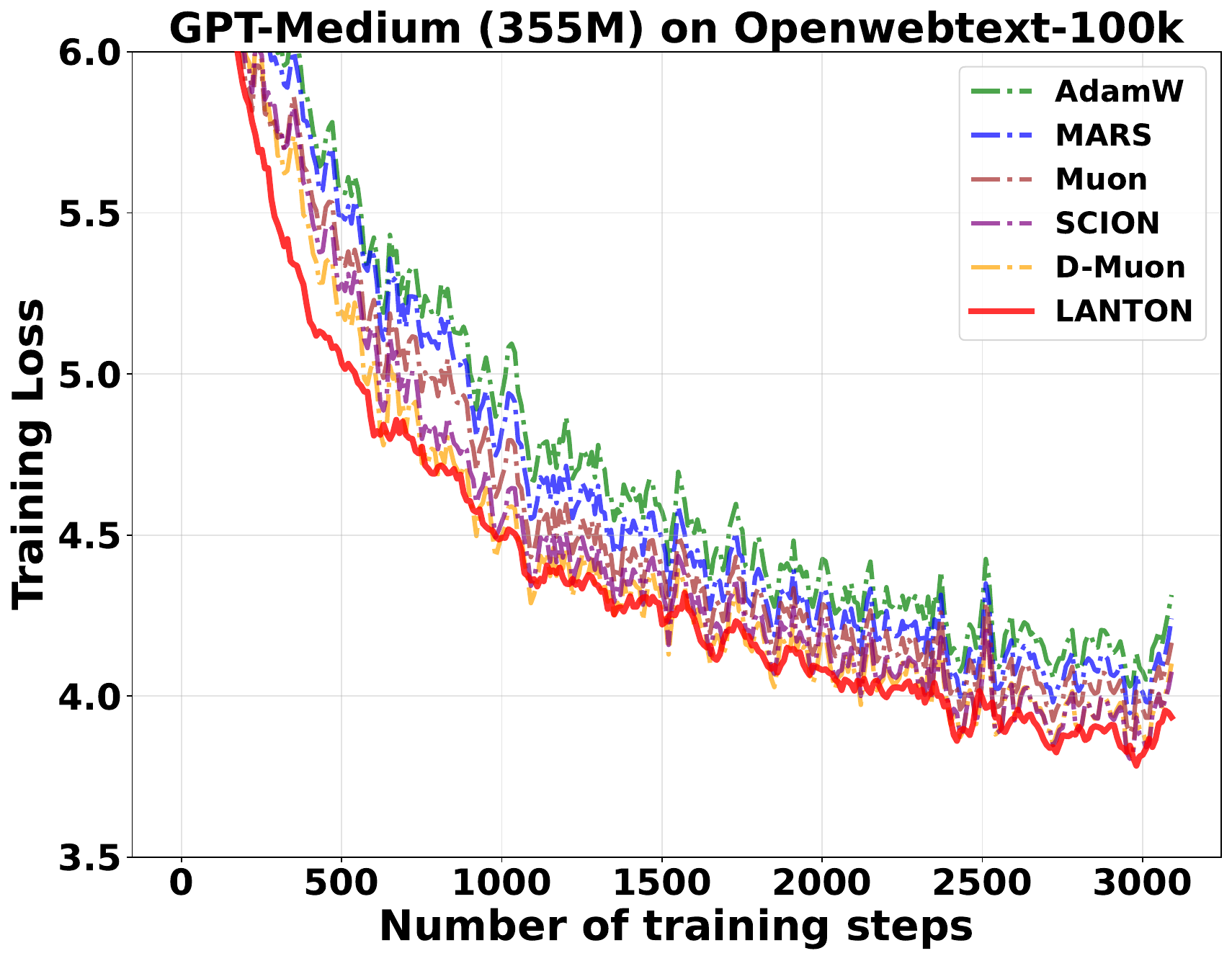} \
    \includegraphics[width=0.24\linewidth]{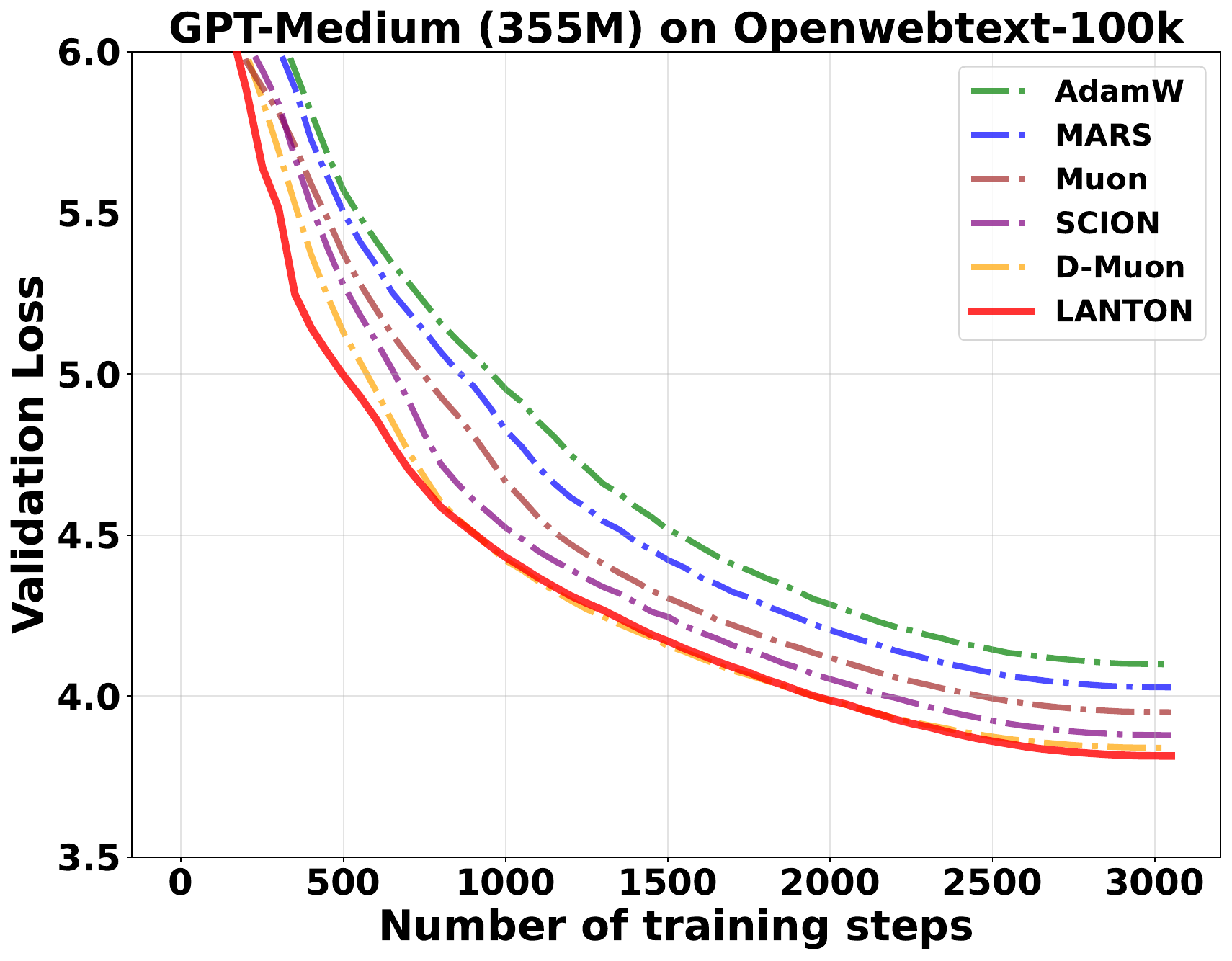}
    % \caption{Training/validation loss on Openwebtext-100k datasets. LANTON consistently shows faster training and convergence.}
    % \vspace{-0.1in}
    \caption{Training/validation loss on Openwebtext-100k datasets.}
    \label{fig:results_gpt}
    % \vspace{-0.1in}
\end{figure}

\begin{figure}
    \centering
    \includegraphics[width=0.245\linewidth]{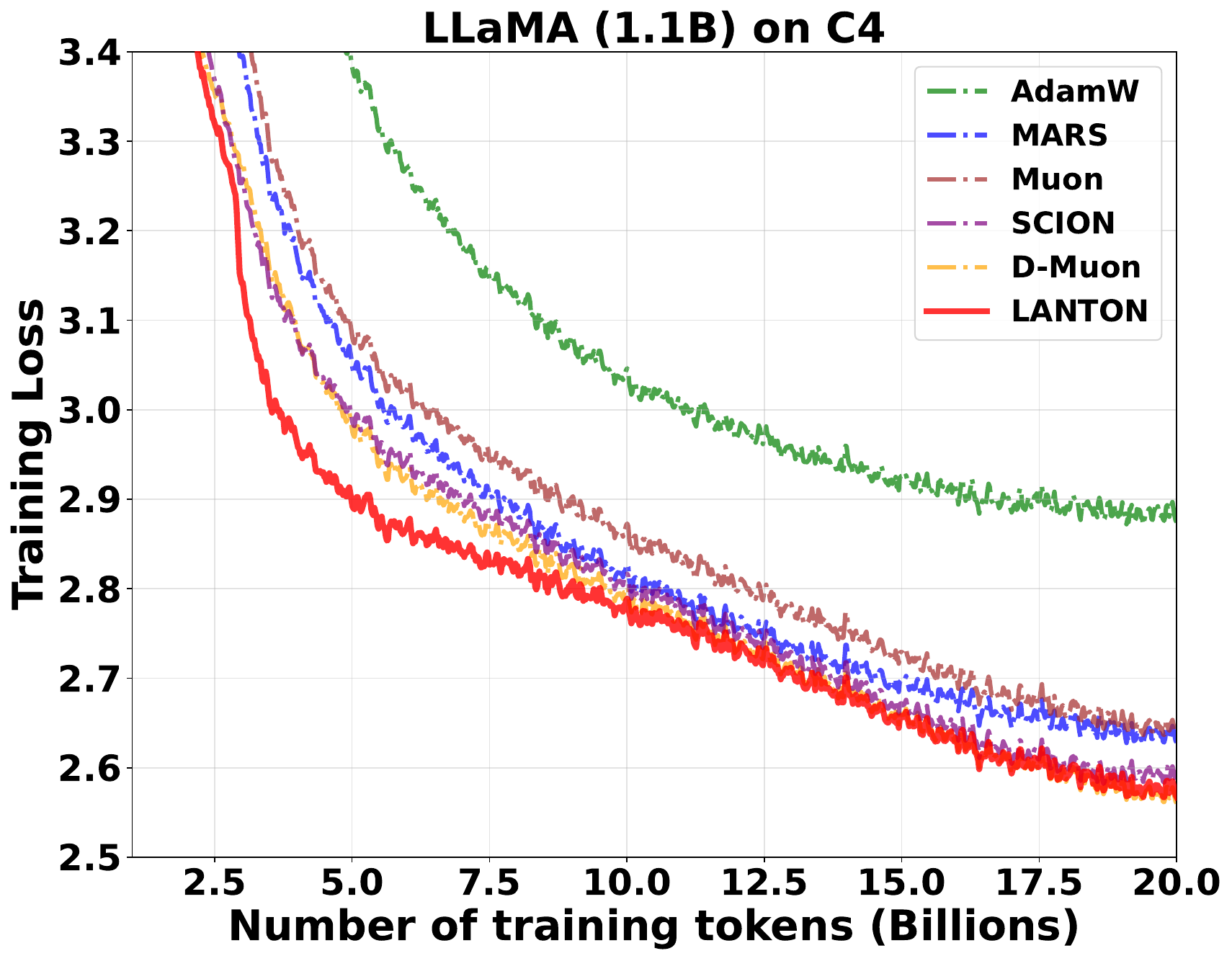}
    \includegraphics[width=0.245\linewidth]{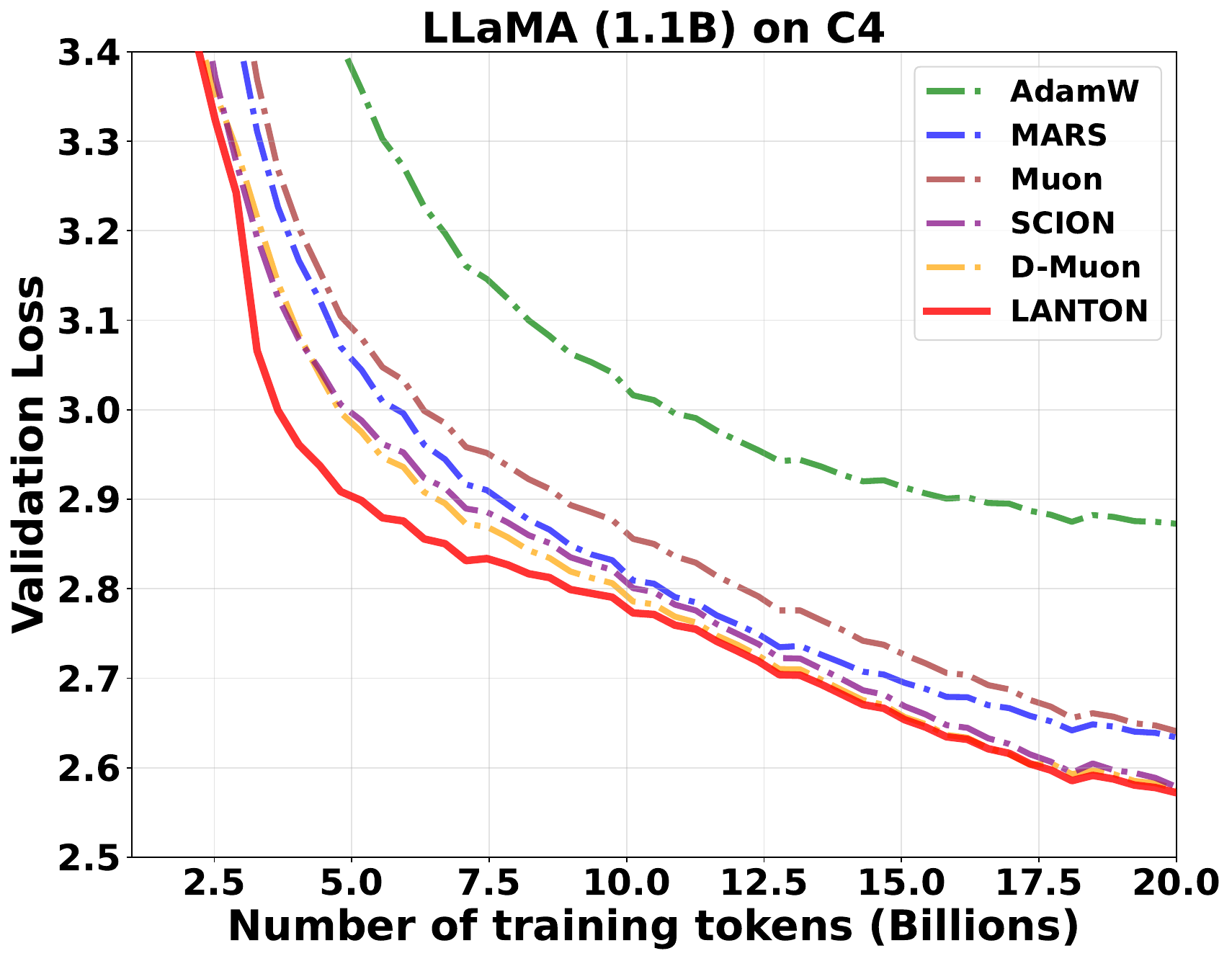}
    \includegraphics[width=0.245\linewidth]{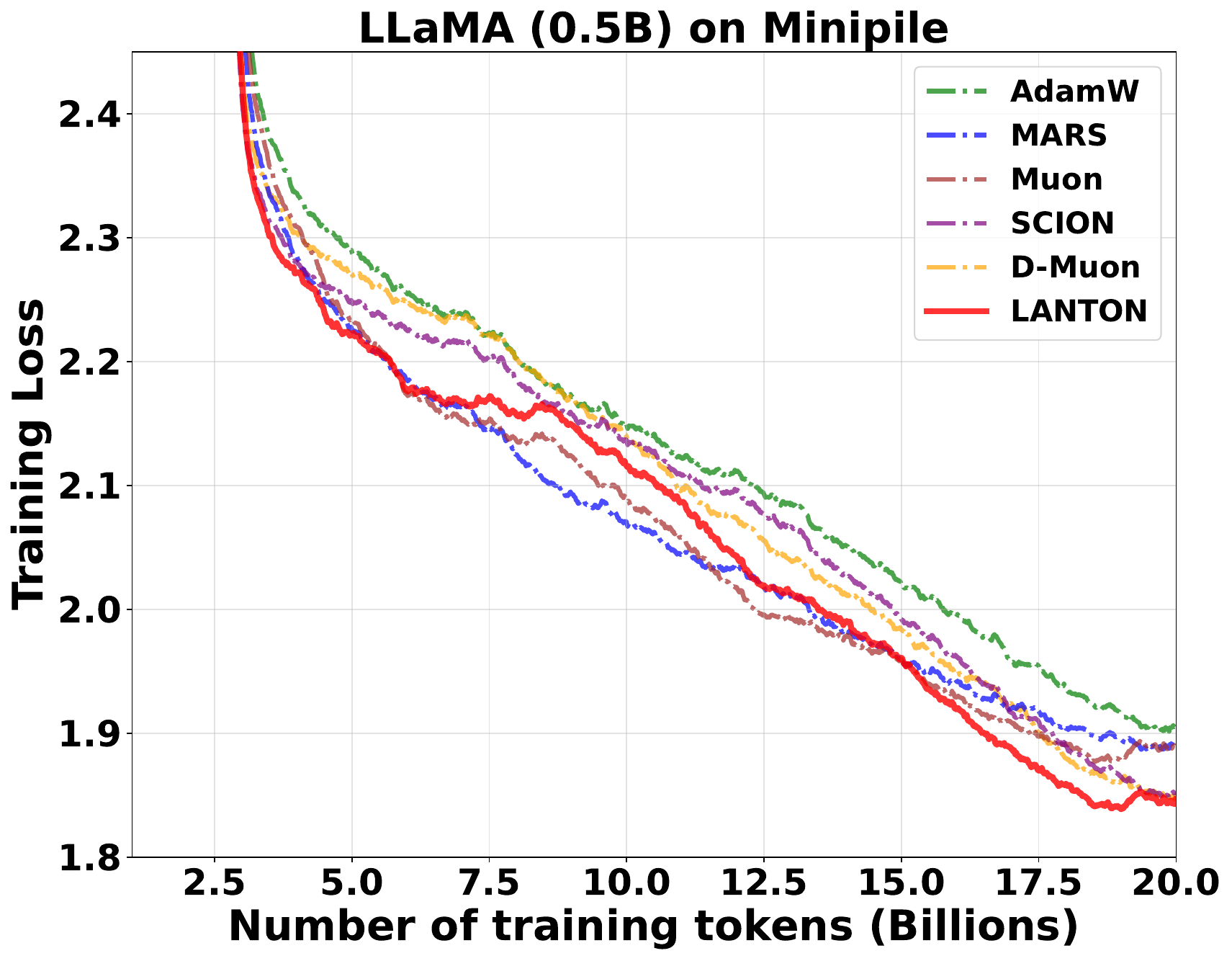}
    \includegraphics[width=0.245\linewidth]{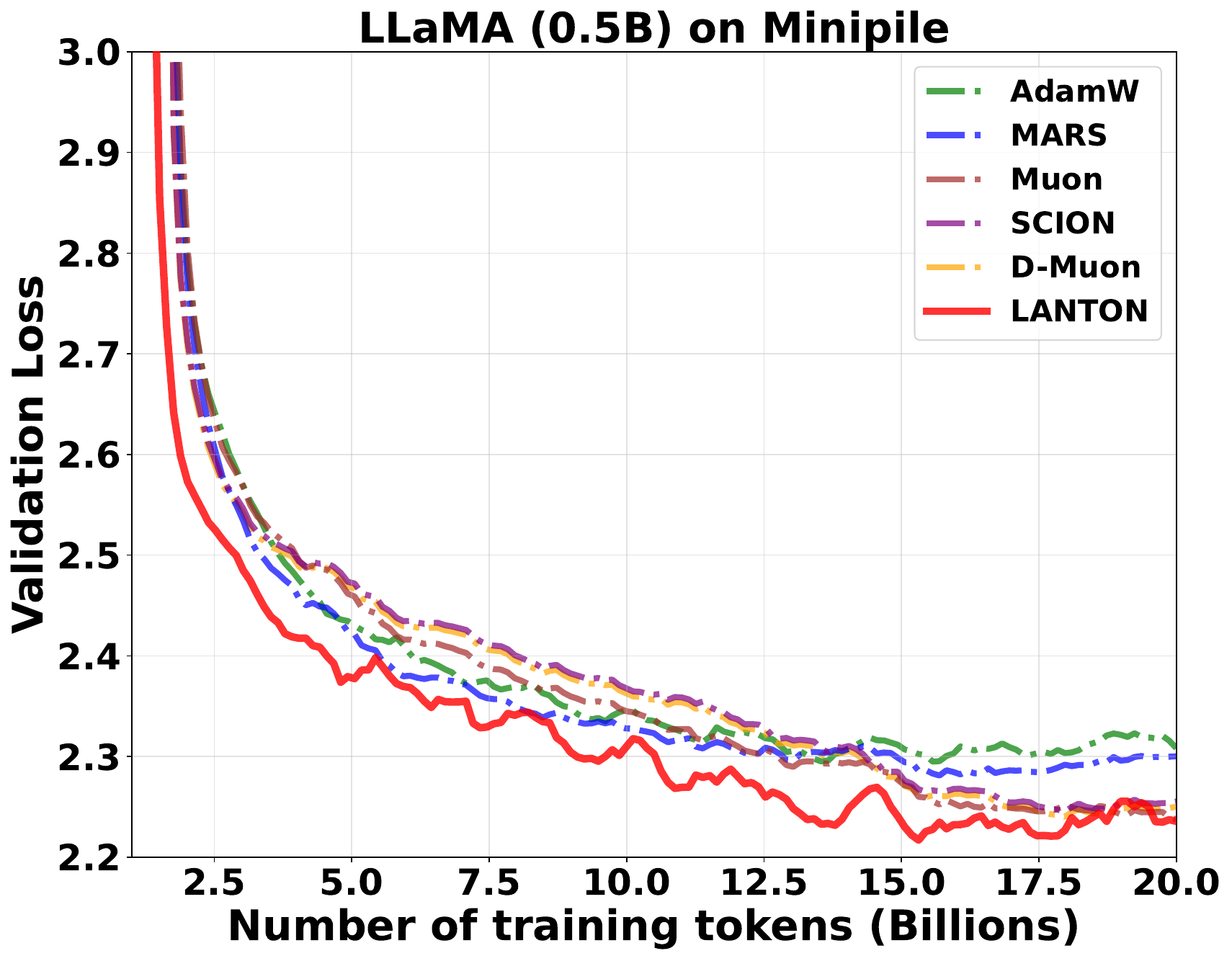}
    % \caption{Training/validation loss on C4 datasets and Minipile. LANTON consistently shows faster training and convergence.}
    \caption{Training/validation loss on C4 and Minipile datasets.}
    \label{fig:results_llama}
    % \vspace{-0.15in}
\end{figure}

\subsection{Training Setup and Results}
\subsubsection{Implementation of LANTON}

We implement LANTON on top of the D-Muon~\citep{liu2025muon}, which carefully adjusts the update magnitudes between hidden layers and non-hidden layers (embedding and LM head layers). Let $\eta_t$ denote the base learning rate at iteration $t$, which is compatible with annealing techniques (e.g., cosine decay).
For layer $\ell$, D-Muon updates the non-hidden layers using AdamW with learning rate $\eta_t$, and the hidden layers parameters $W_\ell\in\mathbb{R}^{d_{\text{out}}^{\ell}\times d_{\text{in}}^{\ell}}$ (i.e., QK, VO, MLP) with a rescaled learning rate $0.2\eta_t \sqrt{\max(d_{\text{in}}^{\ell}, d_{\text{out}}^{\ell})}$. LANTON further rescales the hidden-layer learning rate to $0.2\eta_t \sqrt{\max(d_{\text{in}}^{\ell}, d_{\text{out}}^{\ell})\,\alpha_t^{\ell}/\alpha_t^{m}}$, where $\alpha_t^{m} = \max_{\ell \in \gG_{\ell}} \alpha_t^{\ell}$ and $\gG_{\ell}$ denotes the group of layer $\ell$.
This is the practical instantiation of line~9 in Algorithm~\ref{alg:muon}. In our implementation, there are three layer groups, i.e., \{QK, VO, MLP\}, \{Embedding, LM-Head\}, \{LayerNorm\}, so there are three noise factors $\alpha_t^m$ accordingly. For the first layer group (hidden layers), LANTON applies Newton-Schultz iterations with 5 steps \citep{jordan2024muon} to approximate the LMO update for matrix layers. For embedding and LM head layers, LANTON uses Signum (signed momentum) with a scaled base learning rate $r_1\,\eta_t$. For LayerNorm (vector) parameters, LANTON applies RMS-normalized updates with a scaled base learning rate $r_2\,\eta_t$. Similar to SCION, which requires two distinct update scales for layer groups, LANTON also specifies two update scales $r_1$ and $r_2$, with a base learning rate $\eta_t$.

% \textcolor{red}{define din, dout, define which layer is updated by this, newton-schultz, etc. More details are needed.}

% We implement our algorithm LANTON on the top of a LMO-based algorithm, D-Muon \citep{liu2025muon}, which carefully adjusts the update scale between the AdamW and Muon optimizer. D-Muon assigns a base learning rate $\eta_t$ (With a slight abuse of notation, we omit the layer index), and scales the base learning rate for AdamW update by $1.0$ and Muon updates by $0.2\eta_t\sqrt{\max(d_{\text{in}}, d_{\text{out}})}$. LANTON rescales the Muon update by the factor $0.2\eta_t\sqrt{\max(d_{\text{in}}, d_{\text{out}})\alpha_t/\alpha_t^m}$, which is the practical implementation for line 9 in Algorithm \ref{alg:muon}. For embedding or lm-head layer, LANTON applies sign momentum gradient descent (Signum) with a scaled base learning rate $r_1 \eta_t$. For layernorm (vector parameters), LANTON applies RMS normalization to the parameter update with a scaled base learning rate $r_2 \eta_t$. Similar to SCION optimizer that specifies different update scales to layer groups, LANTON requires to specify three hyperparameters, i.e., the base learning rate $\eta_t$, two scale parameter $r_1$ and $r_2$. 

\subsubsection{GPT2 on Openwebtext} 
We begin with small-scale experiments by pretraining GPT2 from scratch on OpenWebText-100k. All baselines (AdamW, MARS, Muon, SCION, D-Muon), and our method LANTON are trained for a single epoch with context length $512$ and batch size $16$. Unless otherwise specified, for all methods, we fix the random seed to $42$ and weight decay parameter $\gamma=0.1$. We apply a cosine learning-rate schedule to the base step size $\eta_{\max}$ with a linear warmup of 300 steps. After warmup, the per-step learning rate is  $\eta_t = \eta_{\text{min}} + 1/2(\eta_{\text{max}} - \eta_{\text{min}})(1+\cos(\frac{t \pi}{T}))$, where $t$ is the step index, $T$ is the number of training steps, and by default $\eta_{\min}=0$. The detailed hyperparameter settings for every algorithm are summarized in \ref{tbl:hyperparameter_gpt_s} and \cref{tbl:hyperparameter_gpt_m} in Appendix \ref{app:hyperparameter}.

% \textcolor{red}{Describe the results for GPT2 experiments.}
As shown in Figure \ref{fig:results_gpt}, LANTON consistently dominates all baselines (AdamW, MARS, Muon, SCION, D-Muon). Its training loss drops fastest from the earliest iterations and stays below competing methods across the entire training, indicating superior convergence speed. LANTON also achieves the lowest validation loss, exhibit superior performance.

% We conduct small-scale experiments by pretraining GPT-2 from scratch on OpenWebText-100k.
% All baselines (AdamW, MARS, Muon, SCION, and D-Muon) as well as our method LANTON are trained for a single epoch with context length 512 and batch size 16.
% Unless otherwise specified, we fix the random seed to 42 and set the weight decay parameter to $\gamma = 0.1$ for all methods.
% We employ a cosine learning-rate schedule with a linear warmup of 300 steps applied to the base learning rate $\eta_{\max}$.
% After warmup, the learning rate at step $t$ is given by $\eta_t = \eta_{\min} + \tfrac{1}{2}(\eta_{\max} - \eta_{\min})\bigl(1 + \cos(\tfrac{\pi t}{T})\bigr),$
% where $T$ denotes the total number of training steps and $\eta_{\min}=0$ by default.
% Detailed hyperparameter settings for all methods are reported in \cref{tbl:hyperparameter_gpt_s,tbl:hyperparameter_gpt_m} in Appendix~\ref{app:hyperparameter}.

% As shown in \cref{fig:results_gpt}, LANTON consistently outperforms all baselines across training.
% Its training loss decreases most rapidly from the earliest iterations and remains lower than competing methods throughout training, indicating superior convergence speed.
% Moreover, LANTON achieves the lowest validation loss, demonstrating overall superior performance.

% \vspace*{-0.05in}

\subsubsection{LLaMA on C4 and MiniPile} \label{sec:llama_exp}
% We assess large-scale training by pretraining a LLaMA-1.1B model on C4 and a LLaMA-0.5B model on {MiniPile} with a total budget of {20B} training tokens. We use the pretrained LLaMA tokenizer and set the sequence length to {256} on C4 and {512} on MiniPile. The batch size is {1024} for C4 and {300} for MiniPile. We employ a cosine learning rate schedule with a uniform warmup of 1,000 steps for all methods. Full hyperparameter settings for every baseline are reported in \cref{tbl:hyperparameter_c4,tbl:hyperparameter_pile} in Appendix \ref{app:hyperparameter}.

% On C4, LANTON exhibits a significantly steeper loss descent in the early phase and maintains a consistent lead throughout training, while ultimately reaching validation losses comparable to other baselines (see Figure \ref{fig:results_llama}). We track the averaged effective learning rates within each layer group and provide the explanations for training acceleration of LANTON in \cref{app:lr_stats}. 
% On Minipile, although LANTON does not exhibit the lowest loss in the middle of training, it achieves the best final training loss and maintains consistently strong validation performance. 

We evaluate large-scale training by pretraining a LLaMA-1.1B model on C4 and a LLaMA-0.5B model on MiniPile, using a total training budget of 20B tokens.
We adopt the pretrained LLaMA tokenizer, with sequence lengths set to 256 for C4 and 512 for MiniPile, and batch sizes of 1024 and 300, respectively.
All methods use a cosine learning rate schedule with a uniform warmup of 1{,}000 steps.
Complete hyperparameter configurations for all baselines are provided in \cref{tbl:hyperparameter_c4,tbl:hyperparameter_pile} in Appendix~\ref{app:hyperparameter}.

On C4, LANTON demonstrates a substantially faster loss reduction in the early training phase and maintains a consistent advantage throughout training, while converging to validation losses comparable to other baselines (see \cref{fig:results_llama}).
To better understand this acceleration, we analyze the averaged effective learning rates across layer groups in \cref{app:lr_stats}.
On MiniPile, although LANTON does not achieve the lowest loss during mid-training, it attains the best final training loss and consistently strong validation performance.

% \subsection{Evaluation Results}
% Furthermore, We assess our optimizer in both zero-shot and five-shot settings on standard benchmarks-ARC \citep{yadav2019quick}, BoolQ \citep{clark2019boolq}, HellaSwag \citep{zellers2019hellaswag}, OBQA \citep{mihaylov2018can}, PIQA \citep{bisk2020piqa}, WinoGrande \citep{sakaguchi2021winogrande}, and MMLU \citep{hendrycks2020measuring}, using the lm-evaluation-harness framework\citep{gao2021framework}. Table \ref{tbl:eval_results} reports 5-shot results for the pretrained LLaMA-1.1B. Under a fixed training token budget, the model trained with LANTON optimizer achieve stronger downstream performance than those trained with AdamW or Muon, confirming the gains delivered by our approach.

% \begin{table}[t]
% \centering
% \small
% \scalebox{0.85}{
% % \begin{tabular}{lrrrrrrrrrr}
% \begin{tabular}{lcccccccccc}
% \toprule
% Method       & ARC-C & BoolQ & HellaSwag & OBQA  & PIQA & WinoGrande & MMLU & SciQ & Average \\
% \midrule
% AdamW        &      &       &       &   &   &   &   &   & \\
% Muon         &&&&&&&&& \\
% \midrule
% LANTON       & 23.72 & 56.36 & 26.33     & 30.80 & 56.20  & 47.75 & 23.44 & 48.10 & 34.74 \\
% \bottomrule
% \end{tabular}}
% \caption{The 5-shot evaluation results of LLaMA 1.1B with lm-evaluation-harness. The higher score indicates better performance.}
% \label{tbl:eval_results}
% \end{table}

\subsection{Comparison with Algorithms Using Layer-wise/Block-wise Learning Rates}

% To highlight the benefit of our noise-adaptive layer-wise learning rate schedule, we compare with LAMB \citep{you2019large} and the recent block-wise scheme BW-AdamW \citep{wang2025sharpness}. 
% LAMB modifies Adam by applying a per-layer ``trust ratio" to rescale the base learning rate in each layer.
% % $r_\ell=\|W_\ell\|_{F}/\|m_\ell\|_{F}$ to rescale the base learning rate in each layer, where $m_{\ell}$ is the Adam's update. 
% BW-AdamW manually tunes the best block-specific update ratio for each parameter block. Following the original best tuned ratio, we use $r(\text{Emb})=10, r(\text{QK})=8, r(\text{VO})=4, r(\text{MLP/LM-Head})=6, r(\text{Layer norm})=1$ in the experiment. The compared training and validation curves are presented in Figure \ref{fig:results_speedup}(a). LANTON achieves much faster training speed with the same budget of training tokens, and exhibits 0.1 lower validation loss than BW-AdamW. LANTON adapts the noise-adaptive layer-wise learning rate on the fly by monitoring gradient noise, whereas BW-AdamW uses fixed step sizes per parameter group. Moreover, neither baseline explicitly considers the parameter geometry properties.

To highlight the benefit of our noise-adaptive layer-wise learning rate schedule, we compare LANTON with LAMB \citep{you2019large} and the recent block-wise optimizer BW-AdamW \citep{wang2025sharpness}.
LAMB extends Adam by rescaling the base learning rate in each layer using a layer-wise trust ratio, while BW-AdamW relies on manually tuned, fixed update ratios for different parameter blocks.
Following the best-tuned configuration reported in the original work, we use
$r(\text{Emb})=10$, $r(\text{QK})=8$, $r(\text{VO})=4$, $r(\text{MLP/LM-Head})=6$, and $r(\text{LayerNorm})=1$.
The training and validation curves are shown in \cref{fig:results_speedup}(a).
Under the same token budget, LANTON achieves substantially faster training speed and attains a validation loss that is 0.1 lower than BW-AdamW.
Unlike BW-AdamW, which employs fixed step sizes per parameter group, LANTON adaptively adjusts layer-wise learning rates on the fly by monitoring gradient noise.
Moreover, neither baseline explicitly accounts for parameter geometry.

\subsection{Running Time}
\label{sec:running_time}
To efficiently approximate the nuclear-norm term $\|G_{t}^{\ell}-\tilde{G}_{t}^{\ell}\|_{*}^{2}$ for hidden-layer gradients (QK, VO, and MLP layers), we employ randomized SVD (R-SVD) \citep{halko2011finding, oh2015fast}. Rather than computing a full SVD, we project $A=G_t^{\ell}-\tilde{G}_t^{\ell}$ onto a low-dimensional random subspace and estimate its leading singular values, which yields an accurate and efficient approximation of the nuclear norm.
This approximation strategy is also used in SCION \cite{pethick2025training} in their implementation \href{https://github.com/LIONS-EPFL/scion/blob/main/examples/airbench/airbench_muon.py#L163}{link}.

To reduce overhead, gradient-noise estimation is performed once every 10 iterations. As shown in \cref{tbl:running_time} in Appendix, this design introduces only a small computational cost: compared with D-Muon, LANTON adds approximately 3 seconds per 10 steps, corresponding to about 0.84 additional training hours ($\sim4\%$ overhead). Moreover, \cref{fig:running_time_2b} shows that LANTON achieves faster early loss reduction on LLaMA-2B pretraining while maintaining a runtime comparable to D-Muon thereafter. Overall, LANTON incurs negligible overhead while matching the runtime efficiency of the state-of-the-art baseline.

\subsection{Robustness to Base Learning Rate Choice}
% \vspace*{-0.05in}
To evaluate sensitivity to the base learning rate, we keep the model (LLaMA-1.1B), dataset (C4), batch size (1024), optimizer settings, and cosine schedule fixed, then train LANTON with various base learning rates $\eta_{\max} \in \{0.001, 0.003, 0.005\}$. We compare against the best tuned D-MUON under the same setup. As shown in \cref{fig:robustness} in Appendix \ref{app:robustness}, we find that for all learning rates except for $\eta_{\max}=0.001$, LANTON consistently achieves equal or lower loss with fewer training tokens, i.e., converges faster. With $\eta_{\max} =0.001$, LANTON's loss still decreases faster for most ($70\%$) of the training trajectory, with the two methods becoming close only toward the end. Overall, LANTON demonstrates robust performance across base learning rates and superior convergence speed in most hyperparameter settings.

%% file: appendix/appendix.tex
% \tableofcontents
% \newpage

% \section{Martingale Concentration Bounds and Basic Inequalities}
\section{Technical Lemmas}
\label{app:martingale}

\input{appendix/martingale}

\section{Proofs of \cref{sec:proof_outline}}
\label{app:adaptive_muon}

\input{appendix/muon}

\section{More experiments}
\subsection{Experiment of Image Classification}
\label{app:image_class}
\input{appendix/airbench}

\subsection{Comparison with Adaptive Variant of Muon}
\label{app:adamuon}
\input{appendix/adamuon}

\section{Noise Heterogeneity}
\label{app:noise_verify}

\input{appendix/assumption_verification}

\section{Hyperparameter Settings}
\label{app:hyperparameter}
\input{appendix/hyperparameter}
\newpage
\section{Robustness} %\textcolor{red}{Combine I and L}
\label{app:robustness}
\subsection{Base Learning Rate Choice}

\input{appendix/robustness}

\subsection{Robustness to Batch Size}
\label{app:ablation_bs}
\input{appendix/batch_size}

% \section{Running Time \textcolor{red}{Replace section 6.4}}
% \label{app:running_time}
% \input{icml2026/appendix/running_time}

\section{Sample Efficiency with Fixed Token Budget}
\label{app:sample_eff}
\input{appendix/sample_efficiency}

\section{Evolution of Effective Learning Rate}
\label{app:lr_stats}
\input{appendix/lr_stats}

\section{Gradient Noise Estimation: Option I vs. Option II}
\label{app:ablation_option}
\input{appendix/options}

\section{License of Models and Datasets} 
\label{app:license}

\input{appendix/license}

% \section{The Use of Large Language Models (LLMs)}
% LLMs are not involved in our research methodology. Their use is limited to polish the writing.

%% file: appendix/martingale.tex
In this section, we state several standard probabilistic and norm-equivalence lemmas without proof.
% In this section, we state several well-known probabilistic and norm-equivalence lemmas without proof.

\begin{lemma}[Azuma-Hoeffding inequality] \label{lem:azuma}
Let $\{Z_t\}_{t\geq 0}$ be a martingale with respect to filtration $\{\gF_t\}_{t\geq0}$. Assume that $|Z_t-Z_{t-1}|\leq c_t$ almost surely for all $t\geq 0$. Then for any fixed $T$, with probability at least $1-\delta$,
\begin{align*}
    |Z_T - Z_0| \leq \sqrt{2\sum_{t=1}^{T}c_t^2\log(2/\delta)}.
\end{align*}
\end{lemma}

\begin{lemma}[{\citep[Lemma 2.4]{liu2023near}}] \label{lem:MDS}
Suppose $X_1, \dots , X_T$ is a martingale difference sequence adapted to a filtration $\gF_1, \dots, \gF_T$ in a Hilbert space such that $\|X_t\|_{F} \leq R_t$ almost surely for some $R_t\geq 0$. Then for any $\delta\in(0,1)$, with probability at least $1-\delta$, for any fixed $t$ we have
\begin{align*}
    \left\|\sum_{s=1}^{t}X_s\right\|_{F} \leq 4\sqrt{\log\frac{2}{\delta}\sum_{s=1}^{T}R_s^2}.
\end{align*}
\end{lemma}

\begin{proof}[Proof of \cref{lem:MDS}]
Since $\|\cdot\|_{F}$ satisfies $\|X+Y\|_{F}^2\leq \|X\|_{F}^2 + \langle \nabla\|X\|_{F}^2, Y \rangle + \|Y\|_{F}^2$ for all $X,Y$, the condition for applying \citep[Lemma 10]{cutkosky2021high} is satisfied, and therefore \citep[Lemma 2.4]{liu2023near} holds.
\end{proof}

\begin{lemma}[Equivalence of norms] \label{lem:norm}
For any two matrix norms $\|\cdot\|_{a}$ and $\|\cdot\|_{b}$, there exists $0< C_1\leq C_2$ (with $C_2\geq 1$) such that $C_1\|A\|_{a}\leq \|A\|_{b}\leq C_2\|A\|_{a}$ for all matrices $A\in\R^{m\times n}$.
\end{lemma}

\begin{remark} \label{rem:C}
In the subsequent analysis, we will use the relationship among Frobenius norm $\|\cdot\|_{F}$, spectral norm $\|\cdot\|_{2}$, and nuclear norm $\|\cdot\|_{\mathrm{nuc}}$. Specifically, for $A\in\R^{m\times n}$ we have
% \begin{itemize}
%     \item $\|A\|_{2}\leq \|A\|_{F}\leq \sqrt{\mathrm{rank}(A)}\|A\|_{2} \implies C_1\leq1, C_2\geq\sqrt{\max\{m,n\}}$.
%     \item $\|A\|_{\mathrm{nuc}}/\sqrt{\mathrm{rank}(A)}\leq \|A\|_{F}\leq \|A\|_{\mathrm{nuc}} \implies C_1\leq1/\sqrt{\max\{m,n\}}, C_2\geq1$.
% \end{itemize}
\begin{itemize}
    \item $\|A\|_{2}\leq \|A\|_{F}\leq \sqrt{\mathrm{rank}(A)}\|A\|_{2} \implies C_1=1, C_2=\sqrt{\max\{m,n\}}$.
    \item $\|A\|_{\mathrm{nuc}}/\sqrt{\mathrm{rank}(A)}\leq \|A\|_{F}\leq \|A\|_{\mathrm{nuc}} \implies C_1=1/\sqrt{\max\{m,n\}}, C_2=1$.
\end{itemize}
\end{remark}

%% file: appendix/muon.tex
% Before proceeding, we introduce the definition of $\kappa_{\sigma}$ and $t_0$, which will be frequently used throughout the subsequent analysis. Specifically, we define (with the convention $0/0\coloneqq1$)
We first recall a few key definitions from \cref{eq:t0_main} in \cref{sec:proof_outline} (with the convention $0/0\coloneqq1$):
% \begin{equation} \label{eq:t0}
%     \kappa_{\sigma} = 
%     \begin{cases}
%         \bar{\sigma} / \ubar{\sigma} & \ubar{\sigma}>0 \\
%         1 & \bar{\sigma}=0
%     \end{cases},
%     \quad\text{and}\quad
%     t_0 = \frac{\log 2}{\log(1/\beta_2)}.
% \end{equation}
\begin{equation} \label{eq:t0}
    \kappa_{\sigma}^{\ell} = 
    \begin{cases}
        \bar{\sigma}_{\ell} / \ubar{\sigma}_{\ell} & \ubar{\sigma}_{\ell}>0 \\
        1 & \bar{\sigma}_{\ell}=0
    \end{cases},
    \quad
    \kappa_{\sigma} = \max_{\ell}\kappa_{\sigma}^{\ell},
    \quad
    \bar{\sigma}_{\max} = \max_{\ell}\bar{\sigma}_{\ell},
    \quad\text{and}\quad
    t_0 = \frac{\log 2}{\log(1/\beta_2)}.
\end{equation}

The following proofs are based on \cref{ass:objective,ass:noise} and the setting of \cref{thm:muon}.
For simplicity, we omit the $\ell$ superscript/subscript whenever the context is clear.

% \begin{lemma} \label{lem:noise}
% Let $t_0= -\log2/\log\beta_2$ and $\beta_2$ satisfy
% % \begin{align} \label{eq:beta2}
% %     1 - \frac{\ubar{\sigma}^4}{32(2C_2\bar{\sigma}^2-\ubar{\sigma}^2)^2\log(2T/\delta)}
% %     \leq \beta_2 
% %     <1.
% % \end{align}
% \begin{align} \label{eq:beta2}
%     1 - \min_{\ell}\frac{\ubar{\sigma}_{\ell}^4}{32(2C_2\bar{\sigma}_{\ell}^2-\ubar{\sigma}_{\ell}^2)^2\log(2T/\delta)}
%     \leq \beta_2 
%     <1.
% \end{align}
% With probability at least $1-\delta$, for all $\ell$ and $t_0\leq t\leq T$,
% % \begin{align*}
% %     \frac{\ubar{\sigma}^2(1-\beta_2^t)}{C_2}
% %     % \leq \sum_{k=1}^{t}c_{t,k}\|G_k-\tilde{G}_k\|_{*}^2
% %     \leq \sum_{k=1}^{t}\beta_2^{t-k}(1-\beta_2)\|G_k^{\ell}-\tilde{G}_k^{\ell}\|_{(\ell)*}^2
% %     \leq 4\bar{\sigma}^2(1-\beta_2^t).
% % \end{align*}
% \begin{align*}
%     \frac{\ubar{\sigma}_{\ell}^2(1-\beta_2^t)}{C_2}
%     % \leq \sum_{k=1}^{t}c_{t,k}\|G_k-\tilde{G}_k\|_{*}^2
%     \leq \sum_{k=1}^{t}\beta_2^{t-k}(1-\beta_2)\|G_k^{\ell}-\tilde{G}_k^{\ell}\|_{(\ell)*}^2
%     \leq 4\bar{\sigma}_{\ell}^2(1-\beta_2^t).
% \end{align*}
% \end{lemma}

\noise*

\begin{proof}[Proof of \cref{lem:noise}]
Consider the case where $0<\ubar{\sigma}\leq \bar{\sigma}$. Denote $c_{t,k}=\beta_2^{t-k}(1-\beta_2)$. By \cref{ass:noise} and Young's inequality,
\begin{align} 
    \notag
    H_t = \sum_{k=1}^{t}c_{t,k}\|G_k-\tilde{G}_k\|_{*}^2
    &\leq 2\sum_{k=1}^{t}c_{t,k}\left(\|G_k-\nabla f(X_k)\|_{*}^2 + \|\tilde{G}_k-\nabla f(X_k)\|_{*}^2\right) \\ \label{eq:noise_upper}
    &\leq 4\bar{\sigma}^2\sum_{k=1}^{t}c_{t,k}
    = 4\bar{\sigma}^2\sum_{k=1}^{t}\beta_2^{t-k}(1-\beta_2)
    = 4\bar{\sigma}^2(1-\beta_2^t).
\end{align}
We proceed to derive high probability lower bound for $\sum_{k=1}^{t}c_{t,k}\|G_k-\tilde{G}_k\|_{F}^2$. 
Denote $\sigma_k^2=\E_{k-1}[\|G_k-\nabla f(X_k)\|_{F}^2]$. Let $Z_k = c_{t,k}(\|G_k-\tilde{G}_k\|_{F}^2-2\sigma_k^2)$, then $\{Z_k\}_{k\geq 1}$ is a martingale difference sequence since
\begin{align*}
    \E_{k-1}[Z_k] 
    &= \E_{k-1}[\|G_k-\tilde{G}_k\|_{F}^2-2\sigma_k^2] \\
    &= \E_{k-1}[\|G_k-\nabla f(X_k)\|_{F}^2 + \|\tilde{G}_k-\nabla f(X_k)\|_{F}^2 - 2\langle G_k-\nabla f(X_k), \tilde{G}_k-\nabla f(X_k) \rangle] - 2\sigma_k^2 \\
    &= 0.
\end{align*}
Using \cref{ass:noise,lem:norm} and Young's inequality, we have $Z_k\geq -2c_{t,k}\sigma_k^2$ and 
\begin{align*}
    % Z_k\geq -2c_{t,k}\sigma_k^2
    % \quad\text{and}\quad
    Z_k\leq c_{t,k}\left(2C_2^2\left(\|G_k-\nabla f(X_k)\|_{*}^2 + \|\tilde{G}_k-\nabla f(X_k)\|_{*}^2\right) - 2\sigma_k^2\right) 
    \leq c_{t,k}(4C_2^2\bar{\sigma}^2-2\sigma_k^2).
\end{align*}
This implies that
\begin{align*}
    |Z_k|
    \leq c_{t,k}\cdot\max\left\{2\sigma_k^2, 4C_2^2\bar{\sigma}^2-2\sigma_k^2\right\}
    = c_{t,k}(4C_2^2\bar{\sigma}^2-2\sigma_k^2),
\end{align*}
where the last equality is due to $C_2\geq 1$ and $\sigma_k\leq \bar{\sigma}$ almost surely.
Then by the Azuma-Hoeffding inequality (\cref{lem:azuma}) and a union bound over $t$, for any $\delta\in(0,1)$, with probability at least $1-\delta$, for all $t\leq T$,
\begin{align} \label{eq:Xs-bound}
    \left|\sum_{k=1}^{t}Z_k\right| 
    \leq \sqrt{2\sum_{k=1}^{t}(c_{t,k}(4C_2^2\bar{\sigma}^2-2\sigma_k^2))^2\log\frac{2T}{\delta}}
    \leq (4C_2^2\bar{\sigma}^2-2\ubar{\sigma}^2)\sqrt{\frac{2(1-\beta_2)}{1+\beta_2}\log\frac{2T}{\delta}}.
\end{align}
Rearranging \cref{eq:Xs-bound} yields that, with probability at least $1-\delta$, for all $t\leq T$,
\begin{align*}
    \sum_{k=1}^{t}c_{t,k}\|G_k-\tilde{G}_k\|_{F}^2
    &\geq 2\sum_{k=1}^{t}c_{t,k}\sigma_k^2 - (4C_2^2\bar{\sigma}^2-2\ubar{\sigma}^2)\sqrt{\frac{2(1-\beta_2)}{1+\beta_2}\log\frac{2T}{\delta}} \\
    &\geq 2\ubar{\sigma}^2(1-\beta_2^t) - (4C_2^2\bar{\sigma}^2-2\ubar{\sigma}^2)\sqrt{\frac{2(1-\beta_2)}{1+\beta_2}\log\frac{2T}{\delta}}.
\end{align*}
By the choice of $\beta_2$ in \cref{thm:muon} and the definition of $t_0$, for all $t\geq t_0$ we have
\begin{align*}
    \frac{4C_2^2\bar{\sigma}^2-2\ubar{\sigma}^2}{\ubar{\sigma}^2}\sqrt{\frac{2(1-\beta_2)}{1+\beta_2}\log\frac{2T}{\delta}} 
    \leq \frac{1}{2}
    \quad\text{and}\quad
    (4C_2^2\bar{\sigma}^2-2\ubar{\sigma}^2)\sqrt{\frac{2(1-\beta_2)}{1+\beta_2}\log\frac{2T}{\delta}} 
    \leq \ubar{\sigma}^2(1-\beta_2^t).
\end{align*}
Therefore, by \cref{lem:norm}, with probability at least $1-\delta$, for all $t_0\leq t\leq T$,
\begin{align} \label{eq:noise_lower}
    \sum_{k=1}^{t}c_{t,k}\|G_k-\tilde{G}_k\|_{F}^2 \geq \ubar{\sigma}^2(1-\beta_2^t)
    \implies
    \sum_{k=1}^{t}c_{t,k}\|G_k-\tilde{G}_k\|_{*}^2 \geq \frac{\ubar{\sigma}^2(1-\beta_2^t)}{C_2^2}.
\end{align}
We conclude the proof by combining \cref{eq:noise_upper,eq:noise_lower} and noting that the results also hold for the case $\ubar{\sigma}=\bar{\sigma}=0$.
\end{proof}

% \begin{lemma} \label{lem:lr_range}
% With probability at least $1-\delta$, for all $\ell$ and $t\leq T$,
% % \begin{align}
% %     &\frac{\alpha}{\sqrt{\alpha^2+4\bar{\sigma}^2(1-\beta_2^t)}} 
% %     \leq \alpha_t^{\ell} 
% %     \leq \mathbb{I}(t<t_0) + \frac{\alpha}{\sqrt{\alpha^2+\ubar{\sigma}^2(1-\beta_2^t)/C_2}}\mathbb{I}(t\geq t_0), \label{eq:alpha_tl} \\
% %     &\min\left\{\frac{\alpha}{\sqrt{\alpha^2+4\bar{\sigma}^2}}, \frac{\ubar{\sigma}}{2\sqrt{C_2}\bar{\sigma}}\right\}
% %     \eqqcolon \alpha_{r}
% %     \leq \frac{\alpha_t^{\ell}}{\alpha_t^{m}}
% %     \leq 1, \label{eq:alpha_ratio}
% % \end{align}
% \begin{align}
%     % &\frac{\alpha}{\sqrt{\alpha^2+4\bar{\sigma}_{\ell}^2(1-\beta_2^t)}} 
%     % \leq \alpha_t^{\ell} 
%     % \leq \mathbb{I}(t<t_0) + \frac{\alpha}{\sqrt{\alpha^2+\ubar{\sigma}_{\ell}^2(1-\beta_2^t)/C_2}}\mathbb{I}(t\geq t_0), \label{eq:alpha_tl} \\
%     &\min\left\{\frac{\alpha}{\sqrt{\alpha^2+4\bar{\sigma}_{\max}^2}}, \frac{1}{2\sqrt{C_2}\kappa_{\sigma}}\right\}
%     \eqqcolon \alpha_{r}
%     \leq \frac{\alpha_t^{\ell}}{\alpha_t^{m}}
%     \leq 1, \label{eq:alpha_ratio}
% \end{align}
% and therefore, with probability at least $1-\delta$, we have $\alpha_{r}\eta_{\min}\leq \eta_t^{\ell}\leq \eta_{\max}$ for all $\ell$ and $t\leq T$.
% % and therefore, with probability at least $1-\delta$, we have $\alpha_{r}^{\ell}\eta_{\min}\leq \eta_t^{\ell}\leq \eta_{\max}$ for all $\ell$ and $t\leq T$.
% \end{lemma}

\lrrange*

\begin{proof}[Proof of \cref{lem:lr_range}]
By \cref{lem:noise}, for all $t_0\leq t\leq T$, it holds with probability at least $1-\delta$ that
\begin{align*}
    \frac{\ubar{\sigma}^2(1-\beta_2^t)}{C_2^2}
    \leq \sum_{k=1}^{t}\beta_2^{t-k}(1-\beta_2)\|G_k-\tilde{G}_k\|_{*}^2
    \leq 4\bar{\sigma}^2(1-\beta_2^t).
\end{align*}
Therefore, with probability at least $1-\delta$, for all $\ell$ and $t\leq T$,
\begin{align} \label{eq:alpha_tl}
    \frac{\alpha}{\sqrt{\alpha^2+4\bar{\sigma}^2(1-\beta_2^t)}} 
    \leq \alpha_t^{\ell} 
    \leq \mathbb{I}(t<t_0) + \frac{\alpha}{\sqrt{\alpha^2+\ubar{\sigma}^2(1-\beta_2^t)/C_2^2}}\mathbb{I}(t\geq t_0).
\end{align}
Using \cref{eq:alpha_tl}, we have
\begin{align*}
    \frac{\alpha_t^{\ell}}{\alpha_t^{m}}
    &\geq 
    \frac{\alpha}{\sqrt{\alpha^2+4\bar{\sigma}^2(1-\beta_2^t)}} \left(\mathbb{I}(t<t_0) + \frac{\alpha}{\sqrt{\alpha^2+\ubar{\sigma}^2(1-\beta_2^t)/C_2^2}}\mathbb{I}(t\geq t_0)\right)^{-1} \\
    &=
    \frac{\alpha}{\sqrt{\alpha^2+4\bar{\sigma}^2(1-\beta_2^t)}}\mathbb{I}(t<t_0) + \frac{\sqrt{\alpha^2+\ubar{\sigma}^2(1-\beta_2^t)/C_2^2}}{\sqrt{\alpha^2+4\bar{\sigma}^2(1-\beta_2^t)}}\mathbb{I}(t\geq t_0) \\
    &\geq 
    \frac{\alpha}{\sqrt{\alpha^2+4\bar{\sigma}^2(1-\beta_2^t)}}\mathbb{I}(t<t_0) + \frac{\ubar{\sigma}}{2C_2\bar{\sigma}}\mathbb{I}(t\geq t_0) 
    \geq 
    \min\left\{\frac{\alpha}{\sqrt{\alpha^2+4\bar{\sigma}^2}}, \frac{\ubar{\sigma}}{2C_2\bar{\sigma}}\right\},
\end{align*}
that is (we add back the subscript $\ell$ here),
\begin{align*}
    &\min\left\{\frac{\alpha}{\sqrt{\alpha^2+4\bar{\sigma}_{\ell}^2}}, \frac{\ubar{\sigma}_{\ell}}{2C_2\bar{\sigma}_{\ell}}\right\}
    \eqqcolon \alpha_{r}^{\ell}
    \leq \frac{\alpha_t^{\ell}}{\alpha_t^{m}}
    \leq 1.
\end{align*}
Let $\alpha_r=\min_{\ell}\alpha_{r}^{\ell}$, and recall the definitions of $\bar{\sigma}_{\max}$ and $\kappa_{\sigma}$ in \cref{eq:t0}, then for all $\ell$,
\begin{align*}
    &\min\left\{\frac{\alpha}{\sqrt{\alpha^2+4\bar{\sigma}_{\max}^2}}, \frac{1}{2C_2\kappa_{\sigma}}\right\}
    \eqqcolon \alpha_{r}
    \leq \frac{\alpha_t^{\ell}}{\alpha_t^{m}}
    \leq 1,
\end{align*}
which gives \cref{eq:alpha_ratio}. The proof is completed.
\end{proof}

% \begin{lemma} \label{lem:lr_range}
% With probability at least $1-\delta$, we have $\alpha_{r}\eta_{\min}\leq \eta_t^{\ell}\leq \eta_{\max}$ for all $\ell$ and $t\leq T$, where $\alpha_r$ is defined in \cref{eq:alpha_ratio}.
% \end{lemma}

\section{Proof of \cref{thm:muon}} \label{app:thm_proof}
Before proving \cref{thm:muon}, we first provide a descent lemma for \cref{alg:muon}.

\begin{lemma} \label{lem:descent}
For the update in \cref{alg:muon}, we have
\begin{align*}
    f(X_{t+1}) \leq f(X_t) + \sum_{\ell=1}^{p} \left(-\eta_t^{\ell}\|\nabla_{\ell} f(X_t)\|_{(\ell)*} + 2\eta_t^{\ell}\|B_t^{\ell}-\nabla_{\ell} f(X_t)\|_{(\ell)*} + \frac{L_{\ell}}{2}(\eta_t^{\ell})^2\right).
\end{align*}
Moreover, we have
\begin{align*}
    \sum_{t=1}^{T}\sum_{\ell=1}^{p}\eta_t^{\ell}\|\nabla_{\ell} f(X_t)\|_{(\ell)*}
    \leq f(X_1)-f^* + \sum_{t=1}^{T}\sum_{\ell=1}^{p}\left(2\eta_t^{\ell}\|B_t^{\ell}-\nabla_{\ell} f(X_t)\|_{(\ell)*} + \frac{L_{\ell}}{2}(\eta_t^{\ell})^2\right).
\end{align*}
\end{lemma}

\begin{proof}[Proof of \cref{lem:descent}]
Applying \citep[Lemma 1]{riabinin2025gluon} with $X=X_t$ and $Y=X_{t+1}$,
\begin{align*}
    f(X_{t+1})
    &\leq 
    f(X_t) + \langle \nabla f(X_t), X_{t+1}-X_t \rangle + \sum_{\ell=1}^{p}\frac{L_{\ell}}{2}\|X_{t+1}^{\ell}-X_t^{\ell}\|_{(\ell)}^2 \\
    &= f(X_t) + \sum_{\ell=1}^{p}\left(\langle \nabla_{\ell} f(X_t), X_{t+1}^{\ell}-X_t^{\ell} \rangle + \frac{L_{\ell}}{2}(\eta_t^{\ell})^2\right).
\end{align*}
For the second term, using the update of $X_{t+1}^{\ell}$ and the Cauchy-Schwarz inequality we have
\begin{align*}
    \langle \nabla_{\ell} f(X_t), X_{t+1}^{\ell}-X_t^{\ell} \rangle
    &= 
    \langle B_t^{\ell}, X_{t+1}^{\ell}-X_t^{\ell} \rangle + \langle \nabla_{\ell} f(X_t)-B_t^{\ell}, X_{t+1}^{\ell}-X_t^{\ell} \rangle \\
    &\leq 
    -\eta_t^{\ell}\|B_t^{\ell}\|_{(\ell)*} + \eta_t^{\ell}\|\nabla_{\ell} f(X_t)-B_t^{\ell}\|_{(\ell)*} \\
    &\leq 
    -\eta_t^{\ell}\|\nabla_{\ell} f(X_t)\|_{(\ell)*} + 2\eta_t^{\ell}\|B_t^{\ell}-\nabla_{\ell} f(X_t)\|_{(\ell)*}.
\end{align*}
Therefore, we obtain
\begin{align*}
    f(X_{t+1}) \leq f(X_t) + \sum_{\ell=1}^{p} \left(-\eta_t^{\ell}\|\nabla_{\ell} f(X_t)\|_{(\ell)*} + 2\eta_t^{\ell}\|B_t^{\ell}-\nabla_{\ell} f(X_t)\|_{(\ell)*} + \frac{L_{\ell}}{2}(\eta_t^{\ell})^2\right).
\end{align*}
Rearranging the terms and taking summation over $t$ gives the result.
\end{proof}

% \begin{theorem} \label{thm:muon_app}
% Suppose \cref{ass:objective,ass:noise} hold. Let $\Delta_1 = \max_{\ell}f(X_1^{\ell})-f^*$. Set $\beta_1=1-\alpha$ with $\alpha=\min\left(\frac{\sqrt{\Delta_1\sum_{\ell}L_{\ell}}}{\sum_{\ell}\bar{\sigma}_{\ell}\sqrt{T}}, 1\right)$, $1 - \min_{\ell}\frac{\ubar{\sigma}_{\ell}^4}{32(2C_2\bar{\sigma}_{\ell}^2-\ubar{\sigma}_{\ell}^2)^2\log(4T/\delta)}\leq \beta_2< 1$, $\eta_{\max}=\sqrt{\frac{\Delta_1\alpha}{\sum_{\ell}L_{\ell}T}}$, and $\eta_{\min}=\eta_{\max}/\kappa_{\eta}$ with $1\leq\kappa_{\eta}\leq O(1)$. With probability at least $1-\delta$, we have
% % \begin{align*}
% %     \frac{1}{T}\sum_{t=1}^{T}\sum_{\ell=1}^{p}\|\nabla_{\ell} f(X_t)\|_{(\ell)*}
% %     \lesssim \frac{\sqrt{\Delta_1\sum_{\ell}L_{\ell}}}{\sqrt{T}} + \frac{p\bar{\sigma}^2}{\sqrt{\Delta_1\sum_{\ell}L_{\ell}T}} + \left(1 + \frac{p}{C_1}\sqrt{\log\frac{1}{\delta}}\right)\frac{\sqrt{\bar{\sigma}}(\Delta_1\sum_{\ell}L_{\ell})^{1/4}}{T^{1/4}}.
% % \end{align*}
% % \begin{small}
% \begin{align*}
%     \frac{1}{T}\sum_{t=1}^{T}\sum_{\ell=1}^{p}\|\nabla_{\ell} f(X_t)\|_{(\ell)*}
%     \lesssim \frac{\sqrt{C_2}(\sum_{\ell}\bar{\sigma}_{\ell})^2}{\sqrt{\Delta_1\sum_{\ell}L_{\ell}T}} + \frac{C_2^{3/2}}{C_1}\sqrt{\log\frac{T}{\delta}}\left(\frac{\sqrt{\Delta_1\sum_{\ell}L_{\ell}}}{\sqrt{T}} + \frac{\sqrt{\sum_{\ell}\bar{\sigma}_{\ell}}(\Delta_1\sum_{\ell}L_{\ell})^{1/4}}{T^{1/4}}\right).
% \end{align*}
% % \end{small}%
% % $\eta_{\max}=\sqrt{\frac{\Delta_1}{(\sqrt{n}/\alpha+n)LT}}$
% \end{theorem}

\thm*

\begin{proof}[Proof of \cref{thm:muon}]
Define $\hat{\epsilon}_t^{\ell} = B_t^{\ell}-\nabla_{\ell} f(X_t)$, $\epsilon_t^{\ell} = G_t^{\ell}-\nabla_{\ell} f(X_t)$, and $S(X,Y) = \nabla f(X)-\nabla f(Y)$. 
% Following proofs of \citep[Theorem 1]{cutkosky2020momentum} and \citep[Theorem 2.1]{li2025note}, we have
Check that
\begin{align*}
    \hat{\epsilon}_{t+1}^{\ell}
    &= \beta_1\hat{\epsilon}_t^{\ell} + (1-\beta_1)\epsilon_t^{\ell} + S(X_t^{\ell},X_{t+1}^{\ell}) \\
    &= \beta_1^{t}\hat{\epsilon}_1^{\ell} + (1-\beta_1)\sum_{\tau=0}^{t-1}\beta_1^{\tau}\epsilon_{t-\tau}^{\ell} + \sum_{\tau=0}^{t-1}\beta_1^{\tau}S(X_{t-\tau}^{\ell}, X_{t+1-\tau}^{\ell}).
\end{align*}
Using $L$-smoothness, 
% $\|X_{t+1}^{\ell}-X_t^{\ell}\|_{(\ell)}=\eta_t^{\ell}\|O_t^{\ell}\|_{(\ell)}=\eta_t^{\ell}$, 
$\|S(X_t^{\ell}) - S(X_{t+1}^{\ell})\|_{(\ell)*}\leq L_{\ell}\|X_{t+1}^{\ell}-X_t^{\ell}\|_{(\ell)}=L_{\ell}\eta_t^{\ell}\|O_t^{\ell}\|_{(\ell)}=L_{\ell}\eta_t^{\ell}$, 
and $\eta_t^{\ell}\leq \eta_{\max}$ by \cref{lem:lr_range},
\begin{align*}
    \|\hat{\epsilon}_{t+1}^{\ell}\|_{(\ell)*}
    \leq \beta_1^{t}\|\hat{\epsilon}_1^{\ell}\|_{(\ell)*} + (1-\beta_1)\left\|\sum_{\tau=0}^{t-1}\beta_1^{\tau}\epsilon_{t-\tau}
    ^{\ell}\right\|_{(\ell)*} + \eta_{\max}L_{\ell}\sum_{\tau=0}^{t-1}\beta_1^{\tau}.
\end{align*}
Applying \cref{lem:MDS} with $R_{\tau}=C_2\beta_1^{\tau}\bar{\sigma}_{\ell}$ since $\|\beta_1^{\tau}\epsilon_{t-\tau}^{\ell}\|_{F}\leq C_2\|\beta_1^{\tau}\epsilon_{t-\tau}^{\ell}\|_{(\ell)*}\leq C_2\beta_1^{\tau}\bar{\sigma}_{\ell}$, a union bound over $t$, and \cref{lem:norm}, with probability at least $1-\delta$, for all $t\leq T$,
\begin{align*}
    \left\|\sum_{\tau=0}^{t-1}\beta_1^{\tau}\epsilon_{t-\tau}
    ^{\ell}\right\|_{(\ell)*}
    \leq \frac{1}{C_1}\left\|\sum_{\tau=0}^{t-1}\beta_1^{\tau}\epsilon_{t-\tau}
    ^{\ell}\right\|_{F}
    \leq \frac{4}{C_1}\sqrt{\log\frac{2T}{\delta}\sum_{\tau=0}^{t-1}(C_2\beta_1^{\tau}\bar{\sigma}_{\ell})^2}
    \leq \frac{4C_2\bar{\sigma}_{\ell}}{C_1}\sqrt{\frac{\log(2T/\delta)}{1-\beta_1}}.
\end{align*}
Therefore, observing that $\hat{\epsilon}_1^{\ell} = \epsilon_1^{\ell}$ and plugging in the concentration bound yields
\begin{align*}
    \|\hat{\epsilon}_{t+1}^{\ell}\|_{(\ell)*}
    \leq \beta_1^{t}\bar{\sigma}_{\ell} + \frac{4C_2}{C_1}(1-\beta_1)\bar{\sigma}_{\ell}\sqrt{\frac{\log(2T/\delta)}{1-\beta_1}} + \frac{\eta_{\max}L_{\ell}}{1-\beta_1}.
\end{align*}
Taking summation, with probability at least $1-\delta$ we have
\begin{align} \label{eq:momemtum_bias}
    \sum_{t=1}^{T}\|\hat{\epsilon}_t^{\ell}\|_{(\ell)*}
    \leq \frac{\bar{\sigma}_{\ell}}{1-\beta_1} + \frac{4C_2}{C_1}T\sqrt{1-\beta_1}\bar{\sigma}_{\ell}\sqrt{\log\frac{2T}{\delta}} + \frac{T\eta_{\max}L_{\ell}}{1-\beta_1}.
\end{align}
% Rearranging \cref{lem:descent}, taking summation, and using the definition of $\Delta_1$, 
Recall \cref{lem:descent} and the definitions of $\Delta_1$ and $\hat{\epsilon}_t^{\ell}$, 
% \begin{align*}
%     \sum_{t=1}^{T}\sum_{\ell=1}^{p}\eta_t^{\ell}\|\nabla_{\ell} f(X_t)\|_{(\ell)*}
%     \leq \Delta_1 + \sum_{t=1}^{T}\sum_{\ell=1}^{p}\left(2\eta_t^{\ell}\|B_t^{\ell}-\nabla_{\ell} f(X_t)\|_{(\ell)*} + \frac{L_{\ell}}{2}(\eta_t^{\ell})^2\right).
% \end{align*}
\begin{align*}
    \sum_{t=1}^{T}\sum_{\ell=1}^{p}\eta_t^{\ell}\|\nabla_{\ell} f(X_t)\|_{(\ell)*}
    \leq \Delta_1 + \sum_{t=1}^{T}\sum_{\ell=1}^{p}\left(2\eta_t^{\ell}\|\hat{\epsilon}_t^{\ell}\|_{(\ell)*} + \frac{L_{\ell}}{2}(\eta_t^{\ell})^2\right).
\end{align*}
By \cref{lem:lr_range} and a union bound (with \cref{eq:momemtum_bias}), with probability at least $1-2\delta$,
\begin{align*}
    \sum_{t=1}^{T}&\sum_{\ell=1}^{p}\|\nabla_{\ell} f(X_t)\|_{(\ell)*}
    \leq \frac{\Delta_1}{\sqrt{\alpha_r}\eta_{\min}} + \sum_{\ell=1}^{p}\left(\frac{2\eta_{\max}}{\sqrt{\alpha_r}\eta_{\min}}\sum_{t=1}^{T}\|\nabla_{\ell} f(X_t)-B_t^{\ell}\| + \frac{\eta_{\max}^2}{2\sqrt{\alpha_r}\eta_{\min}}L_{\ell}T\right) \\
    &\leq 
    \frac{\kappa_{\eta}\Delta_1}{\sqrt{\alpha_r}\eta_{\max}} + \sum_{\ell=1}^{p}\left(\frac{2\kappa_{\eta}}{\sqrt{\alpha_r}}\left(\frac{\bar{\sigma}_{\ell}}{1-\beta_1} + \frac{4C_2}{C_1}T\sqrt{1-\beta_1}\bar{\sigma}_{\ell}\sqrt{\log\frac{2T}{\delta}}\right) + \frac{\kappa_{\eta}\eta_{\max}}{\sqrt{\alpha_r}}\left(\frac{2TL_{\ell}}{1-\beta_1}+\frac{L_{\ell}T}{2}\right)\right) \\
    &\leq 
    \frac{\kappa_{\eta}\Delta_1}{\sqrt{\alpha_r}\eta_{\max}} + \frac{2\kappa_{\eta}}{\sqrt{\alpha_r}}\left(\frac{\sum_{\ell}\bar{\sigma}_{\ell}}{1-\beta_1} + \frac{4C_2}{C_1}T\sqrt{1-\beta_1}\sum_{\ell}\bar{\sigma}_{\ell}\sqrt{\log\frac{2T}{\delta}}\right) + \frac{5\kappa_{\eta}\eta_{\max}T\sum_{\ell}L_{\ell}}{\sqrt{\alpha_r}(1-\beta_1)} \\
    &\leq 
    \frac{6\kappa_{\eta}}{\sqrt{\alpha_r}}\sqrt{\frac{\Delta_1\sum_{\ell}L_{\ell}T}{1-\beta_1}} + \frac{2\kappa_{\eta}}{\sqrt{\alpha_r}}\left(\frac{\sum_{\ell}\bar{\sigma}_{\ell}}{1-\beta_1} + \frac{4C_2}{C_1}T\sqrt{1-\beta_1}\sum_{\ell}\bar{\sigma}_{\ell}\sqrt{\log\frac{2T}{\delta}}\right) \\
    &\leq 
    \left(\frac{6\kappa_{\eta}}{\sqrt{\alpha_r}} + \frac{2\kappa_{\eta}}{\sqrt{\alpha_r}}\left(1 + \frac{4C_2}{C_1}\sqrt{\log\frac{2T}{\delta}}\right)\right)\sqrt{\Delta_1\sum_{\ell}L_{\ell}T} + \frac{2\kappa_{\eta}(\sum_{\ell}\bar{\sigma}_{\ell})^2\sqrt{T}}{\sqrt{\alpha_r}\sqrt{\Delta_1\sum_{\ell}L_{\ell}}} \\
    &\quad+ 
    \left(\frac{6\kappa_{\eta}}{\sqrt{\alpha_r}} + \frac{8C_2\kappa_{\eta}}{C_1\sqrt{\alpha_r}}\sqrt{\log\frac{2T}{\delta}}\right)\sqrt{\sum_{\ell}\bar{\sigma}_{\ell}}\left(\Delta_1\sum_{\ell}L_{\ell}\right)^{1/4}T^{3/4},
\end{align*}
% where the fourth inequality uses the choice of $\eta_{\max}$, and the last inequality is due to the choice of $\beta_1$.
where the last two inequalities use the choice of $\eta_{\max}$ and $\beta_1$ as stated in \cref{thm:muon}.
Therefore, we obtain with probability at least $1-2\delta$ that
\begin{align*}
    \frac{1}{T}\sum_{t=1}^{T}\sum_{\ell=1}^{p}\|\nabla_{\ell} f(X_t)\|_{(\ell)*}
    &\leq 
    \left(\frac{6\kappa_{\eta}}{\sqrt{\alpha_r}} + \frac{2\kappa_{\eta}}{\sqrt{\alpha_r}}\left(1 + \frac{4C_2}{C_1}\sqrt{\log\frac{2T}{\delta}}\right)\right)\frac{\sqrt{\Delta_1\sum_{\ell}L_{\ell}}}{\sqrt{T}} + \frac{2\kappa_{\eta}(\sum_{\ell}\bar{\sigma}_{\ell})^2}{\sqrt{\alpha_r}\sqrt{\Delta_1\sum_{\ell}L_{\ell}T}} \\
    &\quad+ 
    \left(\frac{6\kappa_{\eta}}{\sqrt{\alpha_r}} + \frac{8C_2\kappa_{\eta}}{C_1\sqrt{\alpha_r}}\sqrt{\log\frac{2T}{\delta}}\right)\frac{\sqrt{\sum_{\ell}\bar{\sigma}_{\ell}}(\Delta_1\sum_{\ell}L_{\ell})^{1/4}}{T^{1/4}}.
\end{align*}
Recall the definition of $\kappa_{\sigma}$ and $\sqrt{\alpha_r}$ in \cref{eq:alpha_ratio,eq:t0}, with probability at least $1-2\delta$,
\begin{align*}
    \frac{1}{T}\sum_{t=1}^{T}\sum_{\ell=1}^{p}\|\nabla_{\ell} f(X_t)\|_{(\ell)*}
    &\leq 
    \kappa_{\eta}\max\left\{\left(1+\frac{4\bar{\sigma}_{\max}^2}{\alpha^2}\right)^{1/4}, \sqrt{2C_2\kappa_{\sigma}}\right\}
    \left(\left(8 + \frac{8C_2}{C_1}\sqrt{\log\frac{2T}{\delta}}\right)\frac{\sqrt{\Delta_1\sum_{\ell}L_{\ell}}}{\sqrt{T}} \right. \\
    &\quad\left.+ 
    \frac{2(\sum_{\ell}\bar{\sigma}_{\ell})^2}{\sqrt{\Delta_1\sum_{\ell}L_{\ell}T}} + \left(6 + \frac{8C_2}{C_1}\sqrt{\log\frac{2T}{\delta}}\right)\frac{\sqrt{\sum_{\ell}\bar{\sigma}_{\ell}}(\Delta_1\sum_{\ell}L_{\ell})^{1/4}}{T^{1/4}}\right).
\end{align*}
Replacing $\delta$ with $\delta/2$ completes the proof.
\end{proof}

%% file: appendix/airbench.tex
Following airbench setting in \url{https://github.com/KellerJordan/cifar10-airbench} and \url{https://github.com/LIONS-EPFL/scion/tree/main/examples/airbench}, we evaluate LANTON on CIFAR-100 image classification using an 8-layer convolutional neural network (CNN). Since stochastic gradient descent (SGD) generally outperforms AdamW on vision tasks, we follow the prior airbench setup and apply SGD to the norm and bias parameters for both Muon and D-Muon.
LANTON partitions the parameters into two groups: (1) convolutional layers (matrix parameters), and (2) norm-layer and bias parameters. Newton–Schulz iterations are applied to the convolutional layers, while sign momentum is used for the norm and bias parameters. The full hyperparameter configuration is provided in \cref{tbl:hyperparameter_image}. 

As shown in Figure \ref{fig:result_cifar100}, all optimizers eventually reach nearly $100\%$ training accuracy on airbench CIFAR-100. However, LANTON exhibits a significantly faster convergence rate than other baselines: it reaches almost maximal training accuracy by around 70 epochs. More importantly, LANTON consistently achieves the highest validation accuracy, demonstrating that LANTON not only accelerates optimization throughout the training process but also yields superior generalization performance compared to all baselines.

\begin{figure}[!h]
        \centering
        \includegraphics[width=0.3\linewidth]{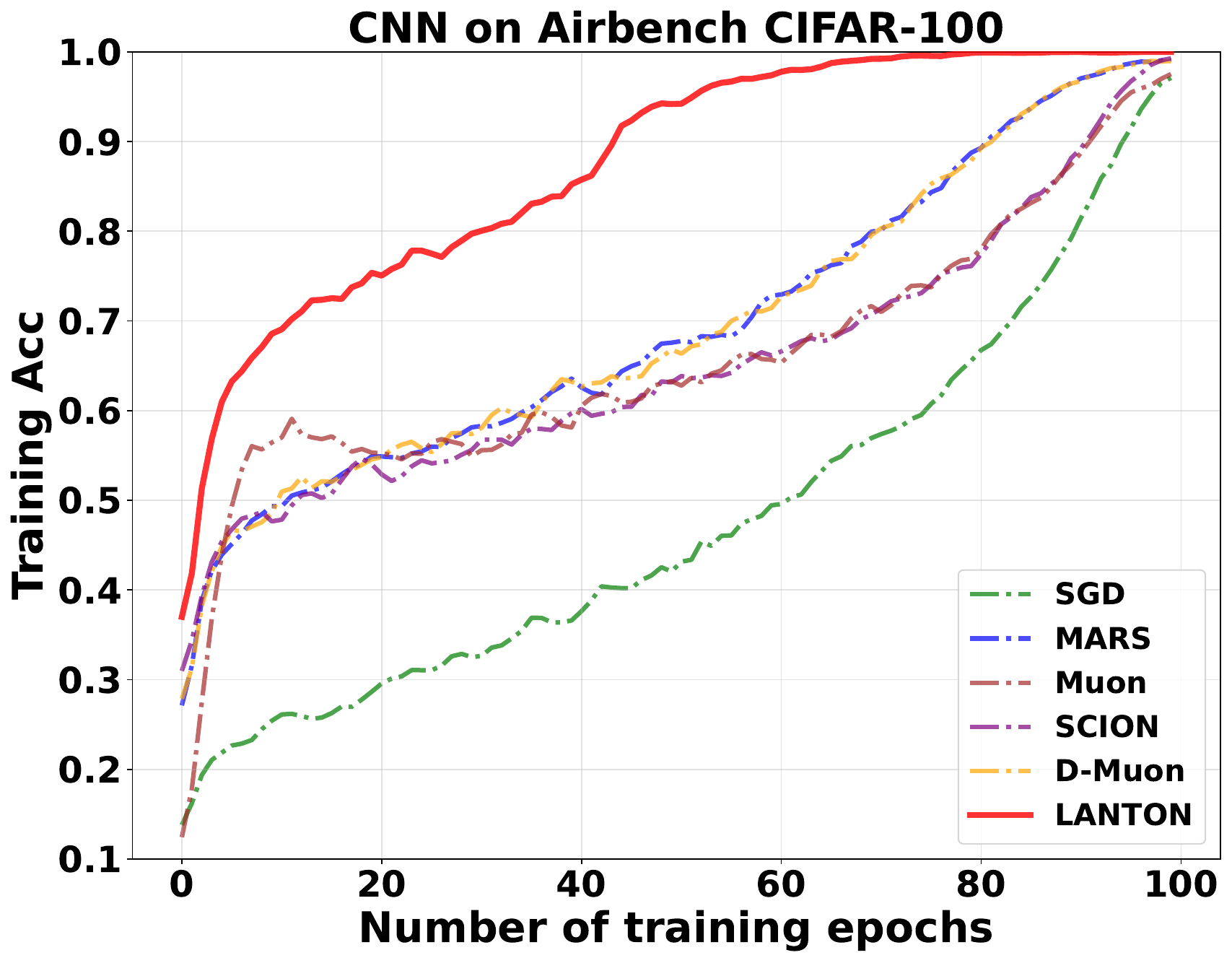}
        \includegraphics[width=0.3\linewidth]{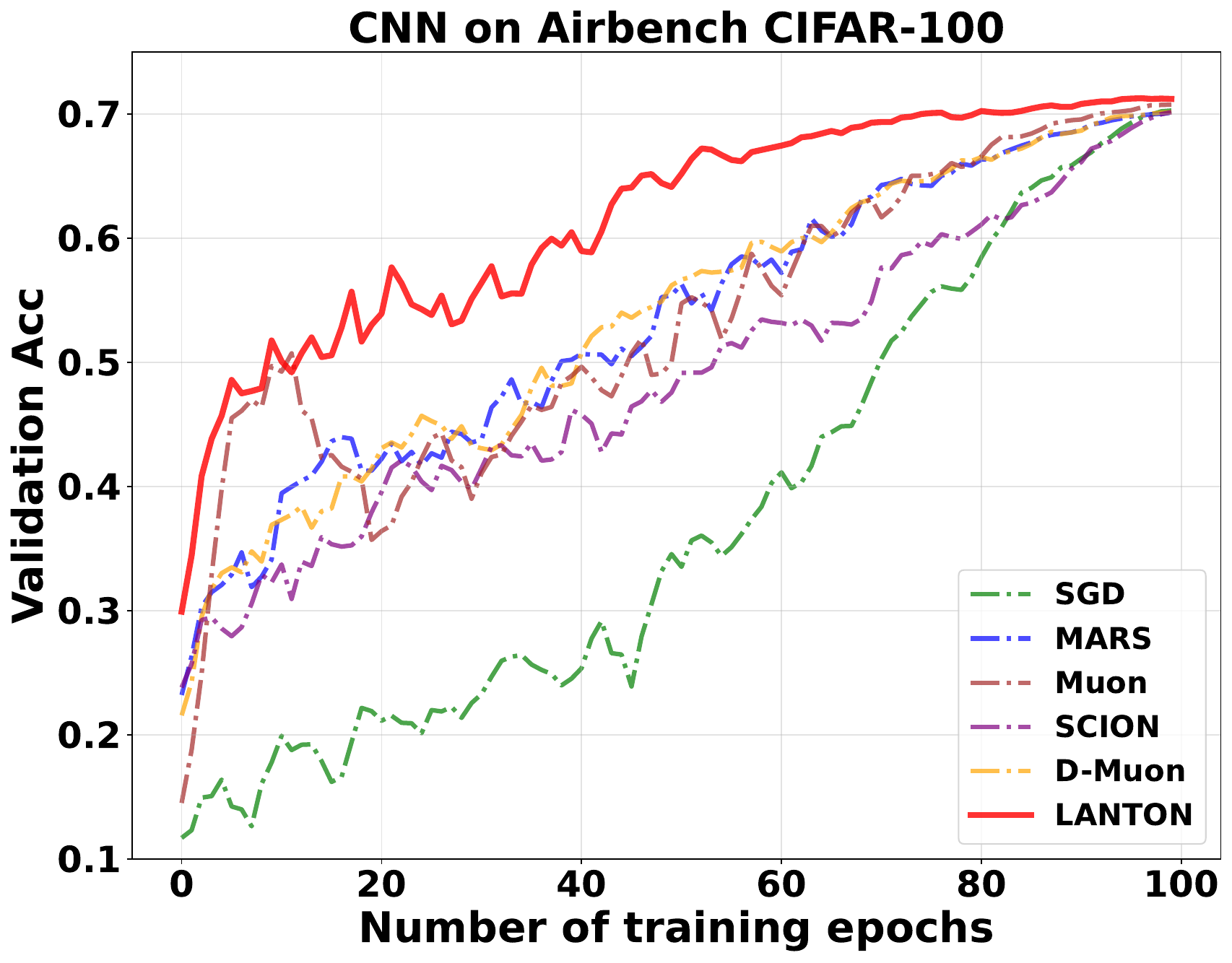}
        \caption{Training/validation accuracy on CIFAR-100.}
        \label{fig:result_cifar100}
\end{figure}

\begin{table}[h]
\centering

\caption{The hyperparameter settings in image classification.}
\label{tbl:hyperparameter_image}
\scalebox{0.8}{
\setlength{\tabcolsep}{6pt}
\renewcommand{\arraystretch}{1.15}
\begin{tabular}{l|ccc}
\toprule
\textbf{Method} & $\eta_{\max}$  & \text{Moment} \\
\hline
\text{SGD}          & $0.1$   & $\beta=0.85$      \\
\text{Muon}           & $0.24$   & $\beta_1=0.6, \beta_2=0.85, \beta_3=0.95$ \\
\text{MARS}           & $0.1$   & $\beta_1=0.9, \beta_2=0.95$       \\
\text{SCION}          & $0.05$   & $\beta=0.5$                     \\
\text{D-Muon}         & $0.1$   & $\beta_1=0.9, \beta_2=0.95, \beta_3=0.95$      \\
\text{LANTON}         & $0.1$   & $\beta_1=0.6, \beta_2=0.85$      \\
\bottomrule
\end{tabular}}
\end{table}

%% file: appendix/adamuon.tex
We additionally compared our method with the recently proposed adaptive variant AdaMuon \citep{si2025adamuon}. Unlike LANTON, AdaMuon does not perform gradient noise estimation; instead, it introduces a momentum-style adaptive scaling on top of Muon and therefore is not noise-adaptive. 

In our experiments in \cref{fig:comparison_adamuon}, AdaMuon achieves slightly better performance than the original Muon but remains worse than LANTON. This matches our design motivation: LANTON is explicitly gradient noise-adaptive, 
adjusting each layer's learning rate based on its noise level. AdaMuon does not estimate 
noise and only plug a second-momentum term to Muon, providing limited gains.

\begin{figure}[!h]
        \centering
        \includegraphics[width=0.3\linewidth]{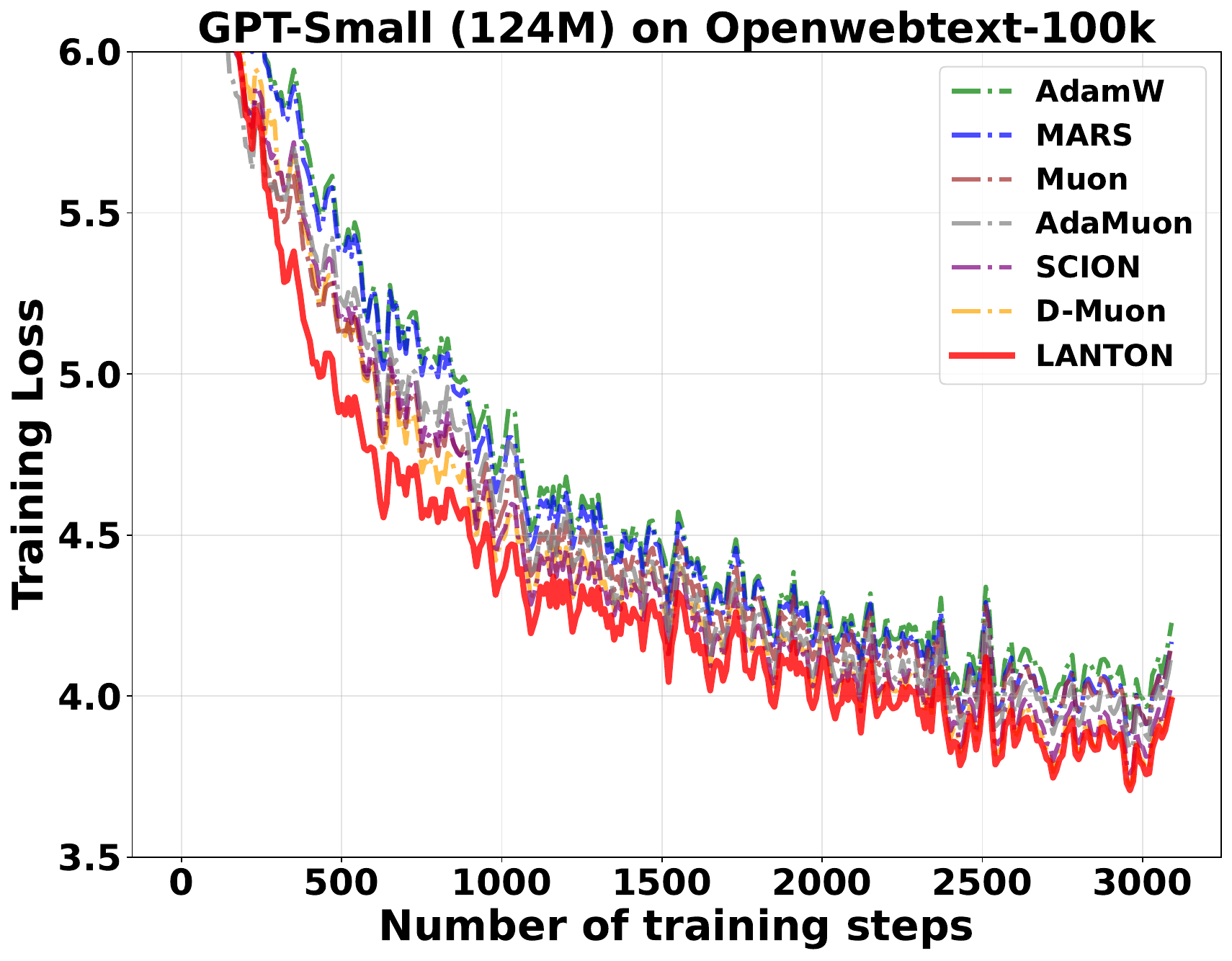}
        \includegraphics[width=0.3\linewidth]{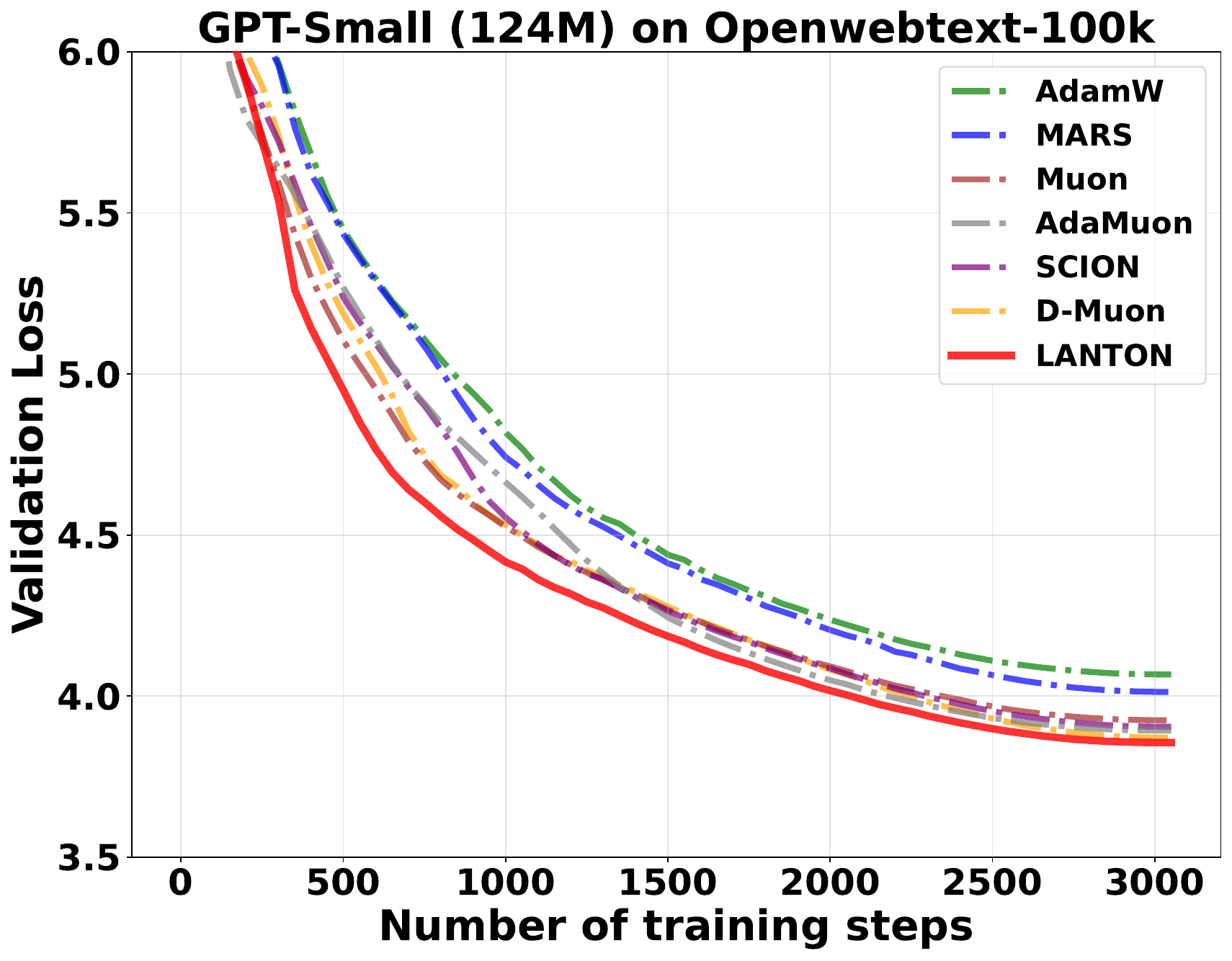}
        \caption{Training and validation loss on Openwebtext-100k.}
        \label{fig:comparison_adamuon}
\end{figure}

%% file: appendix/assumption_verification.tex
\subsection{Implementation Details of \cref{fig:noise_heterogeneity}}
In this section, we provide implementation details of \cref{fig:noise_heterogeneity}. We pretrain LLaMA-1.1B model on C4 dataset for 10k steps, and apply momentum orthogonalized update to the matrix parameters $W_{\ell}\in \mathbb{R}^{d_{\text{out}}\times d_{\text{in}}}$ in the hidden layers (Query, Key, Value, MLP) and AdamW optimizer to the embedding and last layers. We first estimate gradient noise for two parameter groups, formed by matrix shape. For each weight matrix, we compute $\max(d_{\text{out}}, d_{\text{in}})$ and bucket it accordingly. We then aggregate the gradient-noise measure within each bucket over training (e.g., averaging across parameters in the group at each iteration) to obtain group-wise trajectories, which is shown in subfigure \ref{fig:noise_heterogeneity}. Then we measure the layer-wise gradient noise within QK, VO, and MLP layer group in the last three subfigures.   

The stochastic gradient noise is estimated by the nuclear norm (for parameters in Muon optimizer) or $\ell_1\to\ell_1$ operator norm (for parameters in AdamW optimizer) of the difference between the current step's gradient and the previous step's gradient. The implementation follows Option I of line 7 in Algorithm \ref{alg:muon} and line 4 in Table \ref{tbl:lmo}.

\subsection{Noise Magnitude across Different Layer Groups}
We estimate the layer-wise gradient noise within the QK, VO, and MLP layer groups at the midpoint of training (5,000 steps). We find large layer-to-layer disparities within each group, indicating that gradient noise is far from uniform within a group. The statistics is presented in Table \ref{tbl:noise_stat}.
\begin{table}[h]
\centering
\caption{The statistics of stochastic gradient noise in different layer groups of LLaMA.}
\label{tbl:noise_stat}
\scalebox{0.85}{
\setlength{\tabcolsep}{15pt}
\renewcommand{\arraystretch}{1.15}
\begin{tabular}{l|cccc}
\toprule
\textbf{Layer Group} & \textbf{\#Layers} & $\bar{\sigma}$ & $\ubar{\sigma}$ & $\sigma_{\text{mean}}$\\
\hline
\text{QK}            & 44  & 0.026 &0.003 &0.014 \\
\text{VO}            &44 & 0.117 &0.009 & 0.046 \\
\text{MLP}            &66 & 0.107 & 0.018 &0.038 \\
\bottomrule
\end{tabular}
}% end scalebox
\end{table}

\section{Model Configurations}
We pretrain two types of model, GPT2 and LLaMA, the  model configurations are listed in Table \ref{tbl:model_config}.

\begin{table}[h]
\centering
\caption{Model configurations ($d_{\text{model}}$ denotes the hidden dimension, $d_{\text{FF}}$ denotes the feed-forward dimension, and $n_{\text{head}}$ denotes the number of attention head in transformer).}
% \label{tab:model-sizes}
\label{tbl:model_config}
\scalebox{0.85}{
\setlength{\tabcolsep}{15pt}
\renewcommand{\arraystretch}{1.15}
\begin{tabular}{l|cccccc}
\toprule
\textbf{Model} & \textbf{Size} & $d_{\text{model}}$ & $d_{\text{FF}}$ & $n_{\text{head}}$ & \textbf{depth} \\
\hline
GPT-2 (small)     & 124M  & 768  & 3072 & 12 & 12 \\
GPT-2 (medium)     & 355M  & 1024  & 4096 & 16 & 24 \\
LLaMA (0.5B)      & 522M  & 1280 & 5120 & 20 & 15 \\
LLaMA (1.1B)      & 1175M & 2048 & 5632 & 32 & 22 \\
\bottomrule
\end{tabular}
}% end scalebox
\end{table}

%% file: appendix/hyperparameter.tex
\subsection{Hyperparameter Settings in GPT2 Experiments}

We tune the base learning rate $\eta_{\max}$ for each method via a grid search in the range of $[1\times10^{-4},\,5\times10^{-3}]$. For Muon baseline, we additionally sweep a separate base learning rate for non-hidden (embedding/output) layers. All runs use cosine decay from $\eta_{\max}$ down to $\eta_{\min}=0.0$. Muon and D-Muon use three momentum hyperparameters: $(\beta_1,\beta_2)$ for the AdamW auxiliary optimizer and $\beta_3$ for orthogonalized momentum updates. LANTON uses two momentum parameters: $\beta_1$ for the gradient momentum and $\beta_2$ for the gradient noise momentum. All LMO-based methods (SCION, D-Muon, LANTON) apply layer-group learning-rate scaling; for SCION and D-Muon we adopt the best tuned scales reported in their original papers. All the hyperparameter settings are summarized in Table \ref{tbl:hyperparameter_gpt_s} and \ref{tbl:hyperparameter_gpt_m}.

\begin{table}[h]
\centering
\caption{The hyperparameter settings in GPT2-Small experiments.}
\label{tbl:hyperparameter_gpt_s}
\scalebox{0.8}{
\setlength{\tabcolsep}{6pt}
\renewcommand{\arraystretch}{1.15}
\begin{tabular}{l|cccc}
\toprule
\textbf{Method} & $\eta_{\max}$  & \text{Moment} & $\text{Scale}$\\
\hline
\text{AdamW}          & $1\times10^{-4}$   & $\beta_1=0.9, \beta_2=0.95$      & - \\
\text{Muon}           & $(3\times10^{-3}, 3\times10^{-4})$   & $\beta_1=0.9, \beta_2=0.95, \beta_3=0.95$ & - \\
\text{MARS}           & $1\times10^{-3}$   & $\beta_1=0.9, \beta_2=0.95$      & -  \\
\text{SCION}          & $3\times10^{-4}$   & $\beta=0.9$                      &$r_1=50, r_2=3000$ \\
\text{D-Muon}         & $1\times10^{-3}$   & $\beta_1=0.9, \beta_2=0.95, \beta_3=0.95$      &  $r=0.2$ \\
\text{LANTON}         & $5\times10^{-3}$   & $\beta_1=0.95, \beta_2=0.9$      &  $r_1=300, r_2=1.0$ \\
\bottomrule
\end{tabular}
}% end scalebox
\end{table}

\begin{table}[h]
\centering
\caption{The hyperparameter settings in GPT2-Medium experiments.}
\label{tbl:hyperparameter_gpt_m}
\scalebox{0.8}{
\setlength{\tabcolsep}{6pt}
\renewcommand{\arraystretch}{1.15}
\begin{tabular}{l|cccc}
\toprule
\textbf{Method} & $\eta_{\max}$  & \text{Moment} & $\text{Scale}$\\
\hline
\text{AdamW}          & $1\times10^{-4}$   & $\beta_1=0.9, \beta_2=0.95$      & - \\
\text{Muon}           & $(3\times10^{-3}, 3\times10^{-4})$   & $\beta_1=0.9, \beta_2=0.95, \beta_3=0.95$ & - \\
\text{MARS}           & $1\times10^{-3}$   & $\beta_1=0.9, \beta_2=0.95$      & -  \\
\text{SCION}          & $2\times10^{-4}$   & $\beta=0.9$                      &$r_1=50, r_2=3000$ \\
\text{D-Muon}         & $5\times10^{-4}$   & $\beta_1=0.9, \beta_2=0.95, \beta_3=0.95$      &  $r=0.2$ \\
\text{LANTON}         & $3\times10^{-3}$   & $\beta_1=0.95, \beta_2=0.9$      &  $r_1=300, r_2=1.0$ \\
\bottomrule
\end{tabular}}
\end{table}

\subsection{Hyperparameter Settings in LLaMA Experiments}
% We use grid search to tune the best base learning rate $\eta_{\max}$ for each algorithm within the range of $\{1\times 10^{-4}, 3\times 10^{-4}, 5\times 10^{-4}, 3\times 10^{-3}, 5\times 10^{-3} \}$. Muon algorithm has one more learning rate for non-hidden layers. The minimal decayed learning rate is set as $\eta_{\min} = 1/10 \eta_{\max} $.  Muon and D-Muon has three momentum parameters, where $\beta_1$ and $\beta_2$ is used in AdamW optimizer and $\beta_3$ is used for hidden layers. LANTON has two momentum parameters, $\beta_1$ is used for gradient and $\beta_2$ is used for gradient noise. LMO-based algorithms (SCION, D-Muon and LANTON) scale base learning rates for different layer group. For SCION and D-Muon, we set the scale parameters following the best tuned ones mentioned in their original paper.

The best base learning rate for each algorithm is grid searched  over $\{1\times10^{-4},\,3\times10^{-4},\,5\times10^{-4}, \,  8\times10^{-4}, \, 1\times10^{-3},\,3\times10^{-3},\,5\times10^{-3}\}$. The decayed layer rate is set as $\eta_{\min}=1/10 \eta_{\max}$ on C4 and $\eta_{\min}=1/20 \eta_{\max}$ on Minipile. We keep the momentum and scale parameters as that in GPT2 experiments. The hyperparameter choices on C4 and Minipile are summarized in  \cref{tbl:hyperparameter_c4,tbl:hyperparameter_pile}, respectively.
 
\begin{table}[h]
\centering
\caption{The hyperparameter settings on C4.}
\label{tbl:hyperparameter_c4}

\scalebox{0.8}{
\setlength{\tabcolsep}{6pt}
\renewcommand{\arraystretch}{1.15}
\begin{tabular}{l|ccccc}
\toprule
\textbf{Method} & $\eta_{\max}$ & $\eta_{\min}$ & \text{Moment} & $\text{Scale}$\\
\hline
\text{AdamW}          & $3\times10^{-4}$ & $3\times10^{-5}$  & $\beta_1=0.9, \beta_2=0.95$      & - \\
\text{Muon}           & $(5\times10^{-3}, 3\times10^{-4})$   & $(5\times10^{-4}, 3\times10^{-5})$ & $\beta_1=0.9, \beta_2=0.95, \beta_3=0.95$ & - \\
\text{MARS}           & $1\times10^{-3}$ & $1\times10^{-4}$  & $\beta_1=0.9, \beta_2=0.95$      & -  \\
\text{SCION}          & $5\times10^{-4}$ & $5\times10^{-5}$  & $\beta=0.9$                      &$r_1=50, r_2=3000$ \\
\text{D-Muon}         & $5\times10^{-3}$ & $5\times10^{-4}$  & $\beta_1=0.9, \beta_2=0.95, \beta_3=0.95$      &  $r=0.2$ \\
\text{LANTON}         & $5\times10^{-3}$ & $5\times10^{-4}$  & $\beta_1=0.95, \beta_2=0.9$      &  $r_1=300, r_2=1.0$ \\
\bottomrule
\end{tabular}
}% end scalebox
\end{table}

\begin{table}[!h]
\centering
\caption{The hyperparameter settings on Minipile.}
\label{tbl:hyperparameter_pile}
\scalebox{0.8}{
\setlength{\tabcolsep}{6pt}
\renewcommand{\arraystretch}{1.15}
\begin{tabular}{l|ccccc}
\toprule
\textbf{Method} & $\eta_{\max}$ & $\eta_{\min}$ & \text{Moment} & $\text{Scale}$\\
\hline
\text{AdamW}          & $8\times10^{-4}$ & $4\times10^{-5}$  & $\beta_1=0.9, \beta_2=0.95$      & - \\
\text{Muon}           & $(5\times10^{-3}, 5\times10^{-4})$   & $(2.5\times10^{-4}, 2.5\times10^{-5})$ & $\beta_1=0.9, \beta_2=0.95, \beta_3=0.95$ & - \\
\text{MARS}           & $1\times10^{-3}$ & $5\times10^{-5}$  & $\beta_1=0.9, \beta_2=0.95$      & -  \\
\text{SCION}          & $5\times10^{-4}$ & $2.5\times10^{-5}$  & $\beta=0.9$                      &$r_1=50, r_2=3000$ \\
\text{D-Muon}         & $5\times10^{-3}$ & $2.5\times10^{-4}$  & $\beta_1=0.9, \beta_2=0.95, \beta_3=0.95$      &  $r=0.2$ \\
\text{LANTON}         & $5\times10^{-3}$ & $2.5\times10^{-4}$  & $\beta_1=0.95, \beta_2=0.9$      &  $r_1=300, r_2=1.0$ \\
\bottomrule
\end{tabular}
}% end scalebox
\end{table}

%% file: appendix/robustness.tex
The training and validation loss curves with different base learning rates are presented in Figure \ref{fig:robustness}.

\begin{figure}[!h]
        \centering
        \includegraphics[width=0.3\linewidth]{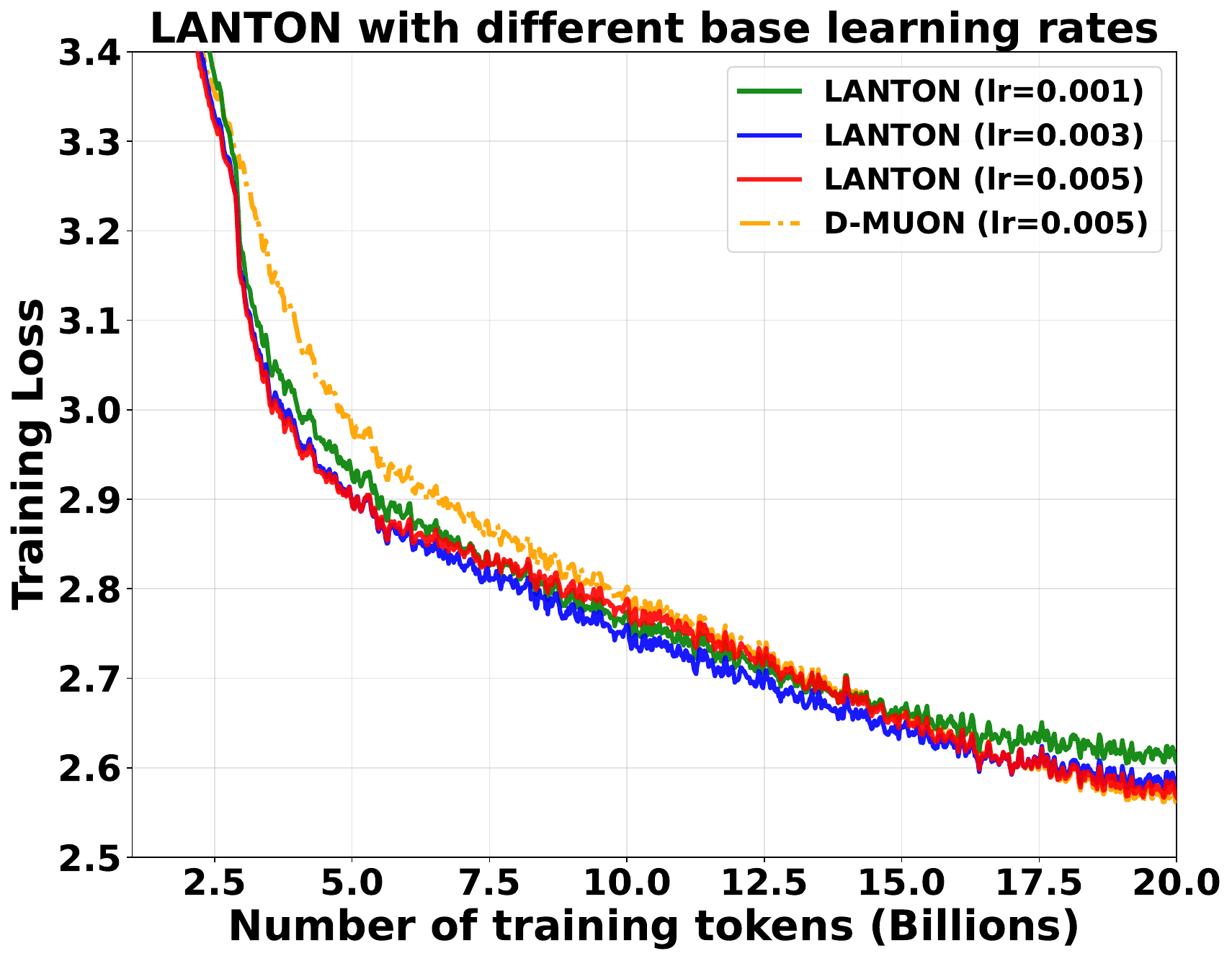}
        \includegraphics[width=0.3\linewidth]{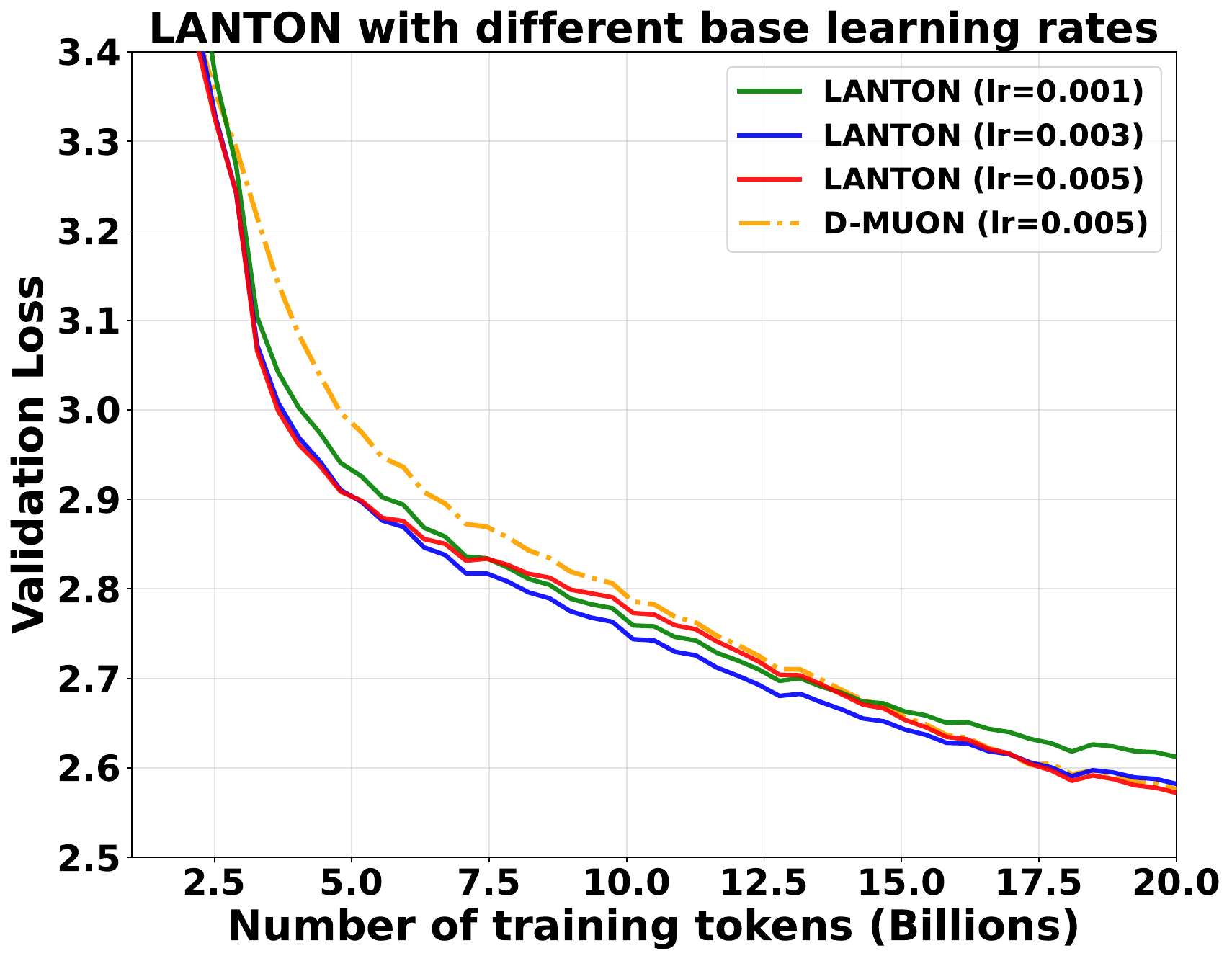}
        \caption{LANTON is robust to the choices of base learning rates.}
        \label{fig:robustness}
\end{figure}

%% file: appendix/batch_size.tex
To assess the influence of batch size on stochastic gradient variance estimation, we trained GPT (124M) models on openwebtext-100k with batch sizes $\text{BS} = \{8, 16, 32, 48, 64\}$ for one epoch (the number of training tokens is fixed to 46 million). For each batch size, we independently tuned the learning rate to its best-performing values ($1.0\times 10^{-2}$ for BS=8, $5.0\times 10^{-3}$ for other BS settings), ensuring a fair comparison across different settings. As shown in training loss curve in \cref{fig:ablation_bs}, smaller batches yield noisier trajectories while larger batches produce smoother curves, yet all settings converge to nearly the same final training and validation loss (approximately 4.0).

These results demonstrate that our method is highly robust to batch-size variation: the convergence behavior and final performance are reasonably good and consistent across a wide range of batch sizes. Among the configurations, $\text{BS}=16$ provides the best model performance, which is used in the main experimental settings.

\begin{figure}[!h]
        \centering
        \includegraphics[width=0.3\linewidth]{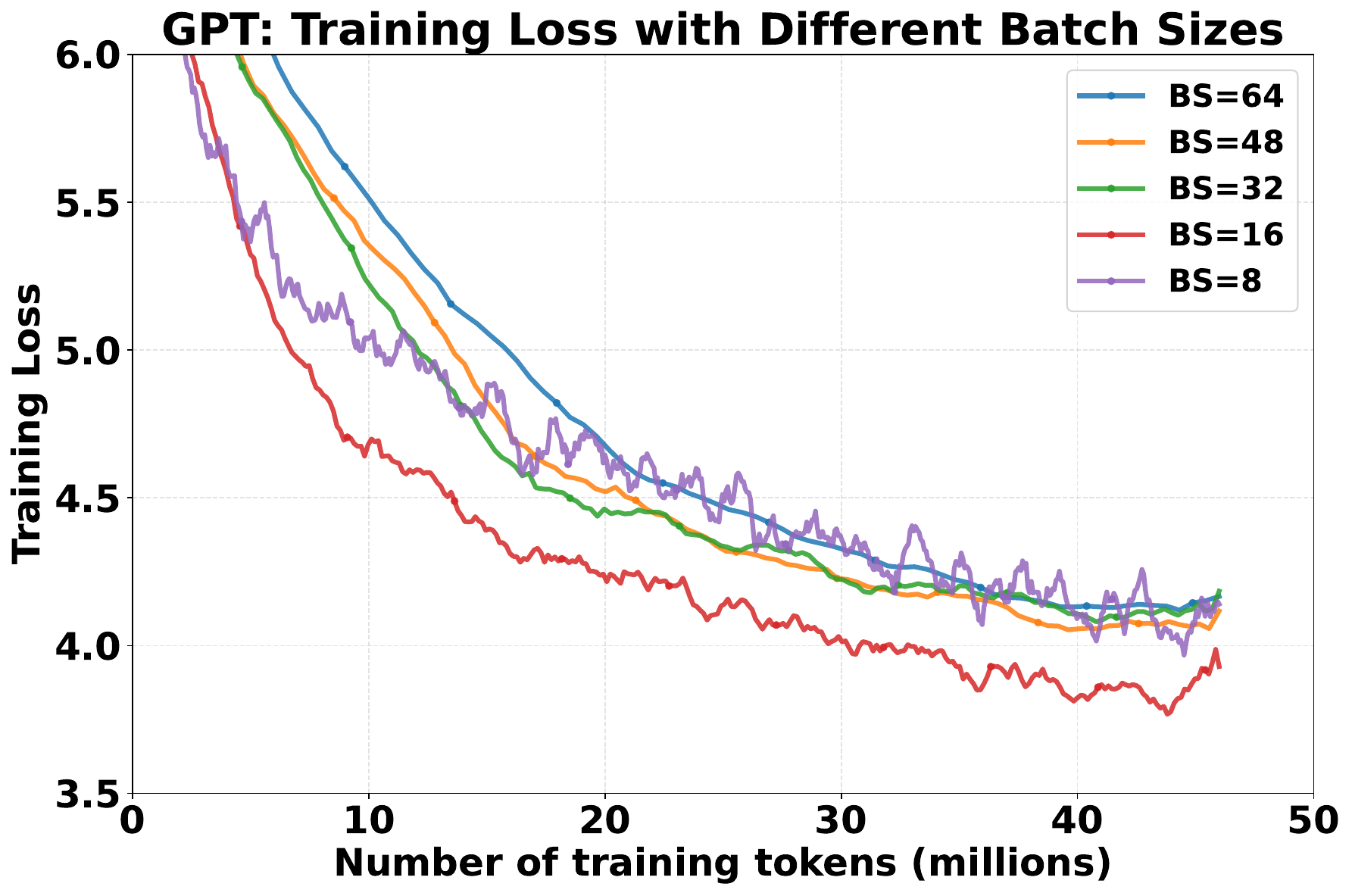}
        \includegraphics[width=0.3\linewidth]{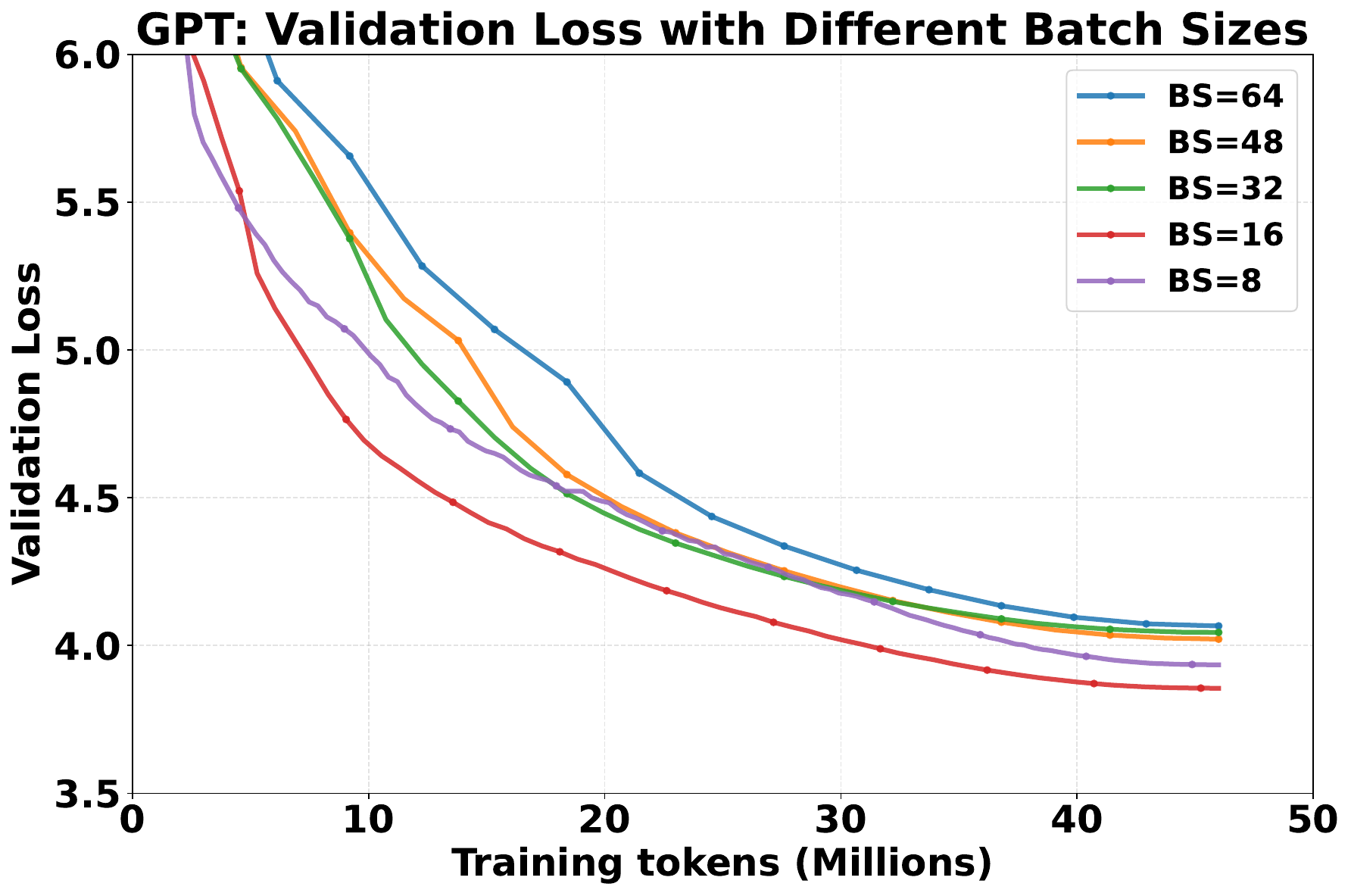}
        \caption{Training and validation loss vs. batch sizes (BS).}
        \label{fig:ablation_bs}
\end{figure}

%% file: appendix/sample_efficiency.tex
To study the sample efficiency of our algorithm under various token budgets, we double the budget of tokens for D-Muon (i.e., $40$B tokens) as that in LANTON  (i.e., $20$B tokens), and keep other experimental settings the same as that in Section \ref{sec:llama_exp}, including the base learning rate, scale hyperparameters and batch size. Both algorithms use cosine learning rate decay, but the difference is that D-Muon has $2\times$ total training steps since it has $2\times$ more training tokens. Figure \ref{fig:sample_eff} shows that D-Muon and LANTON reach comparable training/validation losses when D-Muon uses about $1.5\times$ more tokens than LANTON (i.e., $30$B tokens for D-Muon and $20$B tokens for LANTON for reaching $\sim2.57$ loss), demonstrating that the noise-adaptive learning rates can improve sample efficiency.

% \begin{figure*}[t]
%     \centering
%     \subfigure[Comparison of sample efficiency.]{
%     \includegraphics[width=0.25\linewidth]{fig/speed_train.pdf}
%     \includegraphics[width=0.25\linewidth]{fig/speed_val.pdf}}
%   \label{fig:sample_eff}
%   \subfigure[Running time on 1.1B model.]{
%   \includegraphics[width=0.24\linewidth]{fig/train_loss_vs_time.pdf}
%   \includegraphics[width=0.24\linewidth]{fig/val_loss_vs_time.pdf}}
%   \label{fig:running_time_1b}
% \end{figure*}

\begin{figure*}[t]
    \centering

    \begin{subfigure}{0.48\linewidth}
        \centering
        \includegraphics[width=0.48\linewidth]{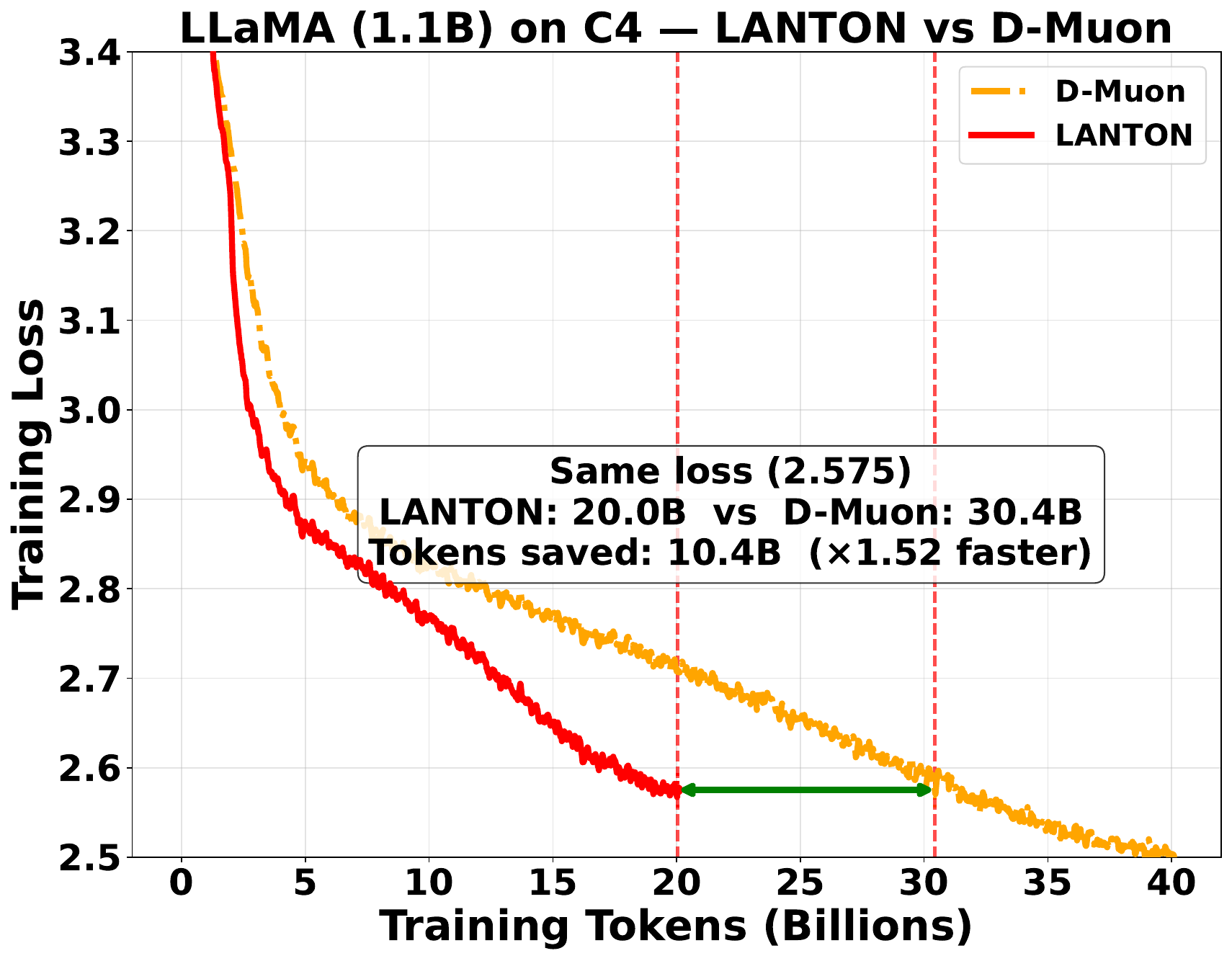}
        \includegraphics[width=0.48\linewidth]{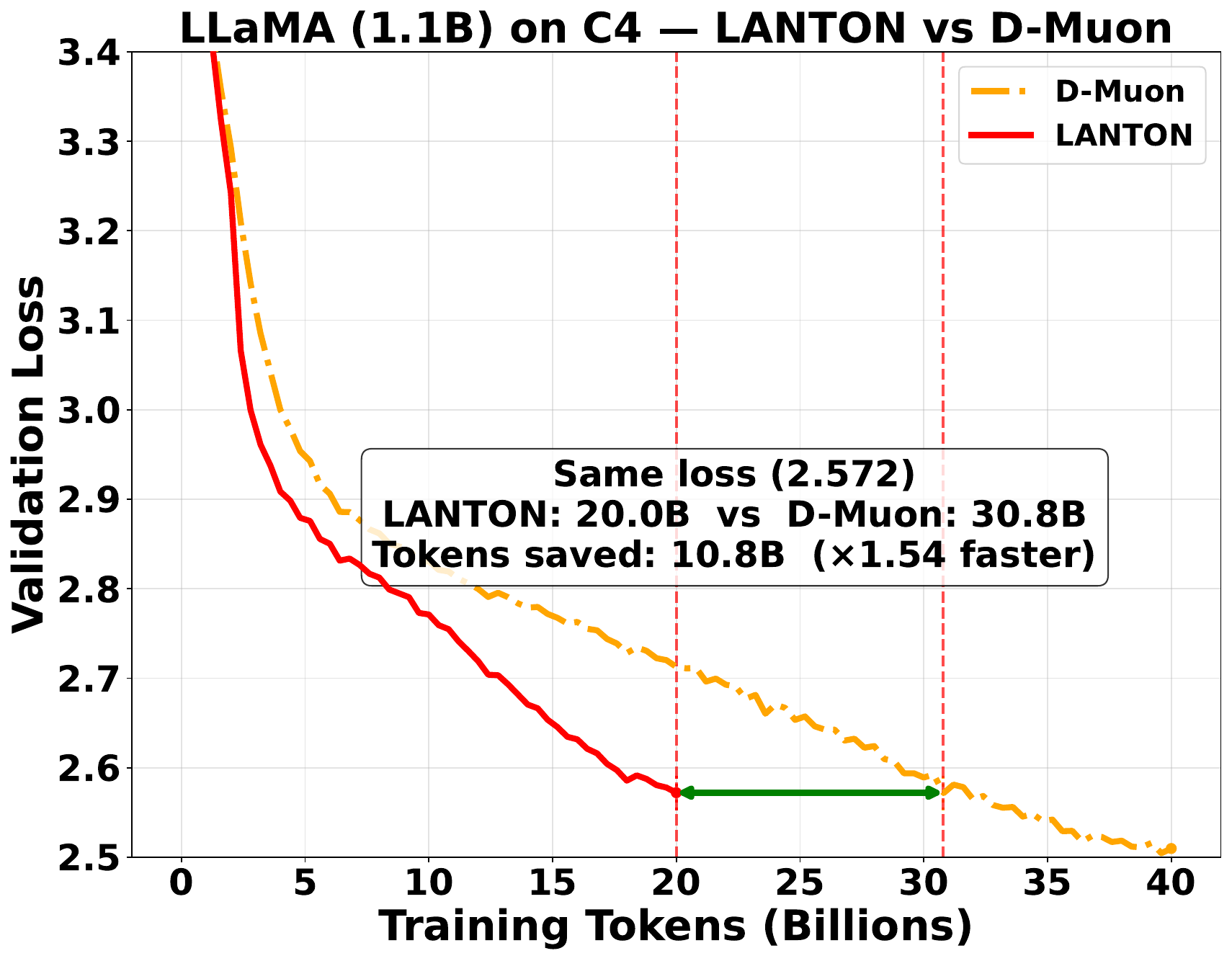}
        \caption{Comparison of sample efficiency.}
        \label{fig:sample_eff}
    \end{subfigure}
    \hfill
    \begin{subfigure}{0.48\linewidth}
        \centering
        \includegraphics[width=0.48\linewidth]{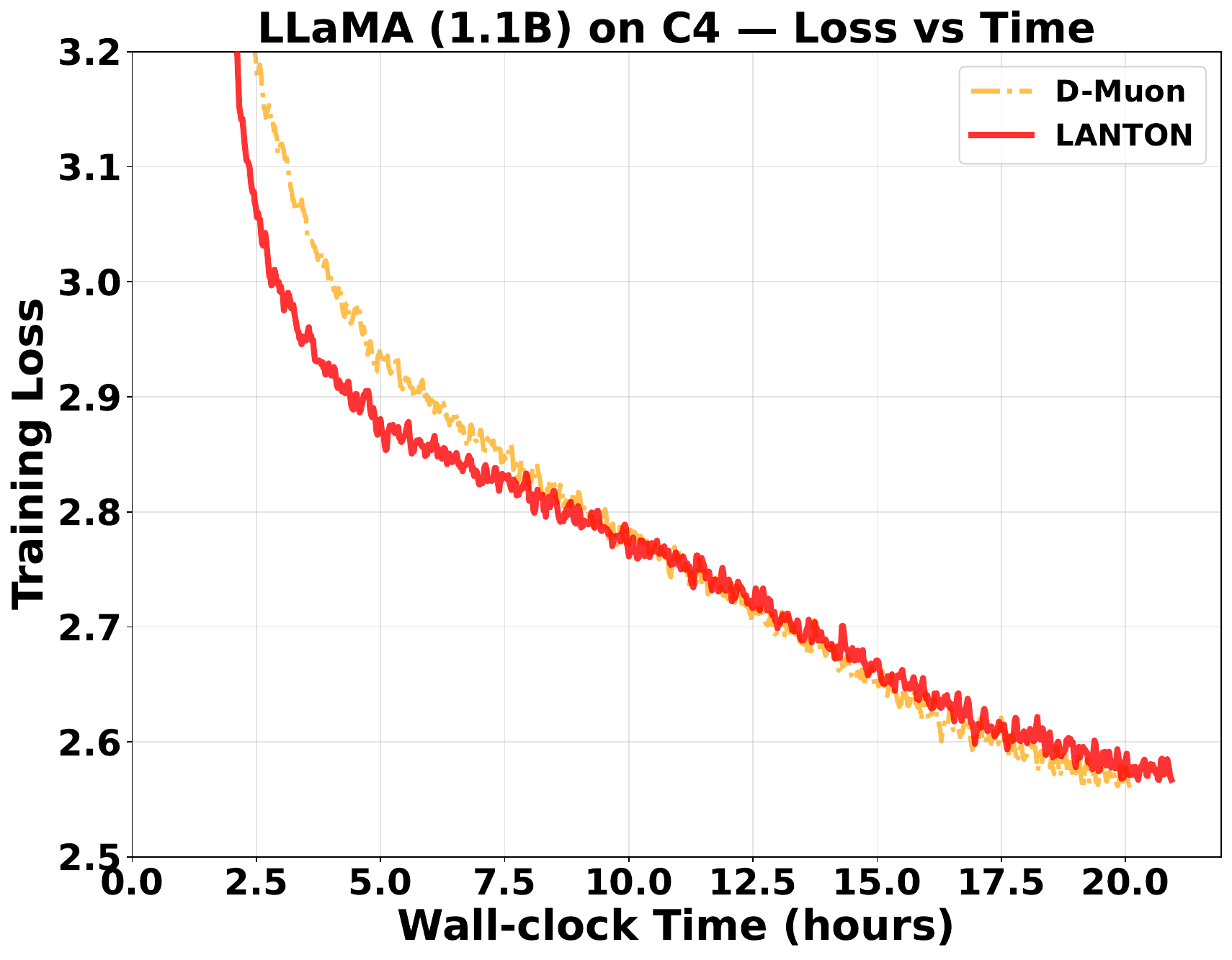}
        \includegraphics[width=0.48\linewidth]{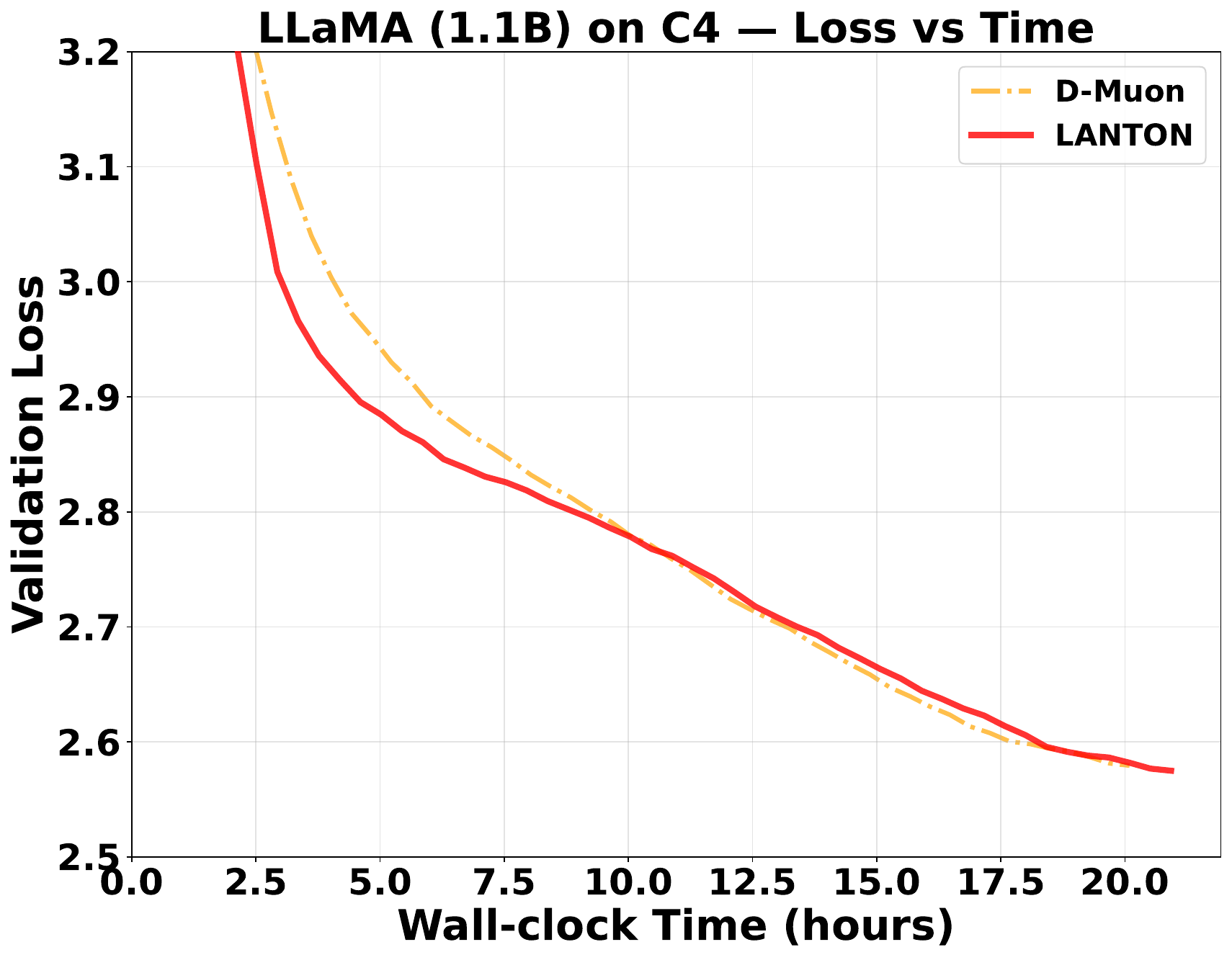}
        \caption{Running time on the 1.1B model.}
        \label{fig:running_time_1b}
    \end{subfigure}
    \caption{Sample efficiency and runtime comparison of \textsc{LANTON} and baselines.}
    \label{fig:overall_runtime}
\end{figure*}

\begin{table}[h]
\centering
\caption{The comparison of running time (LLaMA 1.1B).}
\label{tbl:running_time}
\scalebox{0.8}{
\setlength{\tabcolsep}{6pt}
\renewcommand{\arraystretch}{1.15}
\begin{tabular}{l|cc}
\toprule
\textbf{Method} & Time (second)/10 steps  & Total running time (hours)\\
\hline
\text{AdamW}          & $64.55$        & $18.53$ \\
\text{Muon}           & $69.62$        & $19.96$ \\
\text{MARS}           & $69.01$        & $19.78$ \\
\text{SCION}          & $71.53$        & $20.49$ \\
\text{D-Muon}         & $70.07$        & $20.08$ \\
\text{LANTON}         & $73.08$        & $20.92$ \\
\bottomrule
\end{tabular}
}% end scalebox
\end{table}

% \begin{figure}[!h]
%         \centering
%         \subfigure[Running time on 1.1B model.]{\includegraphics[width=0.24\linewidth]{fig/train_loss_vs_time.pdf}
%         \includegraphics[width=0.24\linewidth]{fig/val_loss_vs_time.pdf}\label{fig:running_time_1b}}
%         \subfigure[Running time on 2B model.]{\includegraphics[width=0.24\linewidth]{fig/train_loss_vs_time_2B.pdf}
%         \includegraphics[width=0.24\linewidth]{fig/val_loss_vs_time_2B.pdf}\label{fig:running_time_2b}}
%         \label{fig:running_time}
% \end{figure}

%% file: appendix/lr_stats.tex
The early-stage speedup arises because gradient noise varies significantly across layers at the beginning of training. As shown in \cref{fig:lr_stats}, the hidden layers (in subfigure (a)) start with an averaged effective learning-rate mean of $0.0028$ and a standard deviation of $0.0007$, indicating notable layer-wise differences that LANTON can exploit to accelerate optimization in the early stage.
By the end of training, cosine decay drives all learning rates toward very small values, and the hidden-layer learning rates converge to a mean of $0.00016$ with a much smaller standard deviation of $0.00008$. The reduced variance shows that layerwise learning rates become nearly uniform in the later stage of the training, and therefore layerwise learning rate is equivalent to using the same learning rate in the same group and the benefit diminishes.

Importantly, LANTON achieves faster early loss descent while still reaching comparable or better final performance, demonstrating that its advantage to accelerate training with noise-adaptive layer-wise learning rates.

% \textcolor{red}{($\star\star$) curve (x: iteration, y: std): lrs' std in group}?

% \begin{figure}[!h]
%         \centering
%         \subfigure[]{\includegraphics[width=0.3\linewidth]{fig/muon_param_effective_lr_bar.pdf}}
%         \subfigure[]{\includegraphics[width=0.3\linewidth]{fig/embedding_effective_lr_bar.pdf}}
%         \subfigure[]{\includegraphics[width=0.3\linewidth]{fig/norm_param_effective_lr_bar.pdf}}
%         \caption{The statistics of learning rate in 3 layer groups: (a) start: $0.0028\pm0.0007$, end: $0.00016\pm0.00008$; (b) start: $0.0035\pm0.0015$, end: $0.00034\pm0.00016$; (c) start: $0.0017\pm0.0003$, end: $0.0002\pm0.00006$. }
%         \label{fig:lr_stats}
% \end{figure}

\begin{figure}[t]
  \centering

  \begin{subfigure}{0.3\linewidth}
    \centering
    \includegraphics[width=\linewidth]{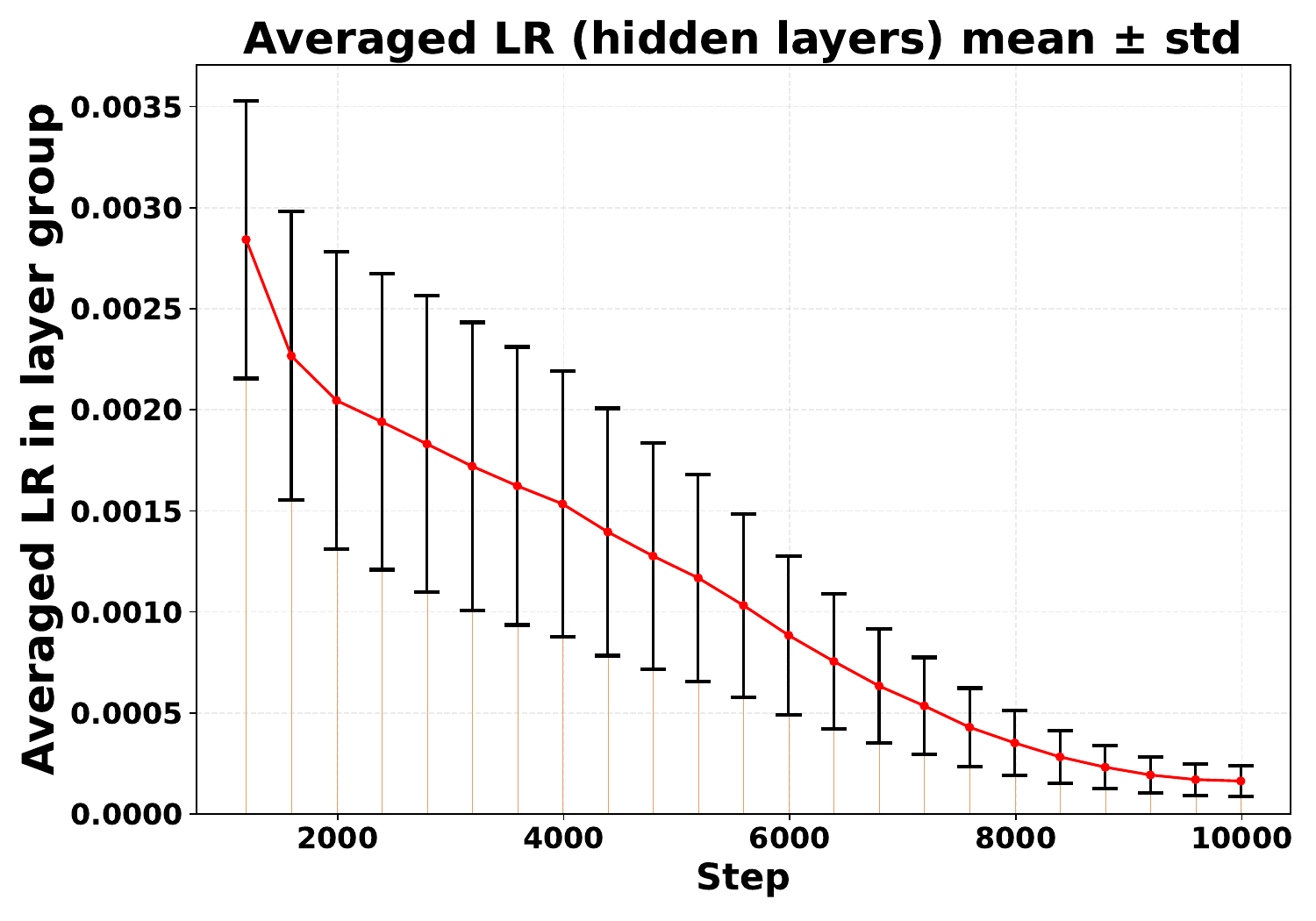}
    \caption{}
  \end{subfigure}
  \hfill
  \begin{subfigure}{0.3\linewidth}
    \centering
    \includegraphics[width=\linewidth]{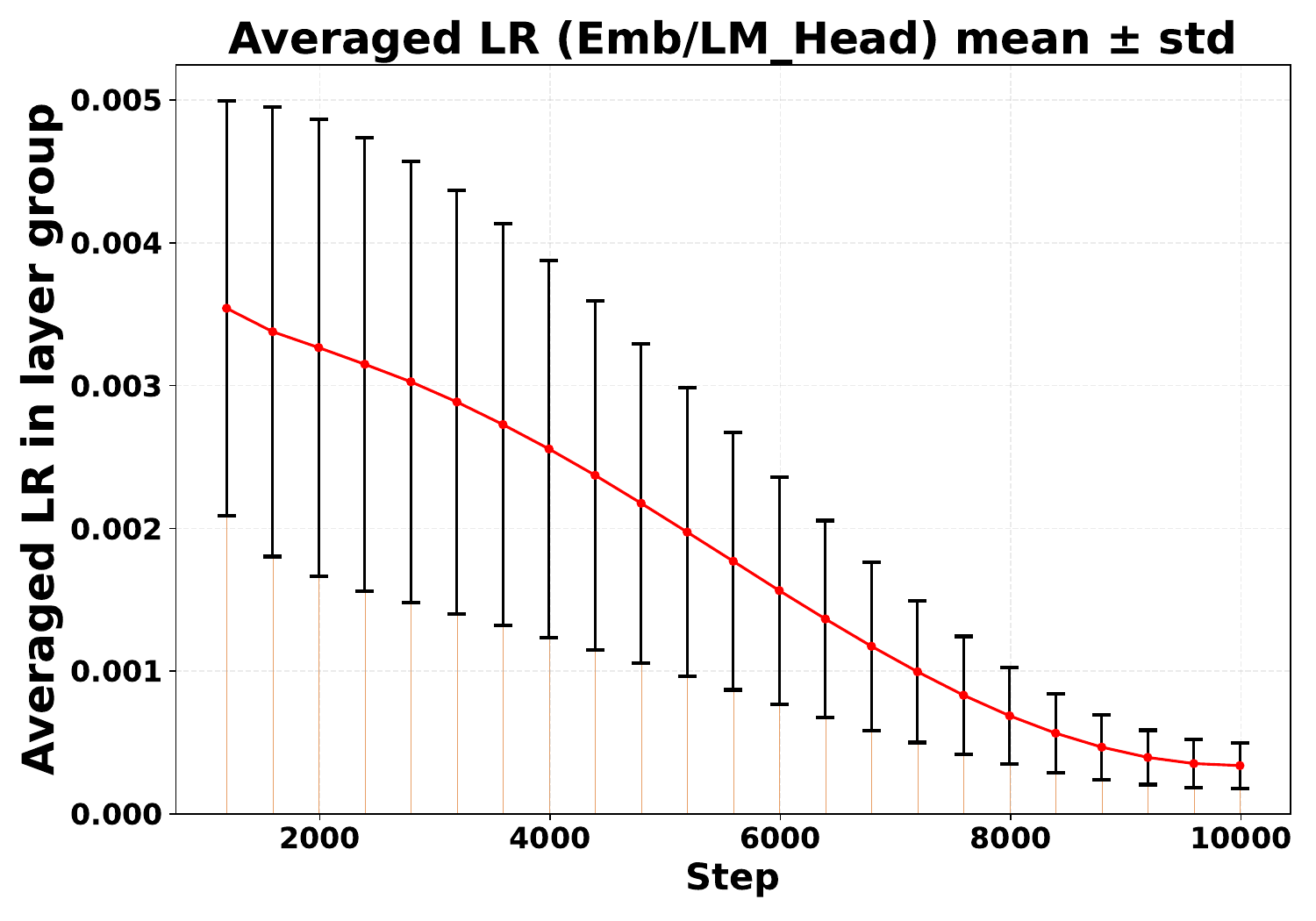}
    \caption{}
  \end{subfigure}
  \hfill
  \begin{subfigure}{0.3\linewidth}
    \centering
    \includegraphics[width=\linewidth]{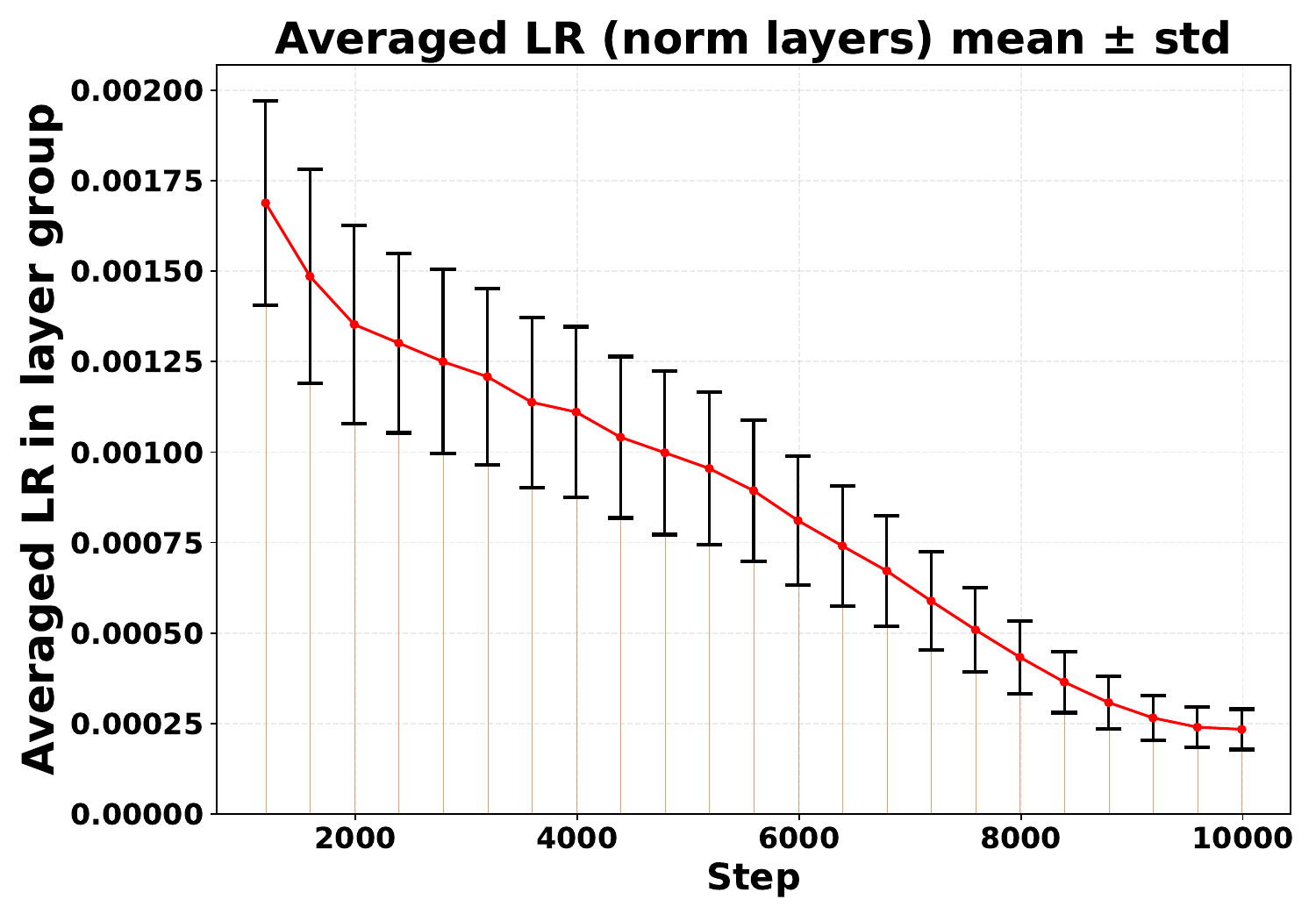}
    \caption{}
  \end{subfigure}

  \caption{The statistics of learning rate in three layer groups:
  (a) start: $0.0028\pm0.0007$, end: $0.00016\pm0.00008$;
  (b) start: $0.0035\pm0.0015$, end: $0.00034\pm0.00016$;
  (c) start: $0.0017\pm0.0003$, end: $0.0002\pm0.00006$.}
  \label{fig:lr_stats}
\end{figure}

%% file: appendix/options.tex
We compared the performance of Options 1 and 2 in \cref{alg:muon}. As described in line~7, our main experiments use Option 1. For Option~2, estimating gradient noise requires two independent mini-batches per iteration; therefore, under a fixed one-epoch budget, Option~2 performs only half as many optimization steps as Option~1.

Figure \ref{fig:ablation_options} reports the training and validation curves for both settings. With the same one-epoch budget, Option~1 achieves much lower final training and validation loss than Option~2 because it performs more gradient updates.

\begin{figure}[!h]
        \centering
        \includegraphics[width=0.3\linewidth]{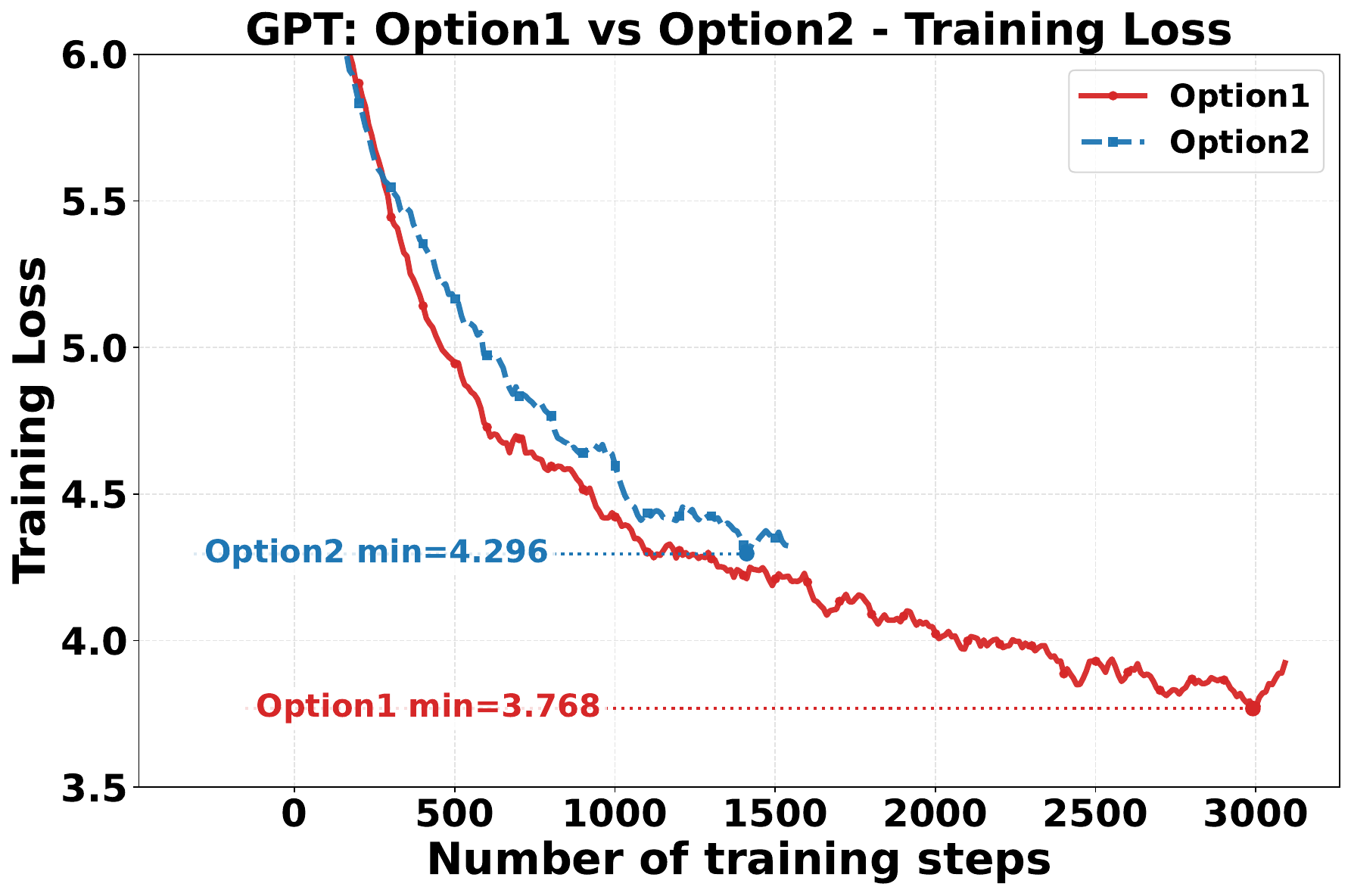}
        \includegraphics[width=0.3\linewidth]{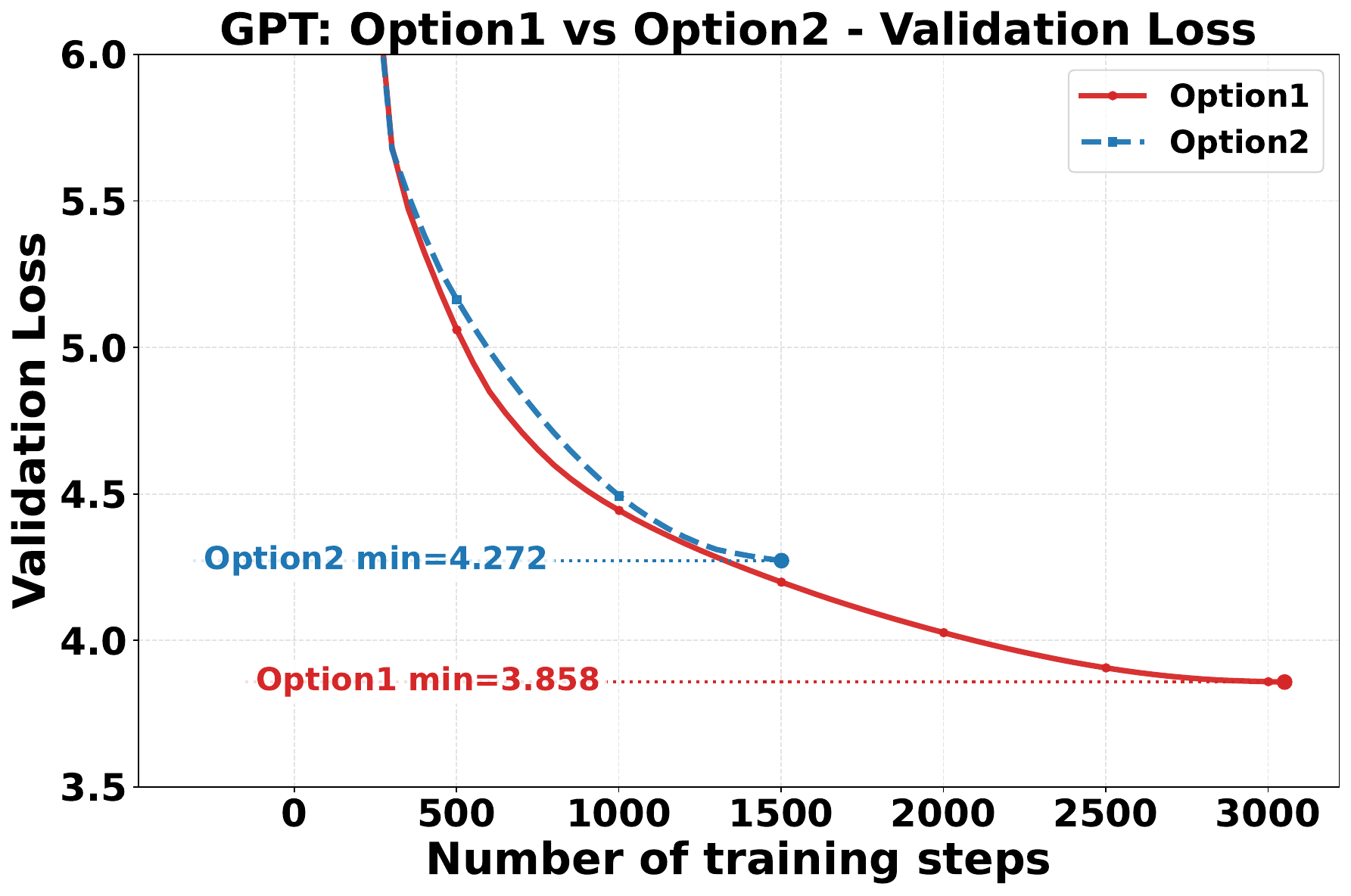}
        \caption{Training and validation loss with two gradient noise estimation options.}
        \label{fig:ablation_options}
\end{figure}

%% file: appendix/license.tex
\paragraph{GPT2}
OpenAI's GPT2 models are distributed by MIT License. We use only the open-source implementation of the GPT2 architecture in Hugging Face Transformers and do not redistribute Meta’s model weights. 

\paragraph{LLaMA}
We follow Meta Llama 2 Community License Agreement. We use only the open-source implementation of the LLaMA architecture in Hugging Face Transformers and do not redistribute Meta’s model weights. 

\paragraph{C4}
The English portion of the C4 (Colossal Clean Crawled Corpus) dataset comes from Hugging Face (allenai/c4), which is distributed under the Open Data Commons Attribution (ODC-By 1.0) license.
\paragraph{Minipile} It can be accessed from Hugging Face (JeanKaddour/minipile), which is distributed under MIT License.
\paragraph{Openwebtext} It can be accessed from Hugging Face (Skylion007/openwebtext), which is distributed under Creative Commons cc0-1.0 license.